\documentclass[twoside,11pt]{article}

\usepackage{blindtext}

\usepackage{jmlr2e}
\usepackage{amsmath}
\usepackage{amssymb}
\usepackage{mathtools}
\usepackage{subcaption}
\usepackage[capitalize,noabbrev]{cleveref}
\usepackage{algorithm}
\usepackage{algorithmic}
\usepackage{enumitem}
\usepackage{booktabs}
\hypersetup{hidelinks}

\newcommand{\mR}{\mathbb{R}} 
\newcommand{\mN}{\mathbb{N}}
\newcommand{\mZ}{\mathbb{Z}}
\newcommand{\mC}{\mathbb{C}}

\newcommand{\E}{\mathbb{E}}
\newcommand{\var}{\mathrm{Var}}

\newcommand{\degprof}{\mathrm{DegP}}

\usepackage{lastpage}
\jmlrheading{25}{2024}{1-\pageref{LastPage}}{2/24}{11/24}{24-0220}{Emmanuel Abbe, Samy Bengio, Aryo Lotfi, and Kevin Rizk}

\ShortHeadings{Generalization on the Unseen, Logic Reasoning and Degree Curriculum}{Abbe, Bengio, Lotfi, and Rizk}
\firstpageno{1}

\begin{document}

\title{Generalization on the Unseen, Logic Reasoning and Degree Curriculum}

\author{\name Emmanuel Abbe \email emmanuel.abbe@epfl.ch \\
       \addr EPFL, Apple,\\
        Lausanne, VD, Switzerland
       \AND
       \name Samy Bengio \email bengio@apple.com \\
       \addr Apple,\\
       Cupertino, CA, USA
       \AND
       \name Aryo Lotfi \email aryo.lotfi@epfl.ch \\
       \addr 
       EPFL,\\
       Lausanne, VD, Switzerland
       \AND
       \name Kevin Rizk \email kevin.rizk@epfl.ch \\
       \addr EPFL,\\
       Lausanne, VD, Switzerland}

\editor{Kilian Weinberger}

\maketitle

\begin{abstract}
This paper considers the learning of logical (Boolean) functions with a focus on the {\it generalization on the unseen (GOTU)} setting, a strong case of out-of-distribution generalization. This is motivated by the fact that the rich combinatorial nature of data in certain reasoning tasks (e.g., arithmetic/logic) makes representative data sampling challenging, and 
learning successfully under GOTU gives a first vignette of an `extrapolating' or `reasoning' learner. 
We study how different network architectures trained by (S)GD perform under GOTU and provide both theoretical and experimental evidence that for sparse functions and a class of network models including instances of Transformers, random features models, and linear networks, a {\it min-degree-interpolator} is learned on the unseen. More specifically, this means an interpolator of the training data that has minimal Fourier mass on the higher degree basis elements. 
These findings lead to two implications: (1) we provide an explanation to the {\it length generalization problem} for Boolean functions (e.g., Anil et al.\ 2022); (2) we introduce a  curriculum learning algorithm called {\it Degree-Curriculum} that learns monomials more efficiently by incrementing supports. Finally, we discuss extensions to other models or non-sparse regimes where the min-degree bias may still occur or fade, as well as how it can be potentially corrected when undesirable. 
\end{abstract}


\section{Introduction}

Neural networks trained by stochastic gradient descent (SGD) have proved to be a powerful learning paradigm when there is enough representative data about the distribution to be learned, specifically in applications involving images or text where there is also a good understanding of the relevant architectures.

There is now an increasing interest in tackling tasks involving more `reasoning' components, which depart from classical perception tasks of images and texts. While such tasks remain vaguely defined, a list that we consider here under this class is given by: 
(1) arithmetic and algebra \citep{saxton2019analysing, lewkowycz2022solving-minerva}, (2) synthetic tasks such as PVR \citep{Zhang2021PointerVR} and LEGO \citep{zhang2022unveiling}, (3) visual reasoning such as CLEVR \citep{johnson2017clevr}, (4) physical reasoning such as Phyre \citep{bakhtin2019phyre}, (5) algorithmic data such as CLRS \citep{velivckovic2022clrs} and reasoning on graphs \citep{mahdavi2022better}.

One common trademark of these tasks is that the input space is usually of a discrete and combinatorial nature, and consequently, the data may not necessarily lay on a low-dimensional manifold that is well sampled. In various cases, the input space may even have a variable length. This combinatorial nature is already present in text, but it is further amplified in, say, arithmetic since most symbol combinations could a priori represent a valid input (in contrast to text). Further, the target function in such tasks may rely on a large composition of logical steps or mathematical operations that require to be jointly learned. Therefore, in such reasoning tasks, the setting with abundant representative data seems less prominent. This motivates us to focus on a strong out-of-distribution (OOD) generalization setting.  

For instance, when learning arithmetic or logic functions on a training set with a bounded length or a limited number of truth assignments, how would the neural network generalize on more general input assignments? (This is a case of length generalization.) 
When training a neural network to learn a Boolean formula, such as a voting scheme on data from a polarized cohort of voters, how does the network generalize to an unpolarized cohort? 

We thus consider the problem of learning functions with a holdout domain where part of the distribution support is barely/never seen at training, and with target functions that are Boolean to capture the discrete and combinatorial nature of various reasoning tasks (e.g., arithmetic, decision trees, logical circuits). Learning successfully under holdout gives a first vignette that the learner is operating with a certain amount of `reasoning' or `extrapolation' since memorization is voided on the unseen domain.

\subsection{Our Main Contributions}
\begin{enumerate}
\item We lay down some basic principles of stronger generalization requirements that rely on the `generalization on the unseen (GOTU)' performance metric, defined as a strong case of OOD generalization to measure `extrapolating' or `reasoning' properties of models on considered tasks.
\item We study how standard neural network architectures trained by (S)GD perform on the GOTU metric, in particular, which solutions are learned on the unseen domain for such architectures: \\
(i) we prove three theoretical results showing that for a class of network models including random features model (Theorem \ref{thm:random-features}), deep diagonal linear networks (Theorem \ref{thm:diagonal}), and 2-layer fully connected linear neural networks (Theorem \ref{thm:2-linear}), a {\it min-degree-interpolator (MD interpolator)} is learned on the unseen; 
\\(ii) we show experimental results (Section \ref{sec:exps}) supporting that encoder-only Transformers tend to also have the {\it min-degree bias (MD bias)} towards MD interpolators.

The MD interpolator is defined as the interpolator of minimal {\it degree-profile}, i.e., the Boolean function interpolating the training data and having a Fourier-Walsh transform whose energy concentrates on basis elements of lowest possible degree. Connections to algebraic geometry are given in Appendix \ref{app:vanishing-ideals} in order to characterize how MD interpolators can be constructed from the `vanishing ideal' of the 
seen data. We also point out that very large learning rates or other architectures (such as mean-field networks) can exhibit leaky MD bias (i.e., assigning larger mass on higher-degree monomials); see Appendix \ref{app:learning-rate-sensitivity}.\footnote
{Our code is available at \url{https://github.com/aryol/GOTU}}
\item Using these, we obtain two additional results:\\
(i) we provide a formal explanation (Theorem \ref{thm:length-gen}) to the `length generalization problem' discussed in the work of \citet{anil2022exploring-length} (for the case of bounded weight vectors, also related to the work of \citealp{zhang2022unveiling});\\ (ii) we turn the min-degree bias into an asset to accelerate learning by introducing a curriculum learning algorithm called `Degree-Curriculum' (\cref{alg:degree-curriculum}), which successively increases the input complexity with respect to Hamming weights in order to incrementally learn the monomials support (see Section \ref{curr}). 
\item Finally, we provide experimental results on the role of the sparsity condition (\cref{sec:dimension-sensitivity}) and the role of the causal attention mask in Transformers (\cref{sec:causal-transformer}). We show that the min-degree bias may still be present or fade depending on these conditions. We also demonstrate how the min-degree bias can be mitigated when undesirable using symmetry-based regularizers in \cref{sec:conclusions}.

\end{enumerate}

This work is an extension of earlier work \citep{icml-version}. In this extension, we have added the theoretical analysis for 2-layer fully connected linear neural networks (Theorem \ref{thm:2-linear}) along with experimental results (in \cref{app:exps}) for the general case conjectured in Conjecture \ref{conj:linear}. We also investigated relaxing the sparsity condition in \cref{sec:dimension-sensitivity} and using Transformers with causal attention masking in \cref{sec:causal-transformer}. We have also included a preliminary attempt to tame the min-degree bias when undesirable using a symmetry-based regularization term in \cref{sec:conclusions}.
\section{Generalization on the Unseen}\label{sec:gotu}
The classical setting of statistical learning theory requires the control of three error pillars for the generalization of a learning model: (1) the approximation error (depending on the properties/richness of the model class), (2) the estimation error (depending on the properties/richness of the training set), (3) the optimization error (depending on the properties/richness of the training algorithm). 

In some of the recent deep learning applications for computer vision and natural language processing, the richness of the training set, the size of the model and its alignment with the data, as well as the computational power, make the three pillars well controlled. The recent success of large language models (LLM) and scaling laws are perfect examples of this phenomenon \citep{alabdulmohsin2022revisiting}.

As mentioned in the introduction, the type of data occurring in reasoning tasks is slightly different due to the richness and combinatorial nature of the data. To better cope with this challenge, we propose in this paper to depart from the  classical generalization objectives described with the three pillars. We focus instead upfront on distribution shift and, more specifically, a strong case of OOD generalization where part of the distribution domain is almost/completely unseen at training but used at testing (in particular, prohibiting any memorization scheme). 

Of course, on the unseen domain, all bets are off for generalization: one cannot hope for an algorithm trained on a given data domain to perform well on a larger data domain without any incentive to do so. Yet various algorithms will have various implicit biases on the unseen and thus produce various solutions on the unseen. Understanding this `bias on the unseen' for different network architectures and Boolean target functions is the objective of this paper.

We start by redefining the generalization error when the train and test distribution are not necessarily the same. 
\begin{definition}
Let $X_1, \ldots, X_m$ be drawn i.i.d.\ under $\mu_1$ and labeled by a target function $f$, and let $\tilde{f}$ be the function learned by a learning algorithm. The algorithm has $(\mu_1,\mu_2,m,\epsilon)$-generalization (for loss $\ell$) if $$\mathbb{E}_{X_1, \ldots, X_m \sim \mu_1, X_{m+1} \sim \mu_2} \left[\ell\left(\tilde{f}_{X_1, \ldots, X_m}\left(X_{m+1}\right) , f\left(X_{m+1}\right)\right)\right] \le \epsilon.$$ In other words, the algorithm is trained under distribution $\mu_1$ and tested under distribution $\mu_2$, producing $\epsilon$-test-loss with sample complexity $m$.
\end{definition}
Now we focus on a special case of interest, a strong case of OOD generalization where we essentially see all the data on some part of the domain but miss another part. Naturally, we will next study a `soft version' of this metric, where both in-distribution and out-of-distribution generalization are considered, but this strong case is already rich and insightful. 
\begin{definition}[Generalization on the Unseen]
    Consider a given sample domain $\Omega$. During training, part of $\Omega$ is not sampled, and we call this the unseen domain (or the holdout set) $\mathcal{U}$. At testing, however, we sample from the full set $\Omega$. This represents a special case of the previous definition where $\mu_1=\mu|_{\Omega \setminus \mathcal{U}}$ and $\mu_2=\mu|_{\Omega}$ for some $\mu$. 
    
    We now further specify the setting: we assume that the training error is 0 on the training domain $\Omega \setminus \mathcal{U}$, e.g., by seeing all the samples in $\Omega \setminus \mathcal{U}$, and define the generalization on the unseen (GOTU) for an algorithm $\tilde{f}$ and target function $f$ as 
    \begin{align}
        GOTU( f, \tilde{f}, \mathcal{U}) = \mathbb{E}_{X \sim_U \mathcal{U}}[\ell(\tilde{f}_{\Omega \setminus \mathcal{U}}(X) , f(X))],
    \end{align}
    where $\sim_U \mathcal{U}$ indicates uniform sampling from $\mathcal{U}$.
    Notice we only sample on $\mathcal{U}$ at testing because we assumed zero training error and considered the whole $\Omega \setminus \mathcal{U}$ as the training set.  
\end{definition}
A few remarks are in order:
\begin{itemize}
    \item GOTU is a special case of OOD and distribution shift setting that is extremal in the sense that it completely gives access to part of the distribution domain and completely omits the complement. Since we consider rich enough models to interpolate the data, the `statistical' and `approximation' pillars of the learning problem are removed (there may still be randomness used by the learning algorithm, thus statistical analysis may still be relevant). The problem thus turns into a pure optimization problem where the central object of study is the implicit bias of the learning algorithm on the unseen. Note that this is not exactly the same implicit bias as studied in the setting of overparametrized models \citep{soudry2017implicit, gunasekar2017implicit, gunasekar2018implicit, arora2019implicit, razin2020implicit, chizat2020implicit, moroshko2020implicit} as here we have the distribution shift and investigate the behavior of the equivalence class of interpolators on the unseen $\mathcal{U}$. 
    \item In some experiments, we  replace the `perfect' training data on the seen domain with a `large' sampling on the seen domain. We  defined the GOTU in the extreme case to simplify the number of parameters to track and to allow for cleaner theorem statements, but there could also be a sampling rate on $\Omega \setminus \mathcal{U}$; this is left for future research. Also, we assume a uniform prior here because this is a natural first case for arithmetic/logic tasks, but this could also be relaxed. 
    \item We will consider different subsets $\mathcal{U}$ in the applications. We are sometimes interested in $\mathcal{U}$s for which the data invariances or equivariances could give hope to learn. This is further specified with the next definition.    
    \end{itemize}

\begin{definition}
A function $f\colon \Omega \to \mathbb{R}$ is (1) $G$-invariant, or invariant under the group action $G$ on $\Omega$, if $f(gx)=f(x)$ for all $g \in G$, $x \in \Omega$; (2) $G_{i, o}$-equivariant, or equivariant under $G_{i,o}$, if $f(g_ix)=g_o f(x)$ for all $(g_i, g_o) \in G_{i, o}$ and $x\in \Omega$.\footnote{See the works of \citet{dummit2004abstract, ravanbakhsh2017equivariance, zhou2020meta} for more details on group actions and the definitions of invariance and equivariance.}
\end{definition}
As stated earlier, we cannot expect algorithms to generalize on the unseen domain by themselves. However, we can hope that certain training algorithms will catch invariances/equivariances and thus extrapolate. 
For example, consider the parity function on $d$ bits $x_1, \ldots, x_d \in \{\pm 1\}$ defined as $f(x_1, \ldots, x_d) = x_1x_2 \cdots x_d$. This function is permutation-invariant (group $G=S_d$). In particular, if one uses a model favoring permutation symmetries, one may not have to see all inputs that are permutation equivalent. There has been a series of works designing layers/architectures that are equivalent under a prespecified family of actions (e.g., all permutations) \citep[see][]{ravanbakhsh2017equivariance, zaheer2017deep, hartford2018deep}. More recently, \citet{zhou2020meta} propose a method to learn invariances in a multi-task setting using meta-learning. 
An example of an equivariant Boolean function would be the majority function on $\{+1,-1\}^d$, $d$ odd, with the action of global bit flipping on the input and the output (since the majority is reversed if all the bits are flipped). Thus a holdout on vectors of dual-weight could again be handled by a model having such an equivariance.   
Note that we are also interested in cases where these equi/in-variances are not  present in the target, to understand what solutions neural networks favor on the unseen.

\section{Results}
We consider $f \colon \Omega \to \mathbb{R}$ with $\Omega = \{\pm 1\}^d$. We use $[d]$ to denote $\{1, 2, \ldots, d\}$. We introduce some preliminary material on Boolean functions in the next part and then state our results. 
\subsection{Preliminaries}
\paragraph{Fourier-Walsh Transform} Any function $f\colon \{\pm 1\}^d \longrightarrow \mathbb{R}$ can be expressed as $f(x) = \sum_{T \subseteq [d]} \hat{f}(T)\chi_T(x)$ where $\chi_T(x) = \prod_{i \in T} x_i$ and $\hat{f}(T) = \mathbb{E}_{X \sim_U \{\pm 1\}^d}[\chi_T(X)f(X)]$ are the monomials and the coefficients respectively \citep{o'donnell_2014}. For example, the majority function on 3 bits can be written as $\mathrm{Maj}(x_1, x_2, x_3) = \frac{1}{2}(x_1+x_2+x_3 - x_1x_2x_3)$.

\paragraph{Unseen Domain and Vanishing Ideals} We now introduce the unseen domain $\mathcal{U}$. First, consider the canonical holdout \citep{abbe2022learning}, when a bit is frozen during training, e.g., $x_i=1$ and  $\mathcal{U}=\{x \in \{\pm 1\}^d: x_i=-1\}$. In this case, one can see that any function of the form $f(x) + (1-x_i)\Delta(x)$ ($\Delta(x)$ is arbitrary) is an equivalent interpolator on the training data. 
For general unseen domain $\mathcal{U} \subseteq \Omega = \{\pm 1\}^d$,  there exist polynomials $v_1(x), \ldots, v_k(x)$ such that $x\in \Omega \setminus \mathcal{U} \iff v_1(x)=\ldots=v_k(x)=0$ (see Appendix \ref{app:vanishing-ideals}). Consequently, all solutions of the form 
$f(x) + \Delta_1(x)v_1(x) + \cdots +\Delta_k(x)v_k(x)$ are equivalent at training. This is the quotient space of $f$ under the vanishing ideal defined by $\Omega \setminus \mathcal{U}$. We refer to Appendix \ref{app:vanishing-ideals} for more details on this relation to algebraic geometry.     

We now define measures of complexity relevant to us. 
\begin{definition}[Degree]
For a function $f\colon \{\pm 1\}^d \to \mR$, the degree $\deg(f)$ refers to the maximum degree of the monomials present in the Fourier-Walsh transform of $f$. For example, the degree of $\mathrm{Maj}(x_1, x_2, x_3) = \frac{1}{2}(x_1+x_2+x_3 - x_1x_2x_3)$ is 3.
\end{definition}
\begin{definition}[Degree profile]\label{def:degree-prof}
    For $f\colon \{\pm 1\}^d \to \mR$, we define the degree-profile of $f$, $\degprof(f) \in \mR^{d+1}$ such that $\degprof(f)_i = \sum_{T \subseteq [d], |T| = d+1-i} \hat{f}(T)^2$ for $1 \leq i \leq d+1$. Furthermore, we consider lexicographic ordering on these vectors, i.e., $\degprof(f) < \degprof(g)$ iff $\exists i\; \degprof(f)_i < \degprof(g)_i$ and $\degprof(f)_j = \degprof(g)_j \; 1\leq j < i$. For instance, the degree-profile of $\mathrm{Maj}(x_1, x_2, x_3) = \frac{1}{2}(x_1+x_2+x_3 - x_1x_2x_3)$ is $(1/4, 0, 3/4, 0)$.
\end{definition}
Intuitively, $\degprof(f)_i$ represents the Fourier mass on degree-$(d+1-i)$ monomials (for $1 \leq i \leq d+1$) and $\degprof(f)$ captures the spectrum of $f$. Note that the degree-profile is a stronger notion than the degree, i.e., $\deg(f) < \deg(g) \implies \degprof(f) < \degprof(g)$.

\begin{definition}[Min-degree interpolators]
Consider a target function $f$ and unseen domain $\mathcal{U}$. The set of interpolators is defined as $\mathcal{F}_{\mathrm{int}} (f,\mathcal{U}) = \{ g\colon\{\pm 1\}^d \to \mR \mid g(x) = f(x), \forall x \in \mathcal{U}^c\},$ where $\mathcal{U}^c \coloneqq \Omega \setminus \mathcal{U}$ is the seen domain. 
We call an interpolator a {\it min-degree interpolator (MD interpolator)} of $(f,\mathcal{U})$ (or of $\{x,f(x)\}_{x \in \mathcal{U}^c}$) if it is an element of $\mathcal{F}_{\mathrm{int}} (f,\mathcal{U})$ that minimizes the degree-profile with respect to the lexicographic order. This means that no part of the Fourier-Walsh expansion of the interpolator could be replaced with a lower-degree alternative and still interpolate. 

\end{definition}

For example, consider the case of `canonical holdout' where we always have $x_1=1$ at training, i.e., $\mathcal{U}=\{x \in \{\pm 1\}^d : x_1=-1\}$, and target function $x_1x_2 + x_1x_3x_4$. Here, both $x_1x_2 + x_3x_4$ and $x_2+x_3x_4$ are of degree $2$ but only $x_2+x_3x_4$ is an MD interpolator since $x_1x_2$ in the first function is replaceable with the lower-degree $x_2$. 
Further, note that there may be multiple interpolators having minimal   max-degree rather than degree-profile. For example, consider the unseen domain induced by $x_i = x_j$ and target  $f(x) = x_i + x_j$. Then $2x_i$ and $x_i + x_j$ are both interpolators with minimal max-degree, but only $x_i + x_j$ is an interpolator with a minimal degree-profile. In fact, the MD interpolator is always unique (if $f_1$ and $f_2$ are interpolators with the same degree-profile, then $\frac{f_1+f_2}{2}$ is an interpolator with a smaller degree-profile unless $f_1=f_2$.) 


\subsection{Main Theoretical Results}
We show that certain models have a min-degree implicit bias on the unseen. We start by giving another example.

\subsubsection{Result Preview from an Example}
Consider trying to learn the majority target function on $3$ voters $x_1,x_2,x_3$ having the following data distribution: voters $1$ and $2$ never vote both negatively, i.e., $(x_1,x_2)$ is never $(-1,-1)$ in the training data. Now train a neural network to learn the target on such a training data distribution (with only 3 variables, one will quickly see all sequences satisfying the required condition; this is to simplify the example, in our results, we consider higher dimensional versions of such examples). Since we always have $(x_1,x_2) \ne (-1,-1)$, it must be the case that $(1-x_1)(1-x_2)=0$ (this ensures that either $x_1$ or $x_2$ must be equal to $1$). Thus, the functions $f(x)$ or $f(x)+\Delta(x) (1-x_1)(1-x_2)$ (for any arbitrary $\Delta$) are equivalent on the training data. One can thus wonder which $\Delta$ function will a neural network trained by (S)GD converge to. There is no reason to expect that it will converge to $\Delta=0$; so can we characterize which $\Delta$ will occur? 

Our main results show that ---(i) provably for random features model or diagonal/2-layer linear networks in the linear case (three architectures that we can analyze rigorously), and (ii) empirically for encoder-only Transformers --- (S)GD will converge to a $\Delta$ that makes  $f(x)+\Delta(x) (1-x_1)(1-x_2)$ having the lowest `degree-profile' (see Definition \ref{def:degree-prof}), which in the above majority example is obtained as follows: first expand the target in the basis of multivariate monomials, $\mathrm{Maj}(x_1,x_2,x_3)=(x_1+x_2+x_3-x_1x_2x_3)/2$, then find $\Delta(x)$ that makes $(x_1+x_2+x_3-x_1x_2x_3)/2+\Delta(x) (1-x_1)(1-x_2)$ having the least $\ell_2$ mass on the highest degree monomials, i.e., in this case, $\Delta(x)=x_3/2$, giving $(x_1+x_2+x_3-x_1x_2x_3)/2+\Delta(x) (1-x_1)(1-x_2)=(x_1+x_2+2x_3-x_1x_3-x_2x_3)/2$ which is degree 2 rather than 3 (see Figure \ref{fig:f4-all} for numerical experiments). This paper describes what are the general mathematical concepts behind this specific example: (i) Fourier-Walsh Boolean analysis, (ii) the notion of vanishing ideal, and (iii) minimal degree-profile interpolators and the implicit bias of neural networks towards them. To begin with, we formalize the concept of unseen domains using the following definition.

\begin{definition}
\label{def:sparse-setting}
We consider a $P$-dimensional latent function $h\colon\{\pm 1\}^P \to \mR$ embedded in ambient dimension $d$. More precisely, we consider learning $f\colon\{\pm 1\}^d \to \mR$ such that $f(x) = h(x_{i_1}, \ldots, x_{i_P})$. We further denote $I = \{i_1, \ldots, i_P\}$ and $x_I = (x_{i_1}, \ldots, x_{i_P})$. We also assume that some specific combinations of $x_I$ are not present in the training samples, i.e., $x_I \notin \mathcal{U}^{*} \subset \{\pm 1\}^P$ and define the unseen domain as $ \mathcal{U}=\{x \in \{\pm 1\}^d\mid x_I \in  \mathcal{U}^{*}\}$.
\end{definition}
Note that considering sparse functions enables us to define the unseen domain properly and differentiate between the unseen domain (where there are minimal structures) and unseen data (for example when there is uniform sampling). In the next section, we present our results on learning sparse Boolean functions with the random features model.

\subsubsection{Results for Random Features Model}
Our first result is for the random features (RF) model \citep{rahimi2007random}. The RF model was initially introduced to approximate kernels and enhance the time complexity of kernel methods \citep{rahimi2007random}. RF models can also be viewed as approximations of neural networks in the NTK regime  \citep{jacot2018neural,ghorbani2019limitations, mei2022generalization}. In this paper, we take the latter view on them as well, with the following formulation.
\begin{definition}[Random features model]\label{def:random-features-model}
Consider $x \in \mathbb{R}^d$ as the input; we define random features model with $N$ random features as 
\begin{equation}
    {f}_{\mathrm{RF}}(x;a, w, b) = \frac{1}{\sqrt{N}}\sum_{i=1}^{N} a_i\sigma(\langle w_i, x \rangle + b_i), 
\end{equation}
where $a_i \in \mathbb{R}$ are the trainable parameters, $\sigma$ is the activation function, and $w_i, b_i \sim \mathcal{N}(0, \frac{1}{d})^{\otimes d} \otimes \mathcal{N}(0, \frac{1}{d})$ are the random weights and biases. We use $\phi_{i}(x) \coloneqq \sigma(\langle w_i, x \rangle + b_i)$ as a shorthand notation for the $i$-th feature.
\end{definition}
The following activation property strengthens the condition presented in the work of \citet{AbbeINAL}.  
\begin{definition}[Strongly expressive]\label{def:strongly-expressive} We call a continuous activation function $\sigma\colon\mathbb{R} \to \mathbb{R}$ strongly expressive up to $P$ if (A1) $\sigma$ satisfies upper bound $\E_{g\sim \mathcal{N}(0,2)}[\sigma(g)^4] < \infty$; and (A2) $\forall T \subseteq [d], |T| \leq P\;  \E_{w, b}[\hat\phi_{w,b}(T)^2] = \Omega_d(d^{-|T|})$, where $\hat\phi_{w,b}(T) \coloneqq \E_{x}[\sigma(\langle w, x \rangle + b)\chi_T(x)]$ is the Fourier coefficient of $T$ in the random feature created by $w, b$. 
\end{definition}
As will be proven in Lemma \ref{lemma:random-low-deg}, property (A1) implies $\E[\hat\phi_{w,b}(T)^2] = O(d^{-|T|})$ for $|T|=O_d(1)$. Therefore, the second condition (A2) is ensuring that the model is able to strongly express degree $k \leq P$ monomials.

We note that $\hat\phi_{w,b}(T)^2$
has been studied in the work of \citet{AbbeINAL} as the initial alignment (INAL) between monomial $\chi_T(x)$ and $\phi_{w, b}(x)$. Indeed, based on Lemma A.2. of \citet{AbbeINAL}, the following conditions give us a family of strongly expressive activation functions. 
\begin{lemma}\label{lemma:hermite-nonzero}
    Any continuous polynomially-bounded function $\sigma$ such that its first $P$ coefficients in the Hermite expansion are non-zero is strongly expressive up to $P$.
\end{lemma}
For example, polynomial activation functions such as $(1+x)^k$ are  strongly expressive up to $k$. 
\begin{theorem}
\label{thm:random-features}
    

Let $f\colon\{\pm 1\}^d \to \mathbb{R}$ be a $P=O_d(1)$-sparse function to be learned in the GOTU setting (Definition \ref{def:sparse-setting}) by a random features model with parameters $(N,\sigma,a,b,w)$ (Definition \ref{def:random-features-model}) with a strongly expressive activation function. As $N$ diverges, the random features model can interpolate the training data with high probability. Furthermore, defining ${f}_{\mathrm{RF}}^{d, N} (\mathcal{U})$ to be the interpolating solution minimizing $\|a\|_2$ (i.e., the solution reached by gradient descent/flow starting from $a=0$ under $\ell_2$ loss), we have w.h.p.\ 
\begin{align}
  {f}_{\mathrm{RF}}^{d, N} (\mathcal{U}) \stackrel{N \to \infty}{\longrightarrow} \mathrm{MinDegInterp}(f,\mathcal{U})  + \epsilon_{d} 
\end{align}
where $\mathrm{MinDegInterp}(f,\mathcal{U})$ is the min-degree interpolator on the training data $\{x,f(x)\}_{x \in \mathcal{U}^c}$ and $\epsilon_{d}$ is a function on $P$ variables that tends pointwise to 0 as $d$ diverges. (We refer to the above as a `min-degree bias' or `MD bias'.) 
\end{theorem}
\emph{Proof Sketch.} In Lemma \ref{lemma:random-low-deg}, we show that random features generated by a strongly expressive $\sigma$ have in general a decaying degree-profile with 
$\E_{w,b}[\hat\phi_{w,b}(T)^2] = \Theta(d^{-|T|})$ for $|T| \leq P$. We then investigate the interpolators in the Fourier-Walsh basis and show that the minimality condition of $\|a\|_2$ is equivalent to learning the minimal degree-profile interpolator since high-degree monomials are less expressed in the features and consequently larger $\|a\|$'s are required to capture them. The full proof relies on concentration results and Boolean Fourier analysis and is given in Appendix \ref{app:proofs}.
\begin{corollary}\label{cor:coef-bound}
    For a random features model with an activation function that satisfies property (A1) in Definition \ref{def:strongly-expressive} and with diverging number of neurons ($N \to \infty$), we have $\E_x[f_{\mathrm{RF}}(x;a,w,b)\chi_T(x)]^2 = O_d(\frac{\|a\|^2}{d^{|T|}})$ for $|T| = O_d(1)$. Put simply, the coefficient of degree $k=O_d(1)$ monomials is bounded by $O(\frac{\|a\|}{\sqrt{d}^{k}}).$
\end{corollary}
This corollary is also proved in Appendix \ref{app:proofs}. As a result, to learn a solution of degree $k=O_d(1)$, the squared norm of the model's weights, $\|a\|^2$, must be at least $\Omega(d^k)$.

\begin{remark}[Other activation functions]
Note that Theorem \ref{thm:random-features} does not hold for any arbitrary activation function. For example, if the activation function is $\sigma(z) = z^2$, one can easily see that 
$\E_{w,b}[\hat\phi_{w,b}(x)(\{i\})^2], \E_{w,b}[\hat\phi_{w,b}(x)(\{i, j\})^2] \in \Theta_d(d^{-2})$, and hence degree-1 monomials have no priority over degree-2 monomials. 
 An important case is the ReLU activation. Results of \citet{AbbeINAL} show that for the ReLU activation and $|T| =O_d(1)$, we have
\begin{equation}
\E_{w,b}[\hat\phi_{w,b}(T)^2] =
  \begin{cases}
     \Omega(d^{-|T|}) & |T| \;\mathrm{even\;or\;} |T|=1\\
     \Omega(d^{-|T|-1}) & \mathrm{otherwise}
\end{cases}.  
\end{equation}
Consequently, the min-degree bias still exists, but in a weaker form. For further discussion and experiments on the ReLU activation refer to Appendix \ref{app:proofs}. 
\end{remark}

In the experiments, we show that having the sparsity assumption may not be necessary in some cases, and the min-degree bias can be observed for small values of $d$ and $N$ as well. Furthermore, we show that the min-degree bias goes beyond the random features and NTK models; see Section \ref{sec:exps}. 

We next focus on linear neural networks where we will be able to analyze non-linear dynamics for gradient flow.
\subsubsection{Results for Linear Neural Networks}
In this section, we study the min-degree bias in linear neural networks. First, note that in the case of linear functions, replacing a degree-1 variable $x_k$ with the degree-0 variable $1$ is the only case of lower degree bias. In other words, we consider the case that unseen data is $\mathcal{U}=\{x\mid x_k = -1\}$ (referred to canonical holdout in work of \citealp{abbe2022learning}). We conjecture that linear neural networks trained with gradient flow learn the min-degree interpolator with a leakage factor (coefficient of the frozen variable in the learned function) that vanishes as their initialization scale goes to $0$. We prove this conjecture formally for diagonal linear neural networks and two-layer fully connected networks. Further, we discuss how the proof can be generalized and provide experiments to support the conjecture. 

We start with the simpler case of diagonal linear neural networks with bias. We define them as follows.
\begin{definition}[Diagonal linear neural network with bias]
	We define a diagonal linear neural network (DLNN) with bias as an extension of diagonal neural networks, where there is only one parameter for bias at the last layer. I.e., 
	\begin{align*}
	\theta &= (b, w_1^{(1)}, \ldots, w_d^{(1)}, \ldots, w_1^{(L)}, \ldots, w_d^{(L)}), \\
	f_{\mathrm{NN}}&(x_1, \ldots, x_d; \theta) = b + \sum_{i=1}^d \left( \prod_{l=1}^{L} w_{i}^{(l)} \right) x_i,
\end{align*}
where $\theta$, $d$, and $L$ represent the model's parameters, input dimension, and depth, respectively. 
\end{definition}
\begin{theorem} \label{thm:diagonal}
Let $f\colon \{\pm 1\}^d \to  \mathbb{R}$ be a linear function $f(x_1, \ldots, x_d) = \hat f(\emptyset) + \sum_{i=1}^d \hat f(\{i\})x_i$. Consider learning this function using gradient flow on a diagonal neural network (where depth $L \geq 2$) while the $k$-th component is frozen at training (the canonical holdout setting with $\mathcal{U}=\{x \in \{\pm 1\}^d \mid x_k=-1\}$). For any $\epsilon > 0$, there exists an $\alpha_{max}$ (increasing with $L$) such that if all the model's parameters are initialized i.i.d.\ under the uniform distribution $U(-\alpha, \alpha)$ for any $0 < \alpha \leq \min\{\alpha_{\max}, \frac 12\}$, then, with probability 1, the training loss converges to 0, and the coefficient of the learned function ${f}_{\mathrm{NN}}$ on the high-degree monomial $x_k$ is less than $\epsilon$, i.e., $\hat{f}_{\mathrm{NN}}(\{k\}) \leq \epsilon$.
\end{theorem} 
\textit{Proof Sketch.} We prove this theorem by analyzing the trajectory of gradient flow on the parameters. Primarily, we show the convergence of the model. Note that $\hat{f}_{\mathrm{NN}}(\{k\}) \leq \epsilon$ is equivalent to $x_k$ being ignored by the neural network, i.e., the frozen variable $x_k$  not contributing to the bias learned by the neural network. We pursue the proof in two steps. As the first step, we show there exists a time $T_\epsilon$ such that the bias is almost learned by the bias parameter and the rest of the parameters and the role of $x_k=1$ are still small (note that this point is close to a saddle). For the second step, we show that the contribution of $x_k=1$ to the bias will not change much throughout the training process.

Next, we focus on the general case of fully connected linear neural networks. We parameterize them with the definition below. 
\begin{definition}\label{def:fclnn}
    We define a fully connected linear neural network  with depth $L$ as follows
    \begin{align*}
        \theta &= (b_1, b_2, \ldots, b_{L}, W_1, W_2, \ldots, W_{L-1}, w_L), \\
        f_{\mathrm{NN}}&(x_1, \ldots, x_d;\theta) = w_L^T(W_{L-1}^T(\cdots (W_2^T(W_1^T x + b_1) + b_2)\cdots) + b_{L-1}) + b_L,
    \end{align*}
    where $b_{i} \in \mathbb{R}^{d_i}, W_i \in \mathbb{R}^{d_{i-1} \times d_i}$ are the weights and biases of the $i$-th layer. Note that $d_0 = d$ is the input dimension and $d_L=1$ is the output dimension. Also, sometimes we slightly abuse the notation by referring to the last layer's weight vector by $W_L = w_L$.
\end{definition}
Now again assume that $x_k$ is frozen to $1$ during training. Denote the first layer's weights connected to $x_k$ by $W_{1, k}$. One can easily see that the weights connected to the frozen coordinate act exactly similar to the biases of the first layer and they follow the same dynamics. More precisely, $\nabla_{W_{1, k}} L(t) = \nabla_{b_1} L(t)$, where $L(t)$ is the loss function. As a result, they have the same updates in gradient descent or gradient flow. Particularly, the weights incident to $x_k$ at time $t$ satisfy $W_{1, k}(t) = b_1(t) + (W_{1, k}(0)-b_1(0))$. As a result, to show that the weight of $x_k$ in the function learned by the linear neural network is negligible, it is enough to show that the role of the first layer's bias is negligible. Accordingly, we propose the more general conjecture below. We will also formalize this equivalence in Remark \ref{remark:equivalence-linear}.

\begin{conjecture}\label{conj:linear}
Consider training a depth $L \geq 2$ fully connected linear neural network defined in Definition \ref{def:fclnn}  with the $\ell_2$ loss function 
$$L(\theta) = \frac{1}{2}\left(\|W_1W_2\cdots W_{L-1}w_{L} - w^*\|^2 + \left(b_L + w_L^T b_{L-1} + \cdots  + w_L^TW_{L-1}^T \cdots W_2^Tb_1 - b^* \right)^2\right)$$
with gradient flow, i.e., $\dot \theta = \frac{d\theta}{dt} = -\nabla_\theta L(\theta)$. Further, assume the neural network is initialized with $W_i(0) = \alpha \overline{W_i}$ and $b_i(0) = \alpha \overline{b_i}$ for $1\leq i\leq L$ where $\alpha$ is the initialization scale, and $\overline{W_i}$,$\overline{b_i}$ are initial directions independent from $\alpha$. For technical reasons, we also allow the first layer's bias to have a different speed, i.e., $\dot b_1 = -\gamma \nabla_{b_1} L(\theta)$ where $\gamma > 0$ is a constant. We conjecture that there exists $0 < \epsilon = o_\alpha(1)$ such that if the model is initialized with scale $\alpha$ then the contribution of first layer's bias would be smaller than $\epsilon$, i.e., $\|b_1\|, |w_L^TW_{L-1}^T \cdots W_2^Tb_1| \leq \epsilon$. Moreover, $\|W_i\|_F$ remains $O_{\alpha}(1)$ for all layers during training.
\end{conjecture}
The intuition for this conjecture is that the updates for the last layer's bias are of constant order at initialization, i.e.,  $\dot b_{L} = \theta_{\alpha}(1)$. On the other hand, the update of the rest of the parameters scale with $\alpha$ at initialization, i.e., $\dot W_i = O(\alpha)$ and $\dot b_i = O(\alpha)$ for $i \neq L$. As a result, the bias of the neural network is first learned by the bias of the last layer. We conjecture that this picture will not change much as training continues. We prove this formally for a two-layer linear neural network with appropriate initialization in Theorem \ref{thm:2-linear}. We also provide experimental evidence for this conjecture in Appendix \ref{app:exps}.

\begin{theorem} \label{thm:2-linear}
     Consider Conjecture \ref{conj:linear} for two-layer neural networks. We prove that for $w^*\neq 0$ if the initialization satisfies $\|\overline{w_2}\|^2 > \|\overline{W_1}\|^2 + \gamma^{-1}\|\overline{b_1}\|^2$ then the statement of the conjecture holds. I.e., we prove an upper bound for $\|b_1\|$ and $ |w_2^Tb_1|$ which vanishes as the initialization scale $\alpha$ goes to zero while $\|W_1\|_F, \|w_2\|$ remain $O_\alpha(1)$ during training.
\end{theorem}
\textit{Proof Sketch.} We prove this theorem by analyzing the trajectory of gradient flow on the parameters in three phases. In the first phase, we prove that the bias of the last layer learns the total bias of the target while all other parameters remain small as the updates for the bias of the last layer are of the constant order compared to the updates of other parameters which are of order $O(\alpha)$. In the second phase, we prove that $w_2$ and $W_1$ will reconstruct the weight term $w^*$ while $\|b_1\|$ remains small. At this point, both the loss function and $\|b_1\|$ are small. For the last phase, we prove that the parameters would not change much after this point. So the bias parameter and its contribution remain negligible. 
The full proof is presented in Appendix \ref{app:proofs}. The condition $\|\overline{w_2}\|^2 > \|\overline{W_1}\|^2 + \gamma^{-1}\|\overline{b_1}\|^2$ on the initialization is only needed for the analysis of the trajectory in the second phase of our proof and is not uncommon (e.g., see a similar one in the results of \citealp{yun2020unifying}).

One can easily see that phase 1 of the proof can be generalized to any neural network in Conjecture \ref{conj:linear}. Also, it is possible to generalize phase 3 to deeper neural networks (assuming that the neural network reaches the starting point of this phase). Phase 2, on the other hand, may require more technical work for deeper networks. Particularly, one may need to break this phase into several steps and analyze the trajectory in each of these steps. 

\begin{remark}
    Note that even for a constant initialization scale, as depth $L$ grows, there are $L$ terms that reconstruct the bias. One could thus expect that the role of the first layer's bias to decrease in this reconstruction as $L$ grows. This is an independent phenomenon that we do not capture in Conjecture \ref{conj:linear} where we have the vanishing initialization scale. 
\end{remark}
\begin{remark}\label{remark:equivalence-linear}
    Again, consider the original problem in which $x_k=1$ is frozen. We noted that the weights incident to $x_k$, $W_{1,k}$ work exactly in the same manner as the bias parameter of the first layer $b_1$. Indeed, one can define an equivalent bias parameter $\tilde b_1 = b_1 + W_{1,k}$ and remove $x_k$ from the coordinates. The only caveat is that this parameter would have a speed of $\gamma=2$ for the updates (this is also the reason that we introduced $\gamma$ in Conjecture \ref{conj:linear}). We can simply check that if Conjecture \ref{conj:linear} is satisfied (e.g., as in Theorem \ref{thm:2-linear}), then the frozen bit will be ignored showing the min-degree bias. We prove this more formally in Appendix \ref{app:proofs} as well.   
\end{remark}
\begin{remark}
Note that with the assumptions of Theorem \ref{thm:diagonal} or Conjecture \ref{conj:linear}, the generalization error of the model becomes\footnote{The factor 4 is removed if we consider the half-quadratic loss and GOTU on the full space.} $$GOTU(f, f_{\mathrm{NN}}, \mathcal{U}=\{x: x_k =-1\}) = 4\mathrm{Inf}_k(f) + O(\epsilon),$$ where $\mathrm{Inf}_k(f) = \hat{f}(\{k\})^2$ is the Boolean influence of the $k$-th bit \citep{o'donnell_2014}. This confirms the empirical observations of \citet{abbe2022learning} on fully connected linear neural networks.
\end{remark}
\section{Experiments}\label{sec:exps}
In this section, we present our experimental results on the min-degree bias of neural networks.\footnote{Code: \url{https://github.com/aryol/GOTU}} We have used two architectures for our experiments in this part: a random features model (Definition \ref{def:random-features-model}) and an encoder-only Transformer \citep{vaswani2017attention-transformer}. We show that both these architectures have a very strong min-degree bias and would learn min-degree interpolators. We will also consider a multi-layer perceptron (MLP) with 4 hidden layers and a 2-layer neural network with mean-field parametrization \citep{mei2018mean} in Section \ref{sec:other-archs}. We will show that these architectures also follow the min-degree bias although in a weaker way and possibly with a leakage. By doing this, we consider a spectrum of models covering lazy regimes, active/feature learning regimes, and models of practical interest. For the Transformer, we use an encoder-only architecture with bidirectional attention, absolute learnable positional embeddings, and a classification token. Also, $\pm 1$ bits are first encoded using an encoding layer and then passed to the Transformer; while for the rest of the architectures, binary vectors are directly used as the input. Note that the input in our tasks is always of fixed size and does not contain any causal structures. Also, the output is continuous and 1-dimensional. This makes encoder-only Transformers the most natural choice in the Transformers family. Nonetheless, we explore the properties of Transformers with causal attention in Section \ref{sec:causal-transformer}.


For each experiment, we generate all binary sequences in $\mathcal{U}^{c} = \{\pm 1\}^d \setminus \mathcal{U}$ for training.\footnote{In practice, one can generate a large enough number of samples so that the function is learned well on the training distribution.} We then train models under the $\ell_2$ loss.
We employ Adam \citep{kingma2014adam} optimizer for the Transformer model and mini-batch SGD for the rest of the architectures. We also use moderate learning rates as learning rate can affect the results (refer to Appendix \ref{app:learning-rate-sensitivity}).  During training, we evaluate the coefficients of the function learned by the neural network using $\hat{f}_{\mathrm{NN}}(T) = \mathbb{E}_{x \sim_U \{\pm 1\}^d}[\chi_T(x)f_{\mathrm{NN}}(x)]$ to understand which interpolating solution has been learned by the model.  Moreover, each experiment is repeated 10 times and averaged results are reported. For more information on the setup of experiments, hyperparameter sensitivity analysis, and additional experiments refer to Appendix~\ref{app:exps}.

Here, we consider the following 3 functions and unseen domains on input dimension $15$. Dimension $15$ is used as a large dimension where the training data can be generated explicitly but has otherwise no specific meaning (Appendix~\ref{app:exps} provides other instances). The first function is an example of degree-2 where the unseen domain induces a degree-1 MD interpolator. The second example is the classic degree-2 parity or XOR function. 
The third example is such that the function is symmetric under cyclic permutations while its MD interpolator is not, in order to test whether certain models would favor symmetric interpolators. 
We consider other examples such as the majority function in Appendix~\ref{app:exps}.
Let:
\begin{enumerate}
    \item $f_1(x) = x_0x_1 - 1.25x_1x_2+1.5x_2x_0$ and $\mathcal{U}_1 = \{x_0x_1x_2 = -1\}$. In this case, we have $x_0x_1 = x_2$, $x_1x_2=x_0$, and $x_2x_0=x_1$ at training, hence the MD interpolator is $\tilde{f_1}(x) = x_2 - 1.25x_0 + 1.5x_1$.
     \item $f_2(x) = x_0x_1$ and $\mathcal{U}_2 = \{(x_0, x_1) = (-1, -1)\}$. Note that the MD interpolator is $\tilde{f_2}(x)= x_1 + x_0 -1$ for the seen domain. 
    \item $f_3(x) = x_0x_1x_2 + x_1x_2x_3 + \cdots + x_{13}x_{14}x_{0} + x_{14}x_0x_1$ and unseen domain $\mathcal{U}_3 = \{(x_0, x_1, x_2) = (-1, -1, -1)\}$. In this case, the MD interpolator is given by $\tilde{f_3}(x) = (x_0x_1+x_1x_2+x_2x_0 - x_0-x_1-x_2 + 1) + x_1x_2x_3 + \cdots + x_{13}x_{14}x_{0} + x_{14}x_0x_1$.
\end{enumerate}
We generally obtain that the encoder-only Transformer exhibits a strong MD bias similar to the random features model.\footnote{Note that the RF model in Figure~\ref{fig:transformers} has a small leakage, simply caused by the ambient dimension being $d=15$ and not diverging as in Theorem \ref{thm:random-features}.} The solutions learned by the Transformer and the RF model for $f_1$, $f_2$, $f_3$ are shown in Figure~\ref{fig:transformers}. It can be seen that these are very close to the MD interpolator in all cases.  
We will try learning the same examples with an MLP and a mean-field model in Section \ref{sec:other-archs}.

\begin{figure*}[ht]
     \centering
     \begin{subfigure}[b]{0.29\textwidth}
         \centering
         \includegraphics[width=\textwidth]{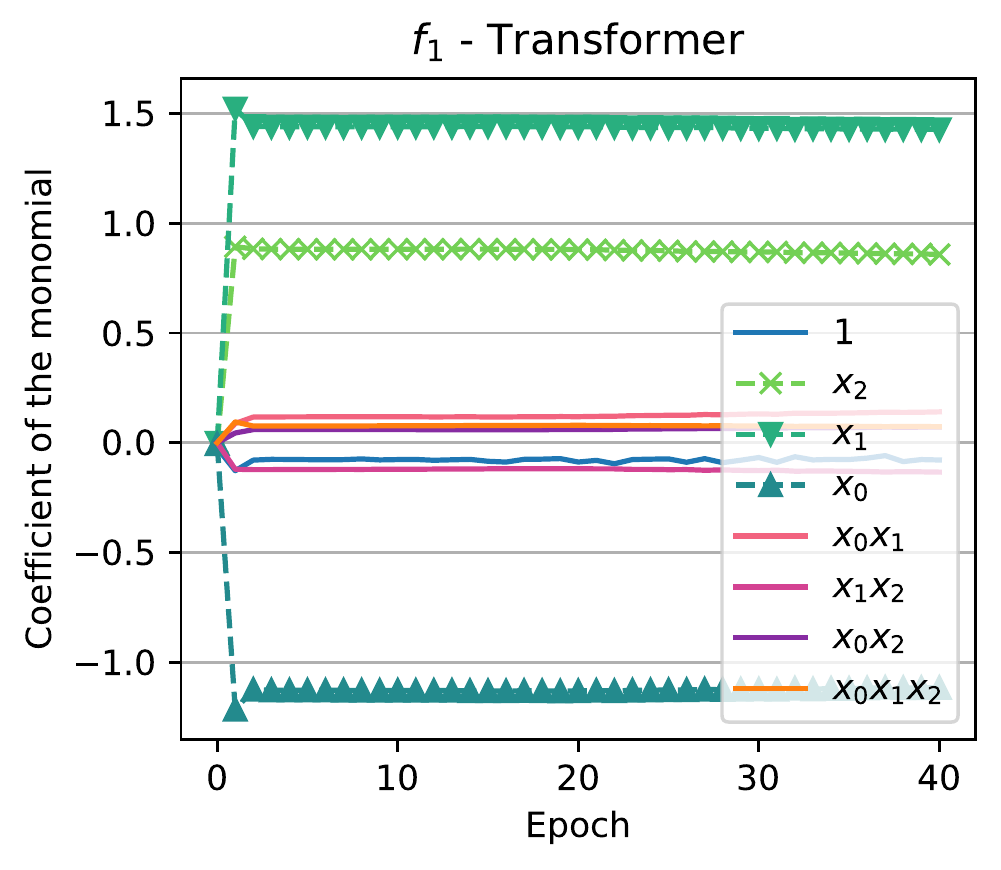}
     \end{subfigure}
     \hfill
     \begin{subfigure}[b]{0.29\textwidth}
         \centering
         \includegraphics[width=\textwidth]{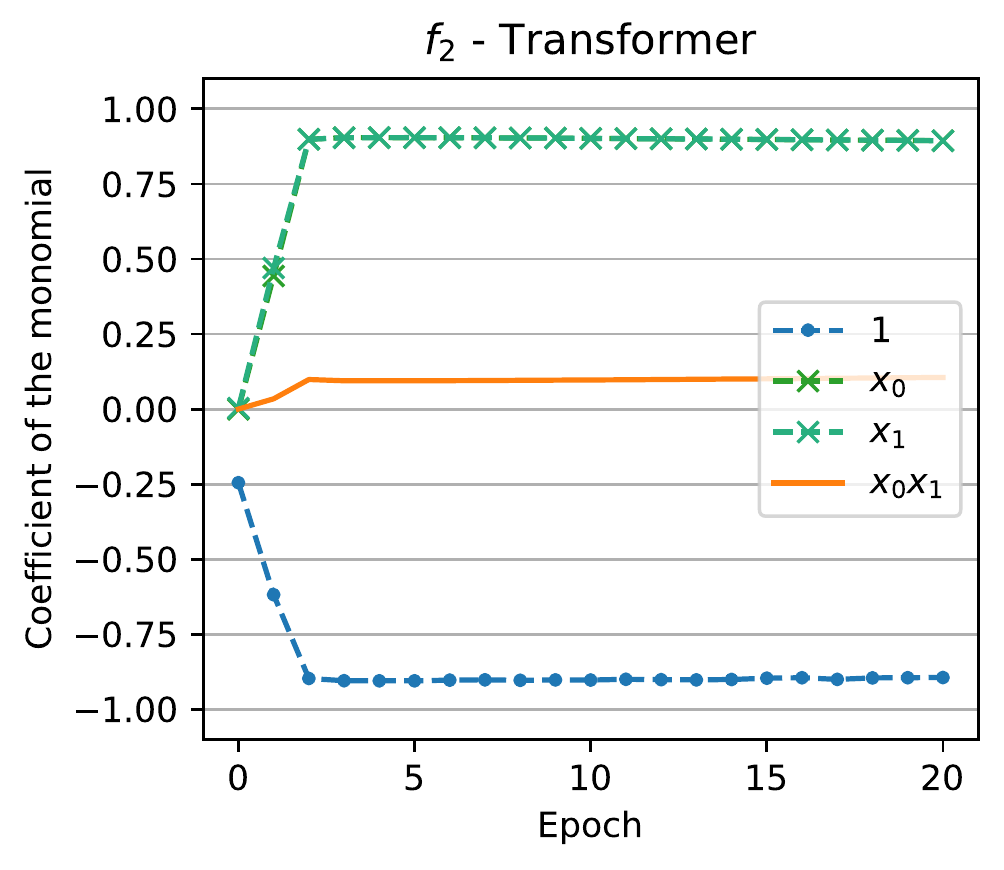}
     \end{subfigure}
     \hfill
     \begin{subfigure}[b]{0.29\textwidth}
         \centering
         \includegraphics[width=\textwidth]{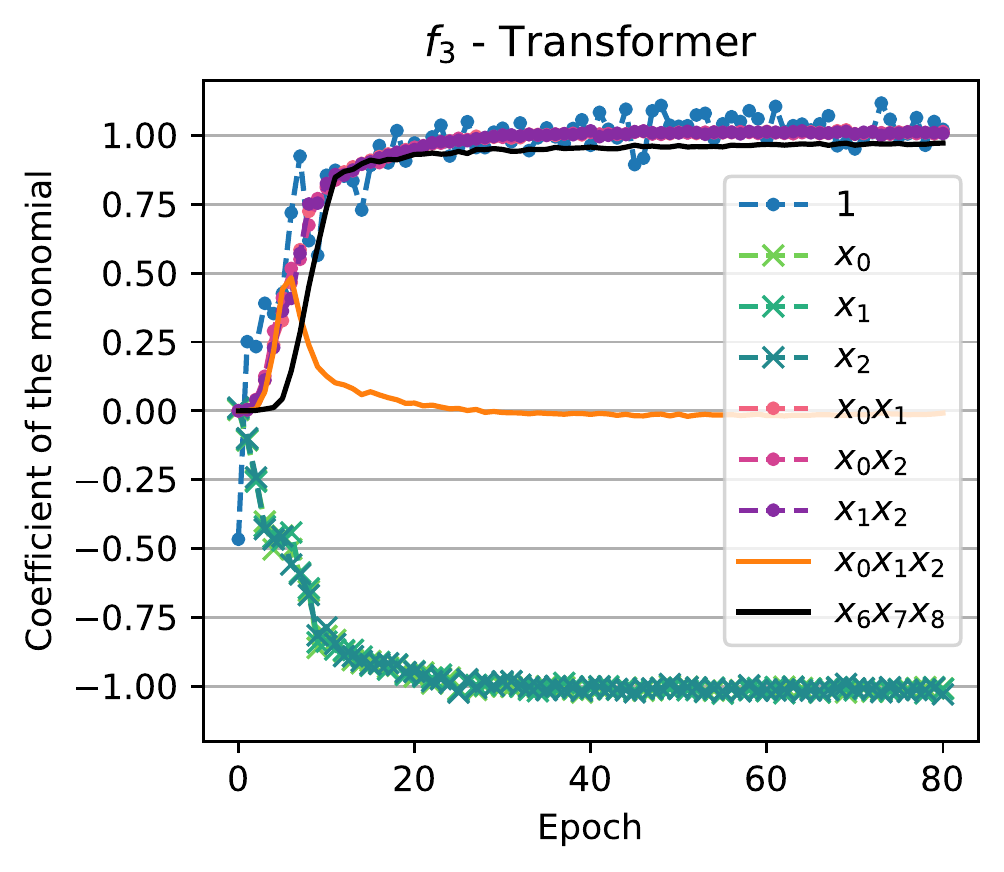}
     \end{subfigure}
      \begin{subfigure}[b]{0.29\textwidth}
         \centering
         \includegraphics[width=\textwidth]{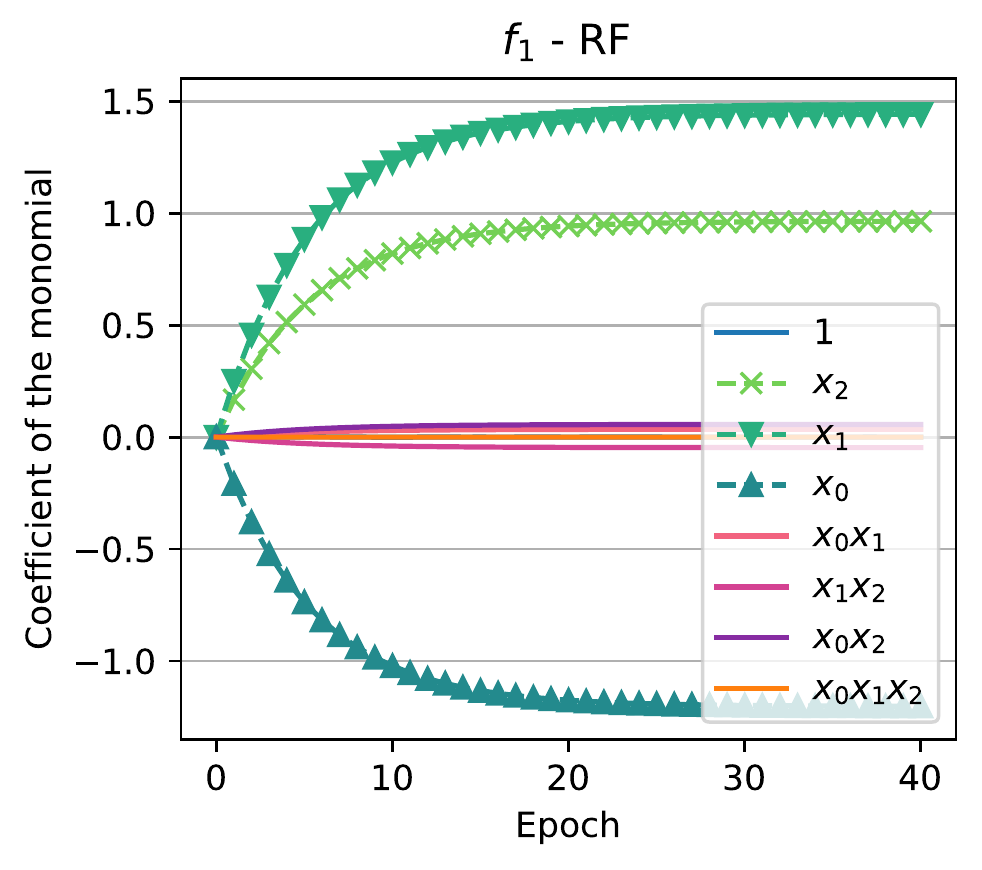}
     \end{subfigure}
     \hfill
     \begin{subfigure}[b]{0.29\textwidth}
         \centering
         \includegraphics[width=\textwidth]{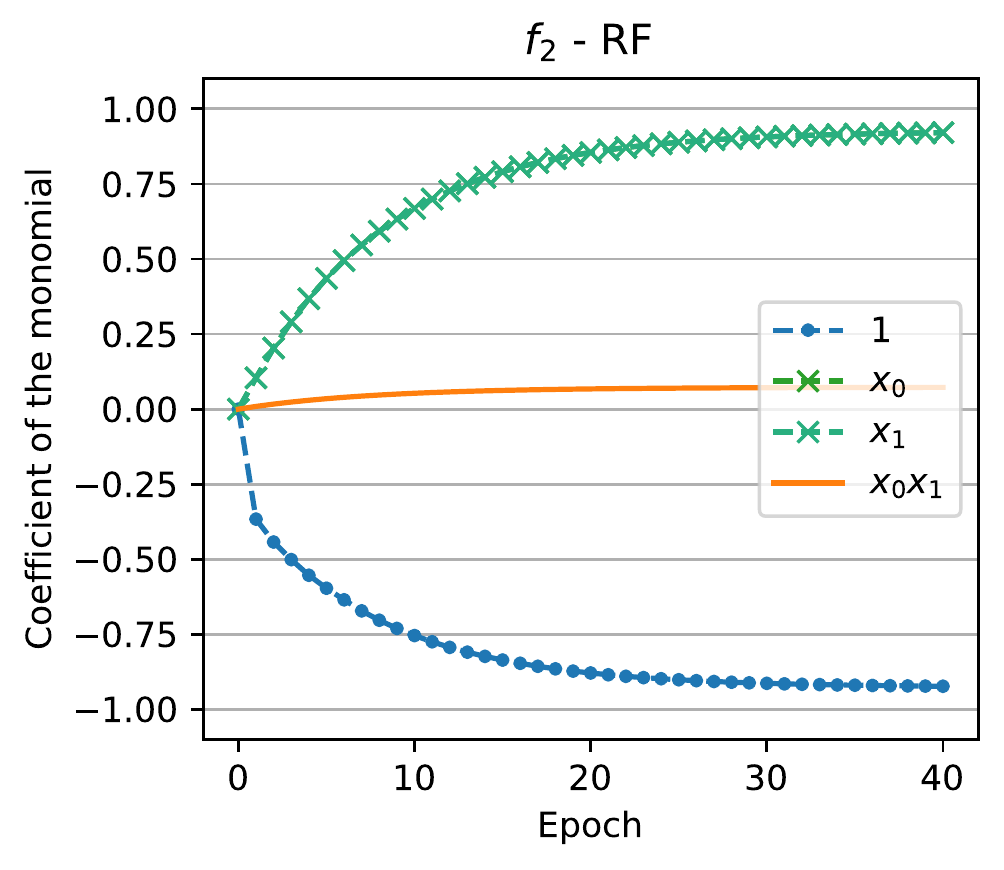}
     \end{subfigure}
     \hfill
     \begin{subfigure}[b]{0.29\textwidth}
         \centering
         \includegraphics[width=\textwidth]{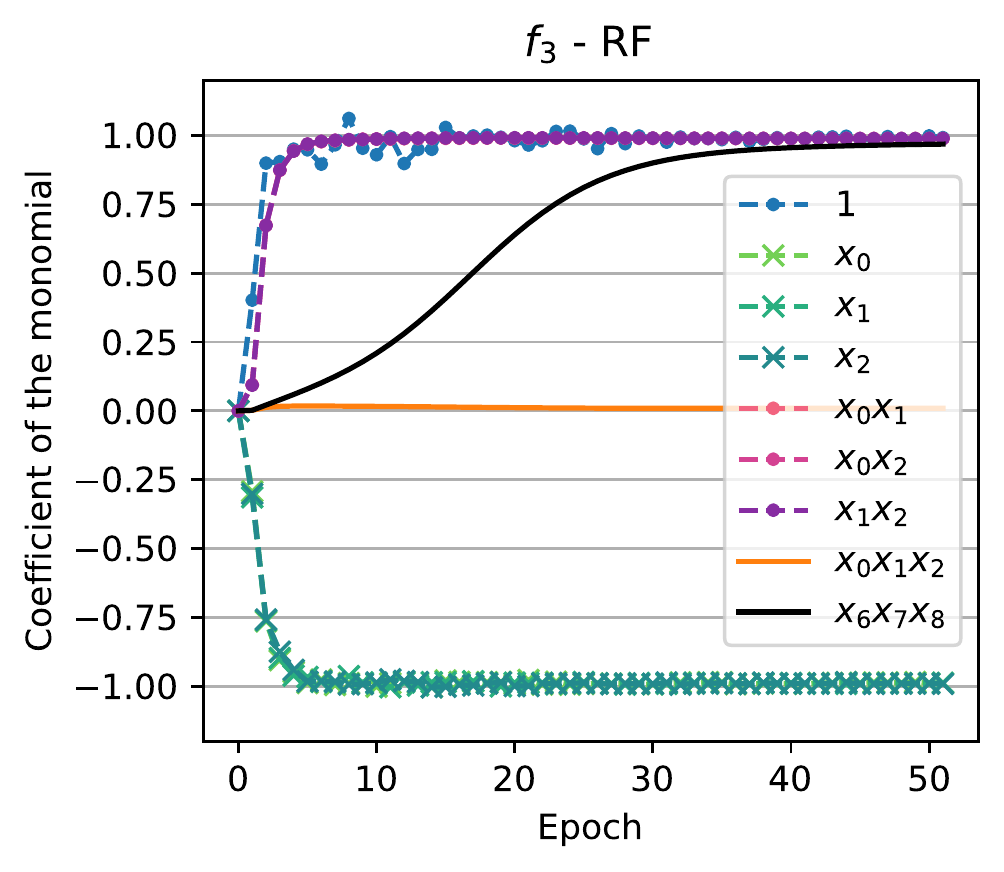}
     \end{subfigure}

        \caption{Target functions $f_1$ (left), $f_2$ (middle), and $f_3$ (right) learned by the encoder-only Transformer (top row) and the RF model (bottom row). Note that in all of the cases, the Transformer and the RF model learn a solution very close to the min-degree interpolator. More precisely, the coefficients of $x_0x_1, x_1x_2, x_2x_0$ in the left plot ($f_1$), the coefficient of $x_0x_1$ in the middle plot ($f_2$), and the coefficient of $x_0x_1x_2$ in the right plot ($f_3$) are close to zero.}
        \label{fig:transformers}
\end{figure*}

\section{Length Generalization}
Several recent works on the reasoning of neural networks evaluate whether neural networks are able to generalize when the length of the problem is increased, and it is often found that neural networks struggle with length generalization \citep{zhang2022unveiling, anil2022exploring-length}. 
For example, consider learning the parity problem
$\mathrm{parity}(x_1, \ldots, x_{d}) = x_1x_2\cdots x_{d}$ on $x_i = \pm 1$. Two variants of this task can be considered: (1) the number of bits, $d$, is increased during test, and (2) $d$ is the same during training and test; however, during training, only samples with a bounded number of $-1$'s are observed, i.e., the radius $r$ Hamming ball $B_r \coloneqq \{x\in\{\pm 1\}^d \mid \#_{-1}(x) \leq r\}$ (note that $+1$ is the identity element in this setting). \citet{anil2022exploring-length} show that both of these variants capture the notion and difficulty of length generalization.\footnote{We train our model directly on the parity function; whereas \citet{anil2022exploring-length} use large language models and fine-tune parity tasks on them. In this sense, our approach is closer to \citet{zhang2022unveiling} who also train models on their synthetic task from scratch.}
Here, we focus on the latter variant which falls under our GOTU setting. 
\begin{theorem}
\label{thm:length-gen}
Consider a Boolean function $f\colon\{\pm 1\}^d \to \mathbb{R}$. Then (i) there exists a unique function $f_r\colon\{\pm 1\}^d \to \mathbb{R}$ such that $\forall x \in B_r, f_r(x) = f(x)$ and $\deg(f_r) \leq r$; (ii) when $f$ is a parity function (monomial) of degree $k \leq d$, the $\ell_2$-test-loss of the MD interpolator is larger than ${\binom{k-1}{r}}^{2}$.
\end{theorem}
\begin{figure}
    \centering
    \includegraphics[width=0.5\textwidth]{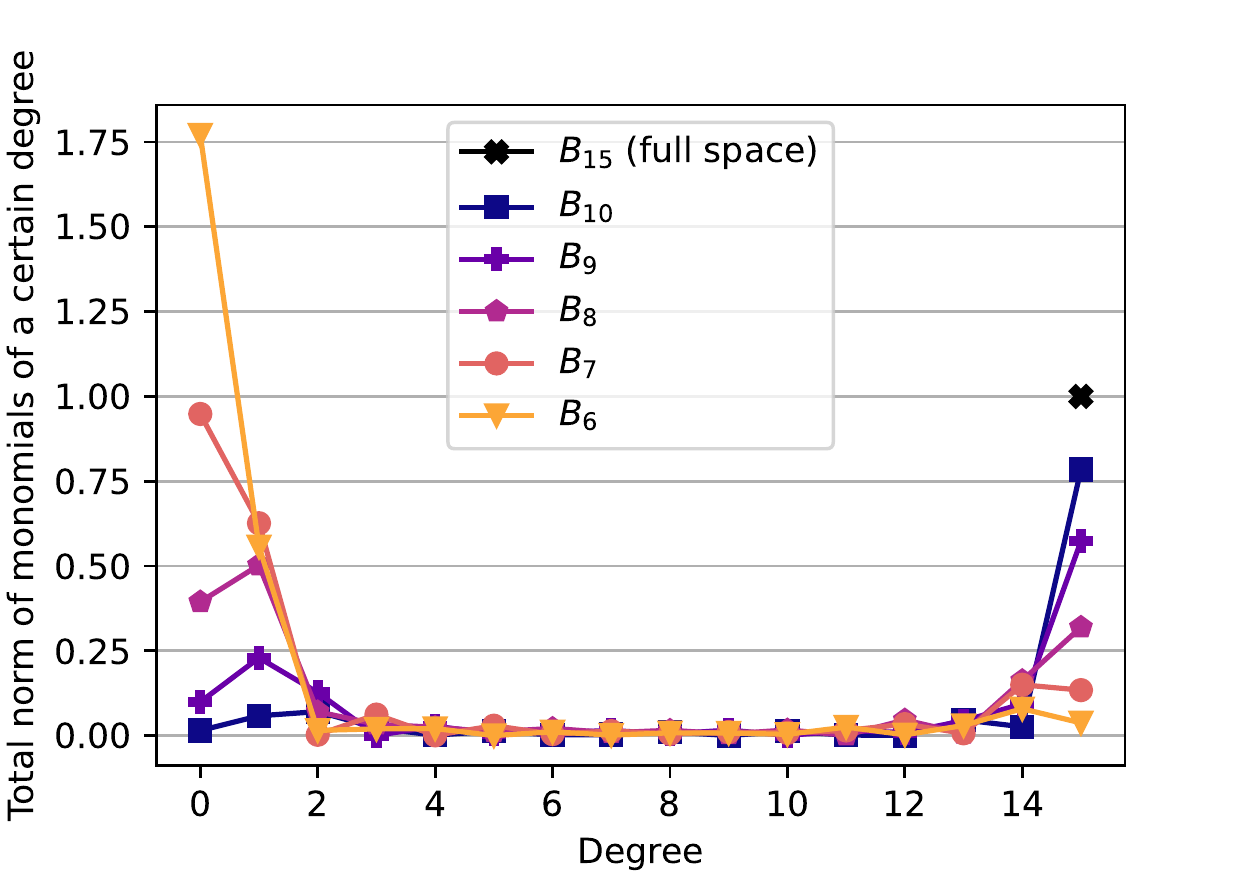}
    \caption{Learning full parity function in dimension
    $d=15$ in the length generalization setting with inputs in $B_6, B_7, B_8, B_{9}, B_{10}$ and $B_{15}$ (full space) respectively, with an MLP (model details in Appendix \ref{app:exps}). X-axis: degree-profile component, Y-axis: degree-profile value, i.e., $\sum_{T:|T|=x}\hat{f}_{\mathrm{NN}}(T)^2$. 
    As the length of training samples is decreased, the coefficient of the full parity gets smaller and the coefficients of low-degree monomials get larger.}
    \label{fig:length-gen}
\end{figure}
We defer the proof to Appendix \ref{app:proofs}. 
Now consider learning the parity function $x_1x_2\cdots x_{d}$ where training samples have $r$ or less $-1$ coordinates, i.e., training samples belong to $B_r$. Using the previous theorem, there is a degree $r$ alternative to $x_1x_2 \cdots x_d$. Note that when such a low-degree alternative exists, assuming the min-degree bias, the model will learn this alternative instead of the full function of degree $d$. This explains why in this case neural networks cannot generalize when the length is increased. We conduct an experiment to evaluate this, where we learn the full parity function on $15$ bits using the MLP model trained on different lengths. Figure~\ref{fig:length-gen} shows that we learn more of lower degree terms and less of the full parity term as we train on shorter lengths.  

\section{Curriculum Learning}\label{curr}
The bias of neural networks towards min-degree solutions can also be used to boost the learning via a curriculum learning \citep{curriculum} algorithm.
We propose to train models by increasing the `complexity' of training samples with respect to the input Hamming weight, i.e., $B_{r_1} \subseteq B_{r_2} \subseteq \ldots \subseteq B_{r_k}$ where $B_r$ is the Hamming ball of radius $r$. Training a model on samples included in $B_{r}$ with $r<d$ produces biased inputs compared to the uniform distribution. It has been shown that learning parities with GD on biased inputs is easier for various architectures    \citep{quantifying,malach_parity}. In particular, the bias in the input distribution can be viewed as converting a monomial on non-centered inputs to a staircase on centered inputs as discussed in the work of \citet{abbe2021staircase}. Moreover, \citet{mergedstaircase} show that the sample complexity for learning staircases is significantly reduced compared to that of monomials of matching degree. In particular, a layer-wise analysis shows that the hidden neurons in the first layer detect the support of a parity function under biased inputs, allowing for the fitting of the target function with the second layer if enough neuron diversity is available. One can thus attempt to bootstrap this approach and progressively climb the support (and degree) of the target function by training successively the network on increasing balls. 
We now develop this approach into a general curriculum algorithm in Algorithm \ref{alg:degree-curriculum}.

\begin{algorithm}[htb]
   \caption{Degree-Curriculum algorithm}
   \label{alg:degree-curriculum}
\begin{algorithmic}
   \STATE {\bfseries Input:} Training samples $S = \{(x_i, y_i)\}_{i=1}^m$; Curriculum $B_{r_1} \subset B_{r_2} \subset \ldots \subset B_{r_k} = B_{d}$; Loss threshold $\epsilon$
   \FOR{$i=1$ {\bfseries to} $k$}
   \STATE $S_{r_i} \coloneqq \{(x,y) \in S | x \in B_{r_i}\}$ (samples in $B_{r_i}$)
   \STATE initialize $\mathrm{train\;loss} = 1+\epsilon$.
   \WHILE{$\mathrm{train\;loss} > \epsilon$}
   \STATE train model with SGD on $S_{r_i}$
   \STATE update $\mathrm{train\;loss}$
   \ENDWHILE
   \ENDFOR
\end{algorithmic}
\end{algorithm}

Note that at the $i$-th step of Algorithm~\ref{alg:degree-curriculum}, all the training samples belong to $B_{r_i}$. Thus, for models obeying the MD bias on the unseen, the model learns the MD interpolator of degree at most $r_i$. 
Further, if the sampling set $S$ is such that $B(r_i) \cap S$ contains enough degree $r_i$ elements, the MD interpolator is of degree $r_i$ --- see Theorem \ref{thm:length-gen}. If one then takes $r_i=r_{i-1}+1$, the new MD interpolator has monomials at step $i-1$ that are contained in those at step $i$, as in the learning of a merged staircase functions \citep{mergedstaircase} (and a lower leap function more generally if one takes a leap in the curriculum degrees). 
  Thus, for a parity target, the  Degree-Curriculum algorithm learns the support sets incrementally as for the implicit staircase function.

We evaluate the Degree-Curriculum algorithm on learning full parity functions of degrees $16$ and $30$, i.e.,   
 $x_0x_1\cdots x_{15}$ and  $x_0x_1\cdots x_{29}$ with an MLP. More precisely, for the same training set and hyperparameters, we once train the MLP with normal SGD and once with the proposed Degree-Curriculum algorithm. We choose curriculum $B_4, B_8, B_{12}, B_{16}$ (leap 4 curriculum) for degree-16 parity and $B_1 \subset B_2 \subset \cdots \subset B_{29} \subset B_{30}$ (leap 1 curriculum) for degree-30 parity. We use loss threshold $\epsilon = 0.001$.  
 The results are depicted in Figure~\ref{fig:curriculum}. It can be seen that the Degree-Curriculum algorithm can reduce the sample complexity for learning parity functions. 
\begin{figure*}
    \centering
     \begin{subfigure}{0.49\textwidth}
        \includegraphics[width=\textwidth]{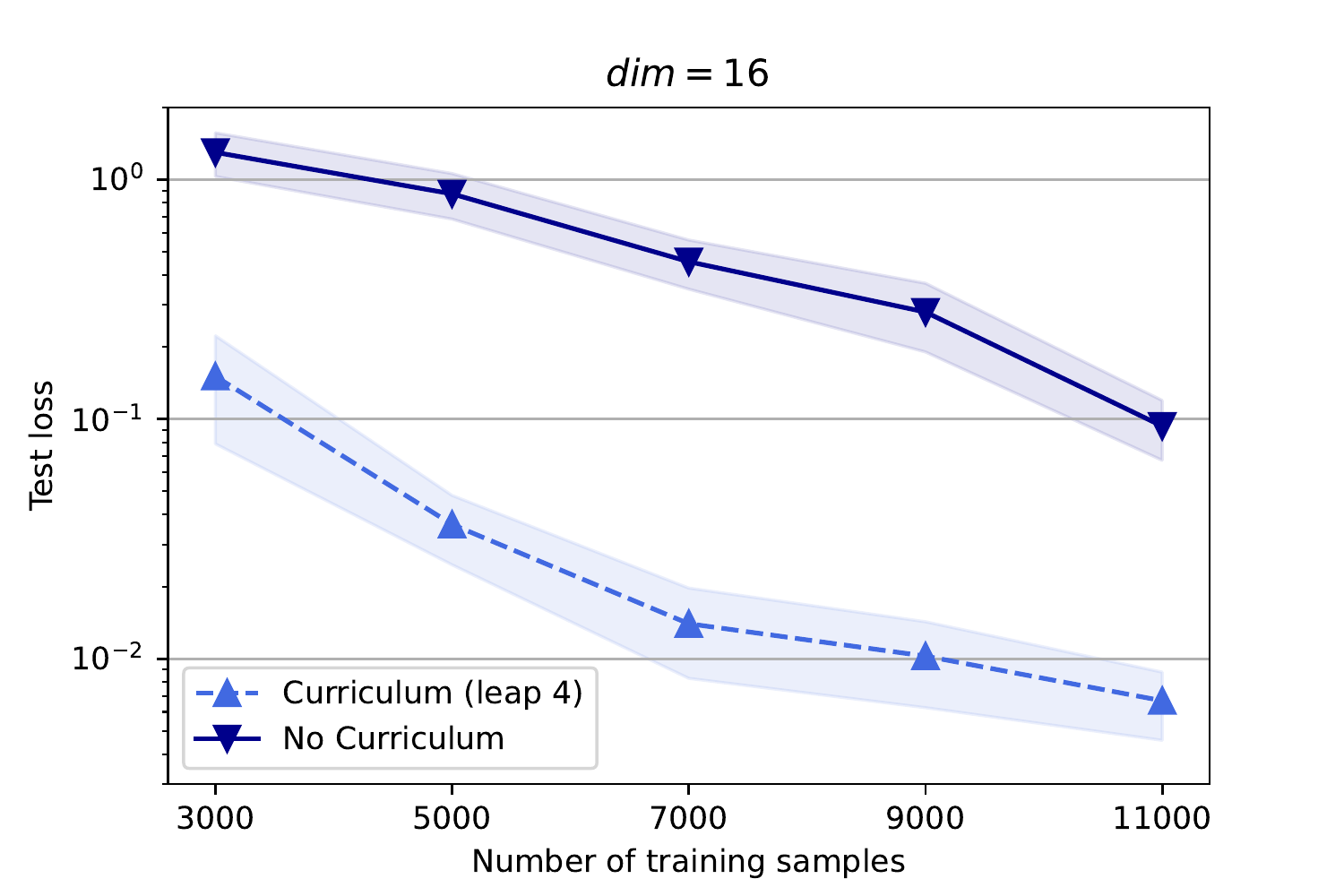}
    \end{subfigure}
    \hfill
    \begin{subfigure}{0.49\textwidth}
        \includegraphics[width=\textwidth]{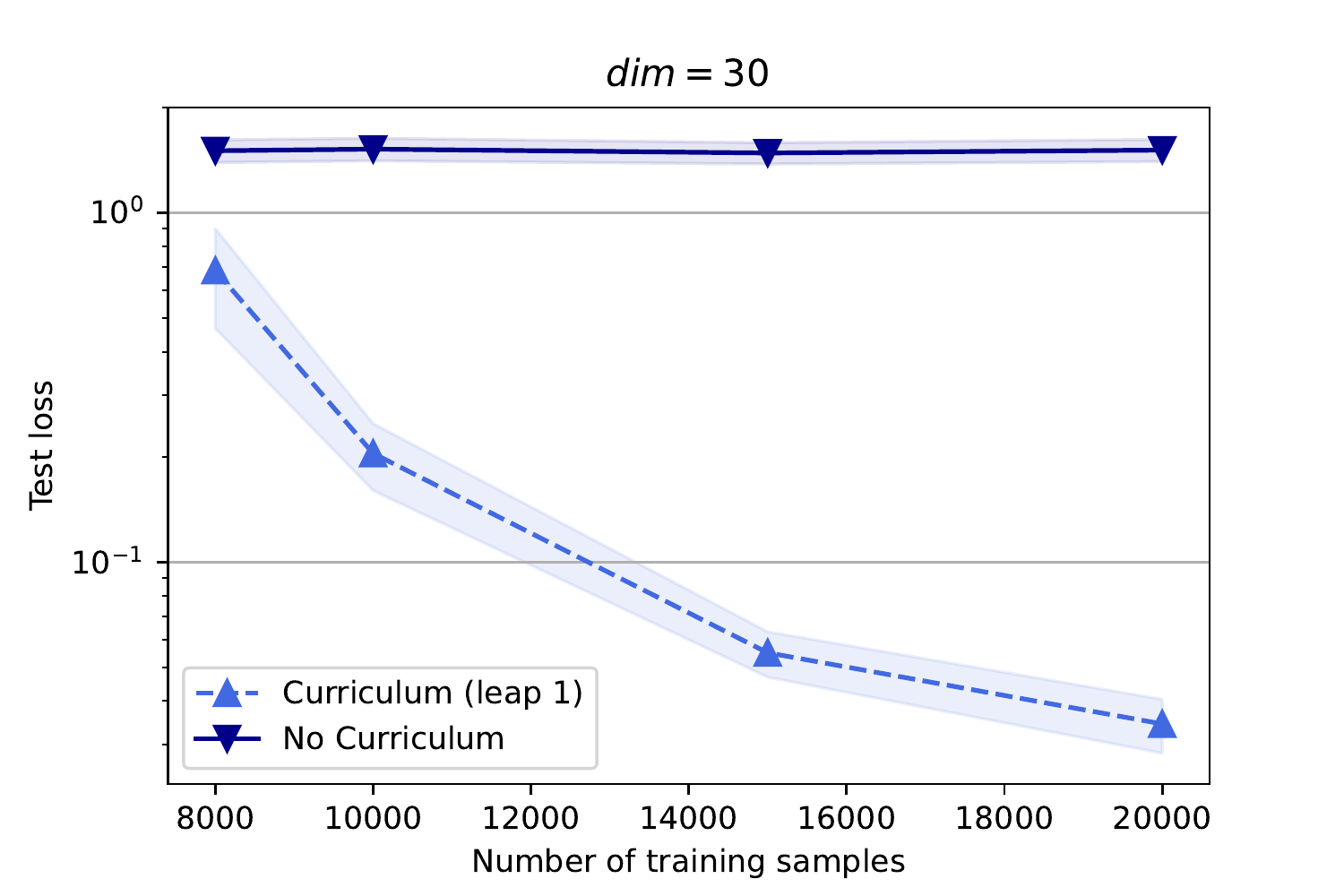}
    \end{subfigure}
    \caption{Generalization loss on the 16-parity (left) and 30-parity (right) targets for different numbers of samples with and without the Degree-Curriculum Algorithm. We note that the MLP model trained without curriculum was not able to learn the full parity function in dimension 30 for the given sample sizes (and even up to $10^5$ samples), in contrast to the same model trained with the Degree-Curriculum.}
    \label{fig:curriculum}
    
\end{figure*}

In Algorithm \ref{alg:degree-curriculum}, it is assumed that the training set is given with the random access model. We can also consider a variant with the query access model, where at step $i$, training samples are queried directly from $B_{r_i}$ (or some distribution). In the former case, the probability of a sample belonging to $B_{r}$ is 
small for small values of $r$ (e.g., $r=o_d(d)$). We thus expect the Degree-Curriculum algorithm under the query access model to be more efficient in that regard. In a concurrent work, \citet{cornacchia2023mathematical} have investigated the benefit of using a query model with a biased sample distribution before a denser distribution to learn parities. Particularly, an improvement in the number of GD iterations has been proved using 1-step gradient arguments. In addition, \citet{abbe2023provable} has pursued the approach from this paper and the paper of 
\citet{cornacchia2023mathematical} and has shown a formal separation between learning with and without curriculum for parities on a common data distribution. More specifically, it has been shown that for a data distribution that is a mixture of dense (uniform) and sparse (e.g., similar to $B_1$) inputs, one can use a two-step curriculum (first on the sparse samples and then on the whole distribution) and learn the parity using fewer optimization steps comparing to the unordered samples.

Note that in the Boolean setting and for the parity functions, $+1$ is the identity element. Thus, the number of $-1$'s used in the Degree-Curriculum algorithm can also be viewed as the length of the inputs. Interestingly, some works in the natural processing domain have used the length of the sentences (possibly along with other properties) to design their curriculum strategy \citep{spitkovsky2010baby, zaremba2014learning, nlp1, nlp3}. Finally, we can naturally extend the Degree-Curriculum algorithm to non-Boolean settings using the same principle as above: \\ {\it Build curriculum sets $\{\tilde{B}_i\}$ of `increased complexity' in order to have a path of learned functions on support sets $\{\mathcal{S}^{(i)}\}$ that are as tightly nested as possible (e.g., staircases or low-leap functions as in the work of \citealp{mergedstaircase}), with the target function at last}. 

\section{Min-Degree Bias Beyond the Previous Settings}
In this section, we study min-degree bias beyond the previous settings. Particularly, we investigate the effects of lifting the sparsity condition, the effects of using causal attention masking in Transformers, and using other architectures, namely, MLPs and mean-field networks. 
\subsection{Small Ambient Dimension}\label{sec:dimension-sensitivity}
In Theorem \ref{thm:random-features}, we showed that for sparse functions and unseen domains (see Definition \ref{def:sparse-setting}) the solution of the random features model would converge to the min-degree interpolator as the ambient dimension and number of features diverge. In our experiments presented in \cref{sec:exps} and \cref{app:exps}, we demonstrated that the min-degree bias is visible even for small values of the dimension. Particularly, for $(f_3, \mathcal{U}_3) = (x_0x_1x_2 + \cdots + x_{13}x_{14}x_0 + x_{14}x_0x_1, \{(x_0,x_1,x_2) = (-1, -1, -1)\})$, we can observe the min-degree bias despite the function not being sparse (see Figures \ref{fig:transformers} and \ref{fig:other-archs}). Note that in this case, the degree and the size of the unseen domain are small in comparison to the ambient dimension. In this section, we show that the min-degree bias can be weak if the ambient dimension is small compared to the degree and size of the unseen domain. Here, we consider two examples: degree-2 parity with holdout of pattern $(-1, -1)$, i.e., $(\mathrm{parity}_2, \mathcal{U}) = (x_0x_1, \{(x_0,x_1) = (-1, -1)\})$ and degree-4 parity with a frozen bit $(\mathrm{parity}_4, \mathcal{U}) = (x_0x_1x_2x_3, \{x_0 = -1\})$. Note that given the unseen domains any interpolator of $\mathrm{parity}_2$ can be written as $(1-\alpha_{\mathrm{Leak}})(x_0+x_1-1) + \alpha_{\mathrm{Leak}}x_0x_1$ where $(1-\alpha_{\mathrm{Leak}})$ is the coefficient of the min-degree interpolator and $\alpha_{\mathrm{Leak}}$ is the leakage coefficient. Similarly, any interpolator of $\mathrm{parity}_4$ is of the form $(1-\alpha_{\mathrm{Leak}})x_1x_2x_3 + \alpha_{\mathrm{Leak}}x_0x_1x_2x_3x_4$. In Figure \ref{fig:low-dimension}, we trained different models on these functions embedded in varying ambient dimensions and computed the leakage coefficient $\alpha_{\mathrm{Leak}}$. It can be seen that if the ambient dimension $d$ is too small, the min-degree bias may become weak or disappear. In such cases, other architecture-specific implicit biases may become relevant.
This is also related to the experiments conducted by \citet{zhou2023algorithms} where it is shown that for a Boolean AND target with all variables being active (i.e., the ambient dimension is equal to the effective dimension), Transformers learn a function different than the min-degree interpolator which is conjectured by \citet{zhou2023algorithms} to be the shortest RASP-L program (a type of program that encodes what Transformers tend to compute, for more about RASP see \citealp{rasp}). However, these non-min-degree results seem to be rather `boundary cases', i.e., when the ambient and effective dimensions are exactly the same or very close. As soon as the ambient dimension exceeds the effective one by a large enough margin, it appears that the min-degree bias dominates again, as seen in Figure \ref{fig:low-dimension}.  

\begin{figure*}[tbh]
     \centering
     \begin{subfigure}[b]{0.49\textwidth}
         \centering
         \includegraphics[width=\textwidth]{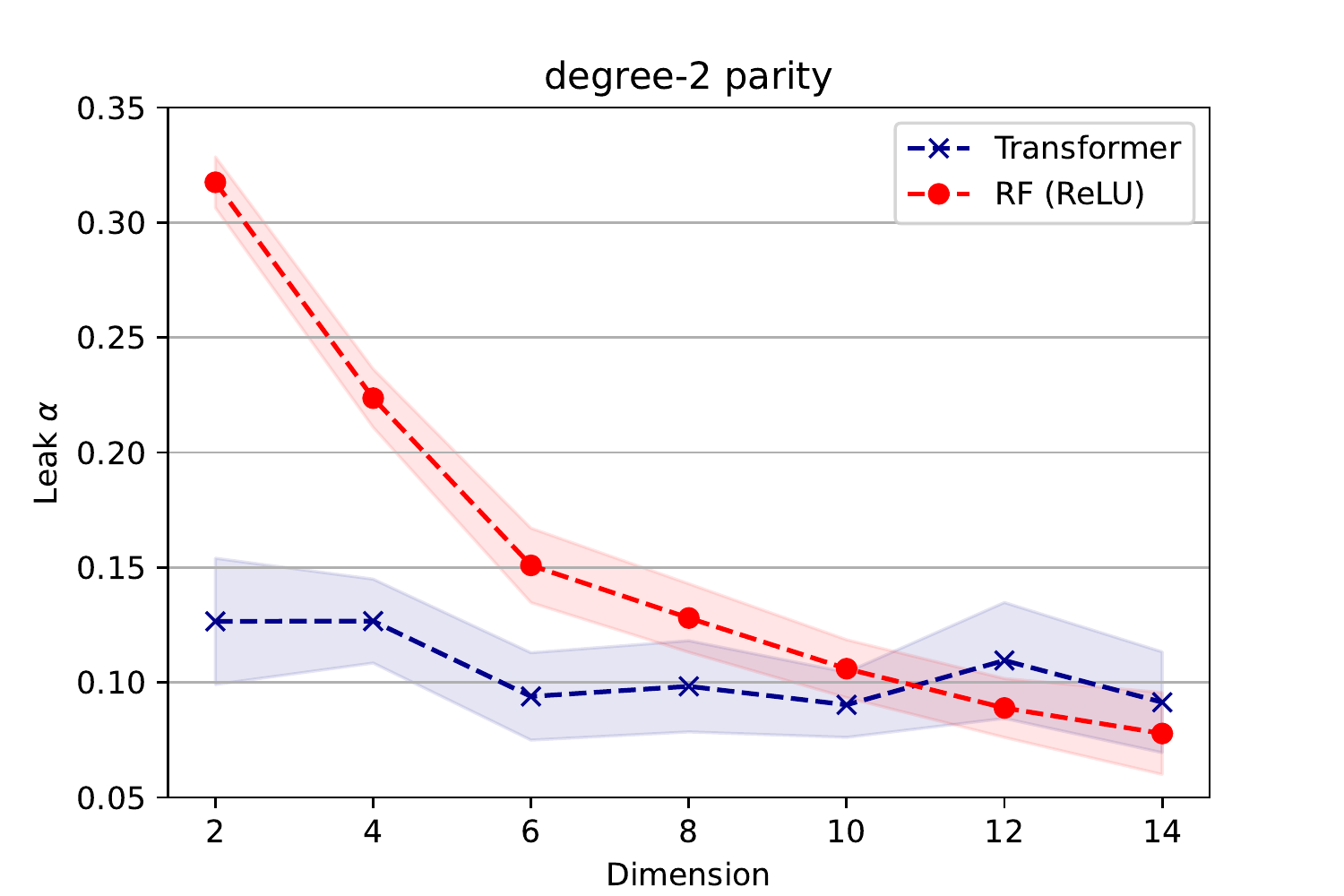}
     \end{subfigure}
     \hfill
     \begin{subfigure}[b]{0.49\textwidth}
         \centering
         \includegraphics[width=\textwidth]{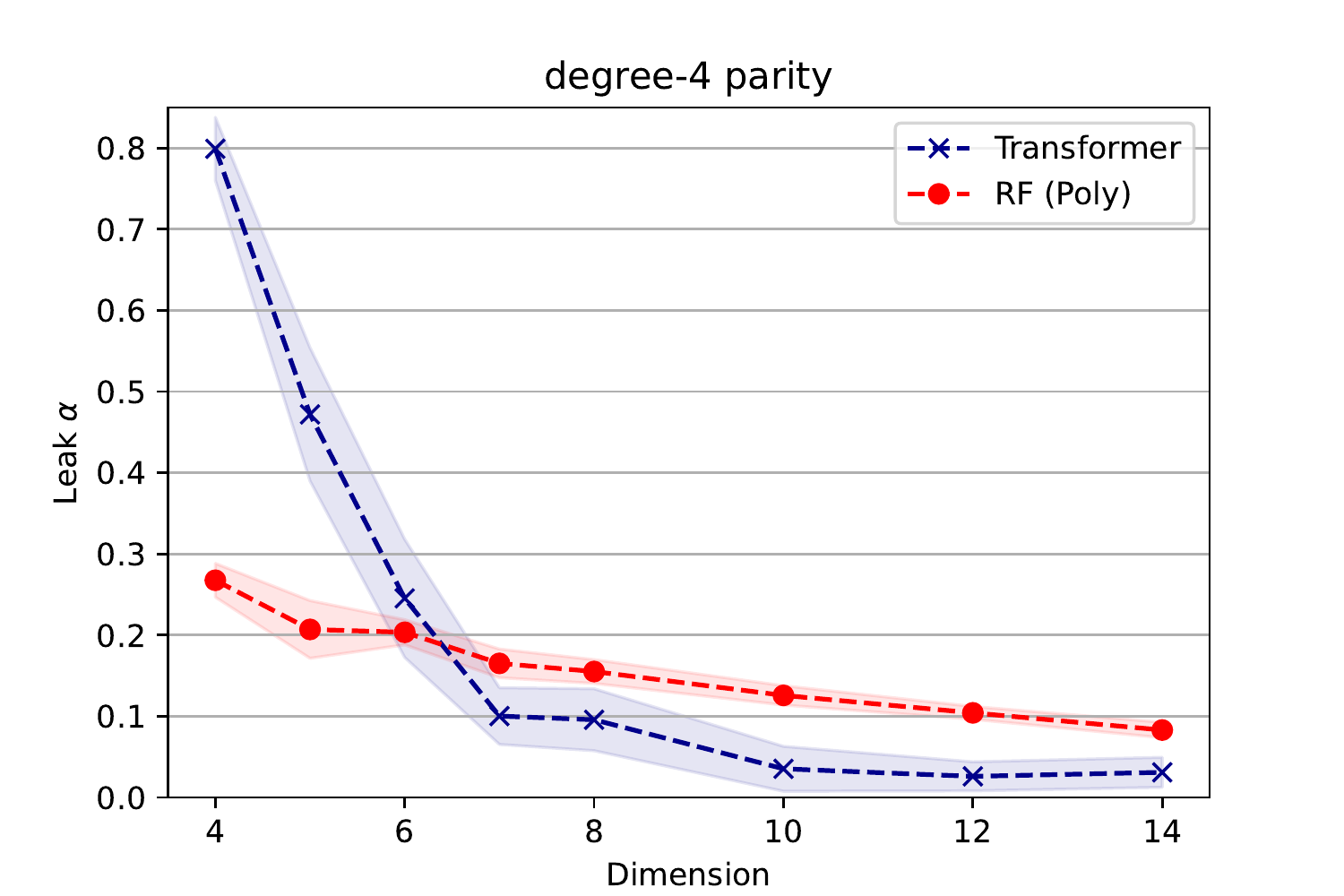}
     \end{subfigure}
     
        \caption{Learning $(\mathrm{parity}_2, \mathcal{U}) = (x_0x_1, \{(x_0, x_1) = (-1,-1)\})$ (left) and $(\mathrm{parity}_4, \mathcal{U}) = (x_0x_1x_2x_3, \{x_0 =-1\})$ (right) embedded in different dimensions with different models. For $\mathrm{parity}_2$ (left) we can see that the min-degree bias is strong for the Transformer even for low-ambient dimensions. We can also see that for the RF model, the min-degree bias becomes stronger as the ambient dimension increases.
        For $\mathrm{parity}_4$ (right) we can see that the Transformer can almost recover the true function when the ambient and active dimensions match. As the ambient dimension grows slightly, we see that the coefficient of the higher degree term falls rapidly resulting in learning the MD interpolator.}
        \label{fig:low-dimension}
\end{figure*}

\subsection{Transformer with Causal Attention}
\label{sec:causal-transformer}
In our main experiments, we have used an encoder-only Transformer architecture with bidirectional attention and learnable absolute positional embeddings. First, we explain our reasoning for this choice.  Note that in our settings, most of the coordinates are i.i.d. uniform $\pm 1$ bits due to the restricted size of the coordinates in the unseen domain (see Definition \ref{def:sparse-setting}). As a result, we do not have any locality structure a priori, and hence, the use of relative positional embeddings does not seem suitable for our tasks. Moreover, our output is a single continuous variable which makes any sort of auto-regressive training inapt as well.

On the other hand, the recent work of \citet{kazemnejad2023impact} has shown that in decoder-only Transformers with no positional embeddings, the causal attention masks make the recovery of positional information possible (in contrast to encoder-only architectures with bidirectional attention in which the removal of positional embeddings makes the architecture permutation invariant). Moreover, they have shown that decoders with no positional embedding may exhibit superior performance in some length generalization tasks compared to decoders with positional embeddings. Motivated by this, we tried modifying our encoder architecture by making the attentions causal (unidirectional) and removing the positional embeddings. We trained this variant in a supervised setting with $\ell_2$ loss similar to the original encoder-only architecture. Interestingly, we observed that this new architecture may lose the min-degree bias in some settings. 

Notice that our original architecture with positional embeddings and bidirectional attention is symmetric with respect to different coordinates. However, this is not true when we use causal attention. For example, the behavior of encoders (with no attention masking) trained on tasks $(f, \mathcal{U}) = (x_0x_1, \{x_0=-1\})$ and $(f, \mathcal{U}) = (x_{14}x_{15}, \{x_{14}=-1\})$ in dimension $d=16$ would be the same, while this is not necessarily true for Transformers with causal attentions. (In Table \ref{tab:decoder}, we see that the behavior is indeed different.) In other words, with causal attention, the positions of latent bits (and unseen domain) matter. 
As an example, we try learning the parity of two bits embedded in ambient dimension $d=16$. We also freeze one of the bits to $+1$ during training (same task as examples above). If our function is $x_ix_j$ with $x_i=1$ during training, the interpolator would have the form $(1-\alpha_{\mathrm{Leak}})x_j + \alpha_{\mathrm{Leak}}x_ix_j$ where the min-degree bias predicts that $\alpha_{\mathrm{Leak}}$ would be small. In Table \ref{tab:decoder}, we have tried different positions for the latent coordinate and the frozen bit and reported the learned solution averaged over $10$ random seeds. Note that we still see the min-degree bias for most of the placements, while for some of the placements, the min-degree bias disappears. Notice that we can also keep the positional embeddings while making the attentions causal. In this case, we can observe the min-degree bias again (potentially still weaker than the encoder-only architecture with bidirectional attention) as seen in \cref{tab:decoder}.
\begin{table}[tb]
    \centering
    \tabcolsep=0.15cm
    \begin{tabular}{cc cc cc}
        \toprule
        & & \multicolumn{2}{c}{Causal mask without pos. emb.} & \multicolumn{2}{c}{Causal mask with pos. emb.} \\
        \cmidrule(lr){3-4} \cmidrule(lr){5-6}
        Target & Fixed & $\overline{\alpha_\mathrm{Leak}} \pm \mathrm{std}$ & learned function & $\overline{\alpha_\mathrm{Leak}} \pm \mathrm{std}$ & learned function \\
        & bit & & (averaged) & & (averaged) \\
        \midrule
$x_0x_1$ & $x_0 $ & $\mathbf{0.55 \pm 0.04}$ & $\mathbf{0.45}x_1 + \mathbf{0.55}x_0x_1$ & $0.17 \pm 0.02$ & $0.83x_1 + 0.17x_0x_1 $ \\
$x_0x_1$ & $x_1$ & $-0.02 \pm 0.02$ & $1.02x_0 -0.02x_0x_1 $ & $0.02 \pm 0.02$ & $0.97x_0 + 0.02x_0x_1 $ \\

$x_{14}x_{15}$ & $x_{14}$ & $0.06 \pm 0.05$ & $0.93x_{15} + 0.06x_{14}x_{15} $ & $0.0 \pm 0.01$ & $0.99x_{15} + 0.0x_{14}x_{15} $ \\
$x_{14}x_{15}$ & $x_{15}$ & $\mathbf{0.47 \pm 0.04}$ & $\mathbf{0.53}x_{14} + \mathbf{0.47}x_{14}x_{15}$ & $0.01 \pm 0.02$ & $1.0x_{14} + 0.01x_{14}x_{15} $ \\

$x_2x_8$ & $x_2 $ & $0.15 \pm 0.05$ & $0.85x_8 + 0.15x_2x_8 $ & $0.01 \pm 0.02$ & $0.99x_8 + 0.01x_2x_8 $ \\
$x_2x_8$ & $x_8 $ & $0.0 \pm 0.02$ & $1.0x_2 + 0.0x_2x_8 $ & $0.0 \pm 0.02$ & $1.0x_2 + 0.0x_2x_8 $ \\

$x_7x_{13}$ & $x_7$ & $0.01 \pm 0.02$ & $0.99x_{13} + 0.01x_7x_{13} $ & $0.01 \pm 0.02$ & $0.99x_{13} + 0.01x_7x_{13} $ \\
$x_7x_{13}$ & $x_{13} $ & $0.02 \pm 0.01$ & $0.99x_7 + 0.02x_7x_{13} $ & $0.02 \pm 0.02$ & $1.0x_7 + 0.02x_7x_{13} $ \\
        \bottomrule
    \end{tabular}
    \caption{Learning parity of two bits while one bit is frozen to one during training using Transformers with causal attention masking. Each row represents one particular combination for the position of the two bits and the frozen (fixed) bit. In columns, we report the average leakage coefficient ($\pm$ standard deviation) and average learned function using 10 seeds for a Transformer with causal masking and no positional embedding and also for a Transformer with causal masking and positional embedding.  For the Transformer without positional embeddings it can be seen that only two of the placements lead to the violation of the min-degree bias (in bold). For the Transformer with causal masking and positional embeddings, there is only placement that leads to a non-negligible leakage. For all other cases, the min-degree bias is still strongly present. Note that the leakage for encoder-only Transformer with bidirectional attention is negligible and independent of placement and thus not reported in this table.}
    \label{tab:decoder}
\end{table}
As reported in \cref{tab:decoder}, the behavior of the Transformer architecture with causal masking heavily depends on the positions of latent coordinates and possibly the position of the bits involved in the unseen domain which creates a large set of placements for each sparse function. As a result, understanding the implicit bias of the Transformer model with causal masking requires a new avenue of investigation which we leave for future work. 

In any case, as mentioned earlier, using a Transformer with causal attention may not be the natural model choice for learning Boolean/logic targets that do not have a causal structure in their input space. In fact, for Boolean inputs in which coordinates do not usually follow any causal relationship, encoder-only Transformers with bidirectional attention (which are symmetric with respect to different coordinates) seem to be the most reasonable choice. In this case, the min-degree bias dominates in the sparse regime. It is mostly intriguing from a theoretical viewpoint to see how the causal attention masking and the removal of positional embeddings affect the min-degree bias in some cases depending on the placement of the function. 

\subsection{Other Architectures} \label{sec:other-archs} 
In Section \ref{sec:exps}, we showed that encoder-only Transformers have a strong min-degree bias similar to the random features model. Now, we investigate two other architectures, namely, multi-layer perceptron (MLP) with 4 hidden layers and 2-layer neural network with mean-field parameterization \citep{mei2018mean}. We show that these architectures also have the min-degree bias although in a weaker format. As a result, they learn leaky min-degree interpolators meaning that they partly capture the higher degree solution along with the min-degree solution. Particularly, we try learning examples $(f_1, \mathcal{U}_1), (f_2, \mathcal{U}_2), (f_3, \mathcal{U}_3)$ of \cref{sec:exps} with the MLP and mean-field model. The results are depicted in \cref{fig:other-archs} showing that these models have leaky min-degree biases. Further, in Appendix \ref{app:learning-rate-sensitivity}, we discuss the effect that large learning rates may increase the leakage of these models. For additional experiments refer to \cref{app:exps}.
\begin{figure*}[ht]
     \centering
     \begin{subfigure}[b]{0.29\textwidth}
         \centering
         \includegraphics[width=\textwidth]{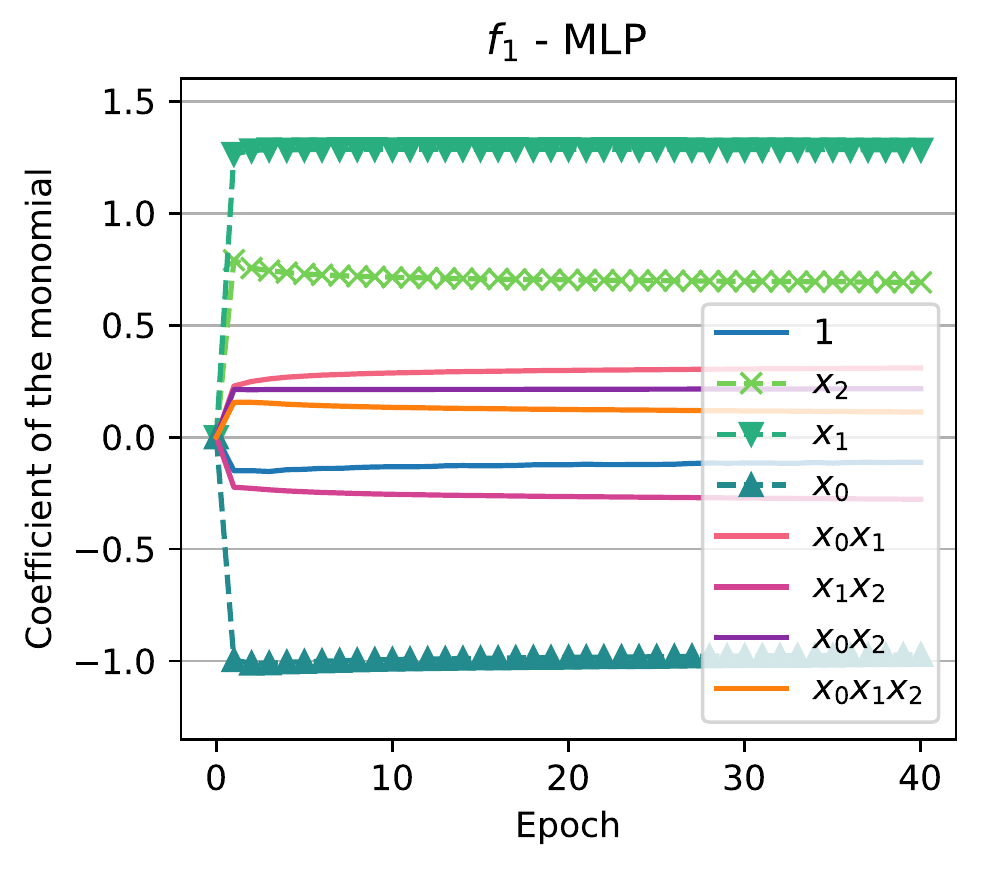}
     \end{subfigure}
     \hfill
     \begin{subfigure}[b]{0.29\textwidth}
         \centering
         \includegraphics[width=\textwidth]{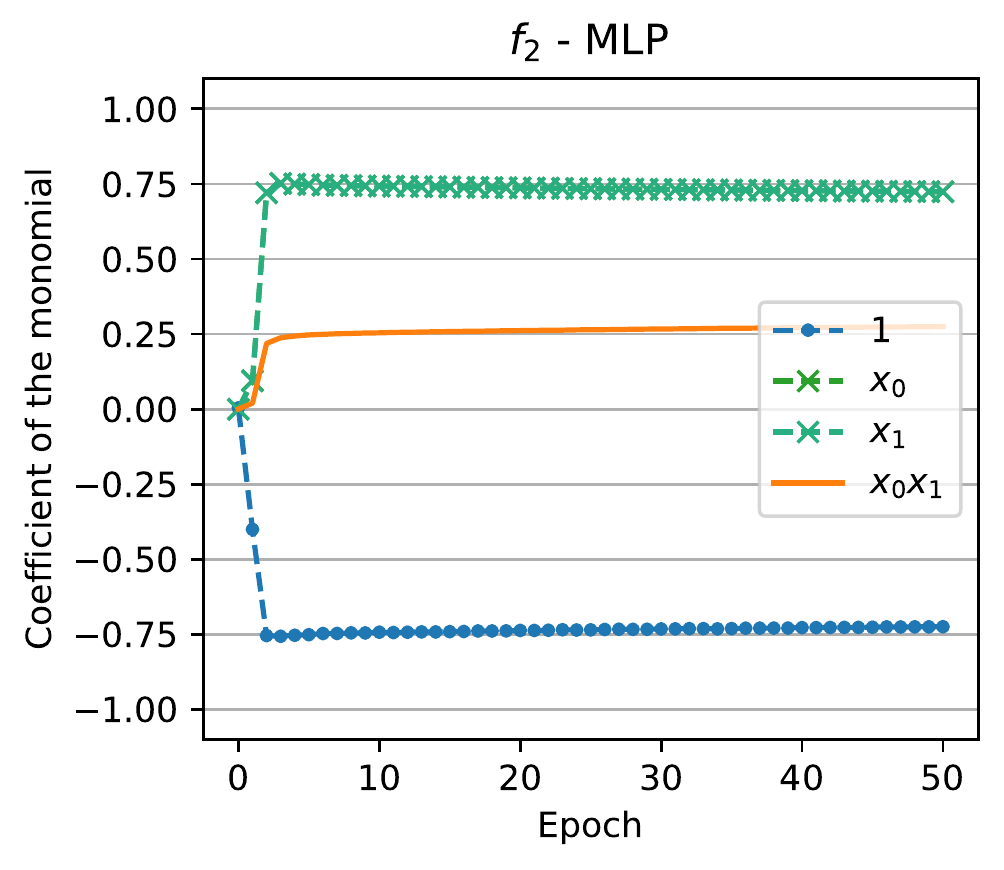}
     \end{subfigure}
     \hfill
     \begin{subfigure}[b]{0.29\textwidth}
         \centering
         \includegraphics[width=\textwidth]{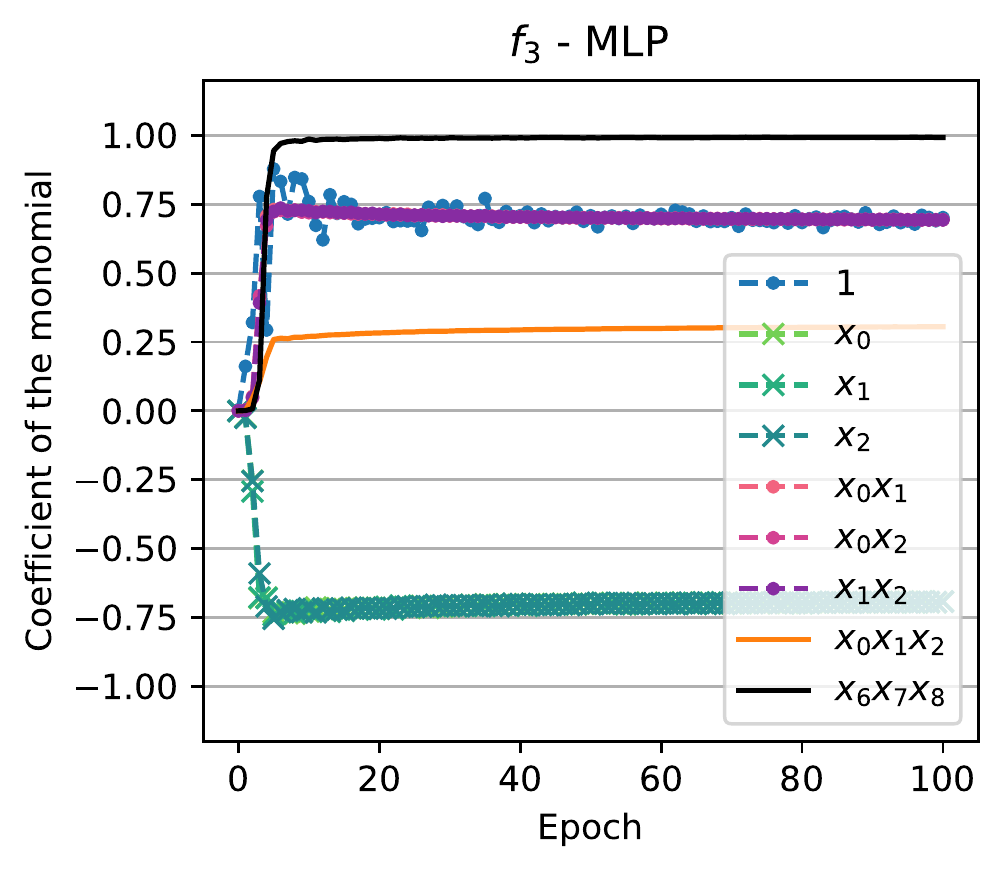}
     \end{subfigure}
          \begin{subfigure}[b]{0.29\textwidth}
         \centering
         \includegraphics[width=\textwidth]{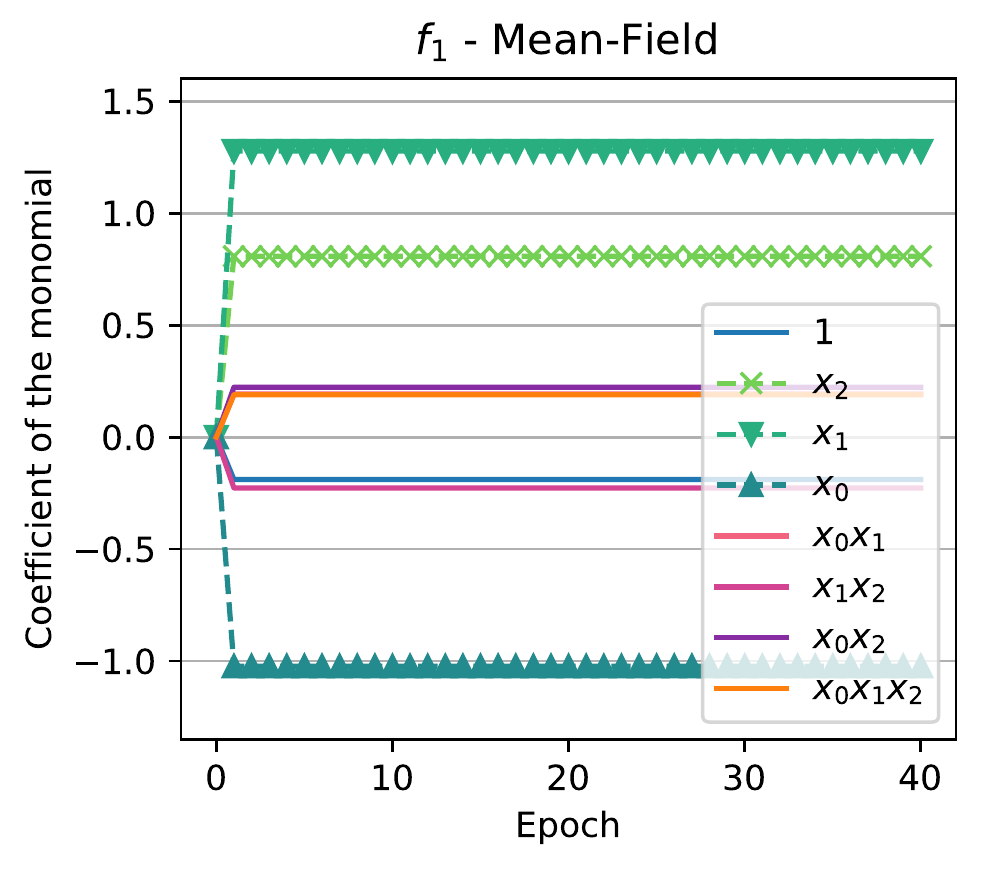}
     \end{subfigure}
     \hfill
     \begin{subfigure}[b]{0.29\textwidth}
         \centering
         \includegraphics[width=\textwidth]{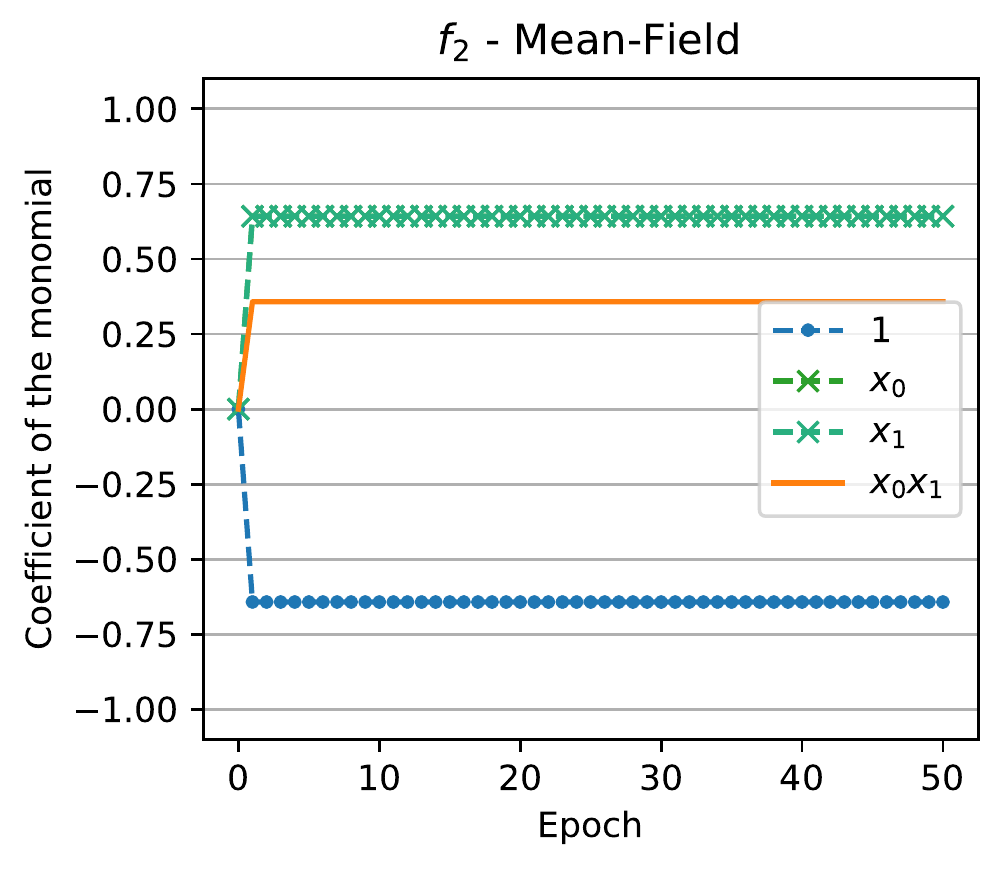}
     \end{subfigure}
     \hfill
     \begin{subfigure}[b]{0.29\textwidth}
         \centering
         \includegraphics[width=\textwidth]{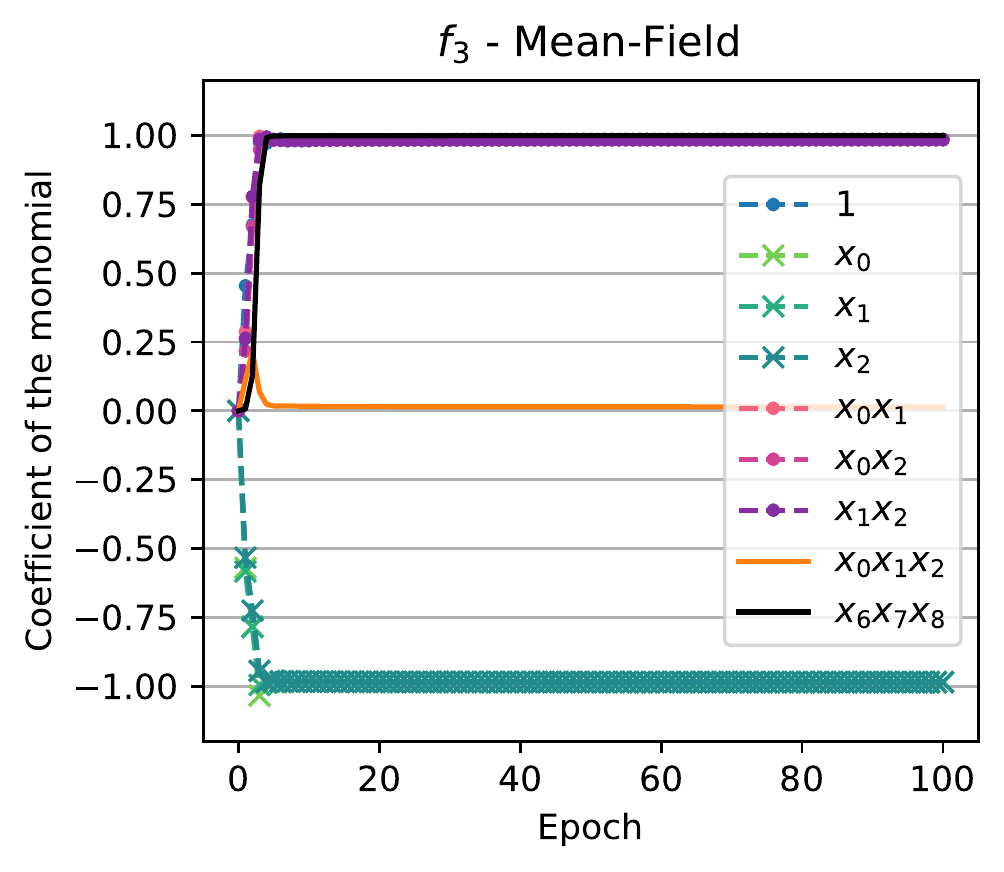}
     \end{subfigure}
     
    \caption{Functions $f_1$ (left), $f_2$ (middle), and $f_3$ (right) of Section \ref{sec:exps} learned by the MLP (top row) and the mean-field model (bottom row). In all of these examples, the higher degree monomials (represented by the solid orange lines in the middle and left columns) are replaceable by the lower degree alternative (represented by the dashed lines). The MLP and mean-field models learn a leaky min-degree interpolator with the coefficient of the higher degree term possibly bounded away from 0.}
    \label{fig:other-archs}
\end{figure*}

\section{Related Literature}

Given the deployment of machine learning models in the real world, out-of-distribution generalization is a critical aspect of machine learning that has been extensively studied both in theory \citep{ben2006analysis, mansour2009domain, redko2020survey} and in practice \citep{gulrajani2020search, miller2021accuracy, wiles2022a}. Our work considers an extreme case of distribution shift in which part of the domain is entirely unseen during the training, and thus OOD generalization is only possible if the target function has special structures (e.g., being compositional or having in/equi-variances) and the model captures those structures.
OOD generalization and the ability to extrapolate have also been used as proxies for measuring the reasoning capabilities of neural networks \citep{saxton2019analysing, Zhang2021PointerVR, csordas2021devil, zhang2022unveiling} as these models are prone to memorization of training samples \citep{carlini2019secret-mem1,feldman2020neural-mem2,kandpal2022deduplicating-mem3,carlini2022quantifying-mem4, Zhang2021PointerVR} or learning undesirable shortcuts \citep{zhang2022unveiling}.
A special case is length generalization \citep{zaremba2014learning, lake2018generalization, hupkes2020compositionality, zhang2022unveiling, anil2022exploring-length, zhou2023algorithms}, i.e., generalization to the input lengths beyond what is seen during the training. In this paper, we provided an explanation for the length generalization problem in the simple instance of parity functions \citep{anil2022exploring-length}.

It has been shown that training with gradient descent imposes particular implicit regularization on the solutions found by the models such as sparsity \citep{moroshko2020implicit}, norm minimization \citep{bartlett2021deep}, and margin maximization (in linear classification setting) \citep{soudry2017implicit}. 
This implicit regularization (or implicit bias) of neural networks trained with gradient-based algorithms has been used to explain the generalization of (often overparametrized) models \citep{bartlett2021deep}.
These results depend on the optimizer \citep{gunasekar2018characterizing} and model \citep{gunasekar2018implicit} and are usually proven for simple models such as linear models \citep{soudry2017implicit,  yun2020unifying, jacot2021saddle} including diagonal linear neural networks \citep{gunasekar2018implicit, moroshko2020implicit} as studied in this paper. Our result for the random feature model builds upon the implicit bias toward solutions with minimum norm \citep{bartlett2021deep}. 
Also related to us is the spectral bias \citep{xu2019frequency, rahaman2019spectral} stating that neural networks, when learning a function in continuous settings, capture the lower frequency components faster (note that degree in Boolean functions plays a similar role to the frequency). In this paper, we develop a related insight in the Boolean setting by introducing the notion of degree-profile and showing the min-degree implicit bias for several models theoretically and empirically.
On the drawbacks of such implicit biases, \citet{shah2020pitfalls} put forward the notion of extreme simplicity bias, showing that neural networks may ignore complex predictive features and solely rely on the simpler features. This simplicity bias may result in vulnerability to adversarial perturbations and poor OOD performance. In this paper, we also discuss how the min-degree bias can practically hinder length generalization.

\section{Conclusions and Future Directions}\label{sec:conclusions}
In this paper, we put forward the concept of generalization on the unseen (GOTU) and considered the learning of Boolean functions. We showed that various network architectures have a bias toward the min-degree interpolator, with theoretical results for the RF and diagonal/2-layer linear neural networks, and experimental results for encoder-only Transformers. We also found empirically that for large learning rates or for other models such as mean-field networks, a leaky version of the MD bias takes place. We also observed that using causal attention masking along with removing positional embeddings in Transformers or having very small ambient dimensions can tame the min-degree bias in some cases. 

We showed that the MD bias can be used in a curriculum learning algorithm where the training takes place on sets of increasing complexity. We also demonstrated that the MD bias can impede the learning of symmetric solutions and can make length generalization difficult.

The min-degree bias is a form of Occam's razor chosen by GD-trained neural nets, where the `simplicity' is measured by the `degree-profile'. However, this might not be a desirable form of razor for various reasoning tasks. We believe that other forms promoting symmetries, compositionality, or more generally minimum description length (MDL) may often be more suitable. The next natural steps are thus to correct this min-degree bias. We propose here some directions to pursue: (1) architecture design promoting symmetries or compositionality, (2) hyperparameter tuning (e.g., learning rates, scale), (3) data augmentation and multitasking, (4) MDL-like regularization at training. 

Lastly, we provide a demonstration of the last idea. For this example, consider learning task $f_3(x) = x_0x_1x_2 + x_1x_2x_3 + \cdots + x_{14}x_0x_1$ and $\mathcal{U}_3 = \{(x_0, x_1, x_2) = (-1, -1, -1)\}$. Note that this target has a cyclic symmetry but the min-degree interpolator is not invariant under any permutation (other than the identity). We add regularization term $L_{\mathrm{reg}} = \mathbb{E}_{\pi,x}[(f_{\mathrm{NN}}(x) - f_{\mathrm{NN}}(\pi(x)))^2]$, where $x$ is a random binary vector and $\pi$ is drawn uniformly from the set of all permutation ($\pi(x)$ refers to permuting indices of $x$ according to $\pi$), to the loss function. Note that this regularization term is $0$ if and only if the neural network's function is permutation invariant. We train the encoder-only Transformer with this regularizer. Particularly, we use 256 random samples and 1 random permutation in each iteration to estimate the regularization loss. In Figure \ref{fig:reg}, we show the coefficient of different monomials during training with this regularizer. It can be seen that the high-degree term $x_0x_1x_2$ is mostly recovered by the end of training. In Figure \ref{fig:transformers}, we saw that training Transformers on this task without regularizer would result in high-degree term $x_0x_1x_2$ not being learned. Moreover, we can compare the generalization error of these two cases. With the regularization term, the Transformer achieves the generalization error $0.95 \pm 0.13$, while if no regularization is used the generalization error would be $8.14 \pm 0.9$ (we report average and standard deviation over $10$ seeds). More general investigations of this approach, also on how to learn symmetry groups, are left to future work.  

\begin{figure}[htb]
\centering
\includegraphics[width=0.6\textwidth]{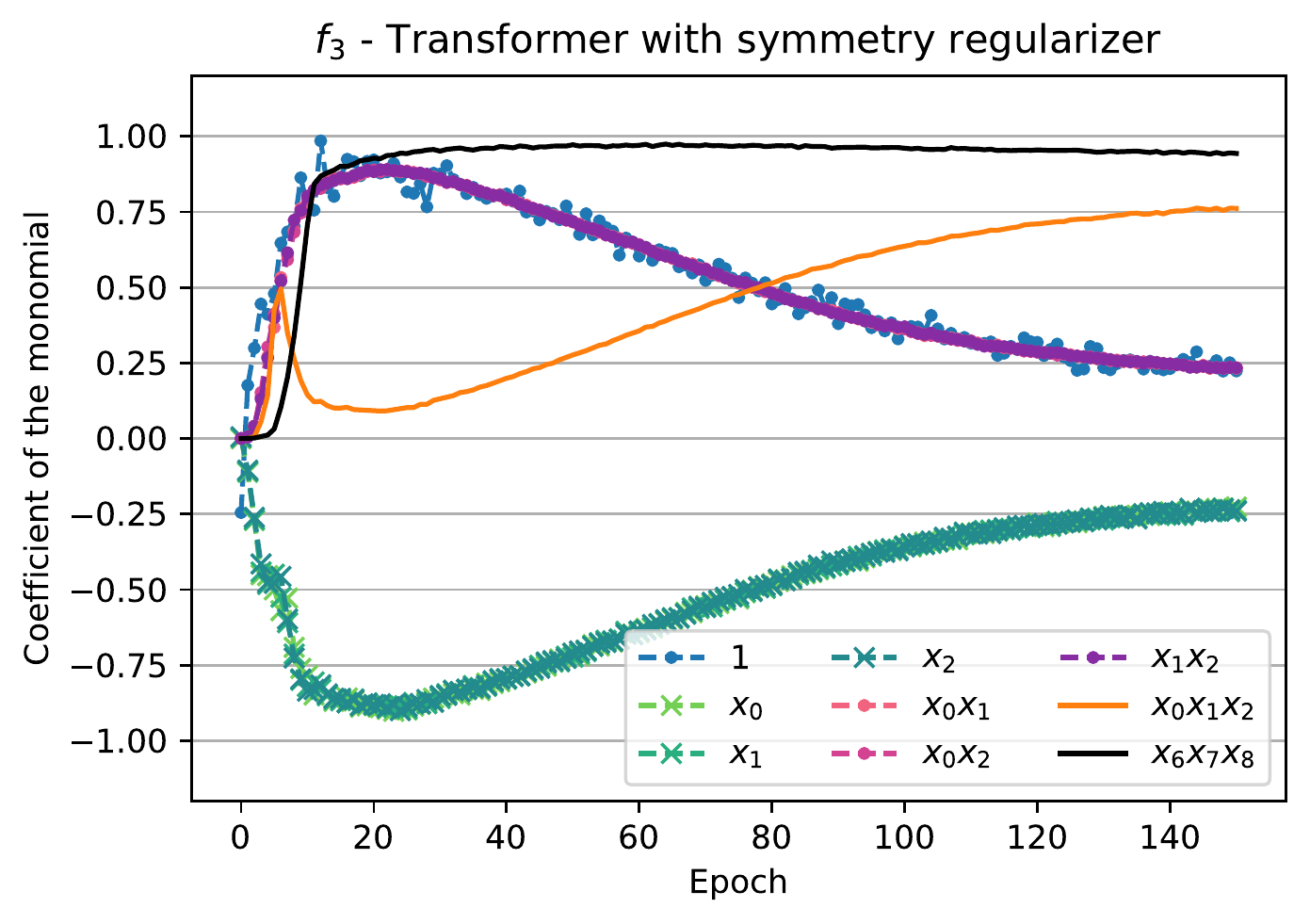}
\caption{Function $f_3$ learned by a Transformer with a regularization term. One can see that the high-degree monomial $x_0x_1x_2$ (orange solid line) is mostly recovered compared to Figure \ref{fig:transformers} where it was completely lost.}
\label{fig:reg}
\end{figure}


\acks{The authors of this work were funded by EPFL and Apple. We would like to thank anonymous reviewers and colleagues for their feedback on the earlier versions of this work.}


\newpage

\appendix
\section{Proofs}\label{app:proofs}
\subsection{Proofs for the Random Features Model}
We start by proving a lemma showing that for strongly expressive activation functions each random feature is low-degree in the sense that the high-degree monomials have small coefficients in the Fourier-Walsh expansion of the random features. 
\begin{lemma}[Random features are low-degree]\label{lemma:random-low-deg}
Consider random features generated \\by an activation function that is strongly expressive up to $P=O_d(1)$, i.e.,  $\phi_{w,b}(x) = \sigma(\langle w, x \rangle + b)$ where $w_i, b \sim \mathcal{N}(0, \frac{1}{d})$ are the random weights and bias. We have the following additional properties:
\begin{enumerate}[label=A{{\arabic*}}., leftmargin=28pt]
\setcounter{enumi}{2}
    \item $\forall T \subseteq [d]\;\;  \mathbb{E}_{w,b}[\hat{\phi}_{w,b}(T)^2]$ exists and $\mathbb{E}_{w,b}[\hat{\phi}_{w,b}(T)^2] = \Theta(d^{-|T|})$ for $|T| \leq P$; 
    \item $\mathbb{E}_{w,b}[\hat{\phi}_{w,b}(T)\hat{\phi}_{w,b}(T')] = 0$ for $T \neq T'$; and
    \item $\mathbb{E}_{w,b}[\hat\phi_{w,b}(T)^2] =0  \iff \hat\phi_{w,b}(T) = 0 \; \forall\; w, b$,
\end{enumerate}
where $\hat{\phi}_{w,b}(T)$ is the coefficient of monomial $T$ in random feature $\phi_{w,b}(x).$
\end{lemma}
\begin{proof}
    For property (A3), consider all the subsets of $[d] = \{1, \ldots, d\}$ with size $k \leq P$: $T_1, T_2, \ldots, T_{\binom{d}{k}}$. Due to the symmetry, we have $\E_{w,b}[\hat \phi(T_1)^2]=\cdots=\E_{w,b}[\hat \phi(T_{\binom{d}{k}})^2]$. Moreover, we have
    \begin{align}
        \binom{n}{k}\E_{w,b}[\hat \phi(T_i)^2] &= \sum_{i=1}^{\binom{d}{k}} \E_{w,b}[\hat \phi(T_i)^2] = \E_{w,b}[\sum_{i=1}^{\binom{d}{k}} \hat \phi(T_i)^2] \leq \E_{w,b}[\sum_{T \subseteq [d]} \hat \phi(T)^2] \\&= \E_{w,b}[\E_{x}[\phi(x)^2]]\label{eq:temp1}=\E_{x}[\E_{w,b}[\sigma(\langle w, x \rangle + b)^2]] = \E_{g \sim \mathcal{N}(0, \frac{d+1}{d})}[\sigma(g)^2],
    \end{align}
    where in Equation \ref{eq:temp1} we used Parseval's identity. By assumption (A1) on the function we know that $\E_{g \sim \mathcal{N}(0, 2)}[\sigma(g)^4]$ is finite. Thus, $\E_{g \sim \mathcal{N}(0, 2)}[\sigma(g)^2]^2$ is also finite and consequently $\E_{g \sim \mathcal{N}(0, \frac{d+1}{d})}[\sigma(g)^2]^2$ can be upper bounded independently of $d$, which proves the existence part. Furthermore, $\E_{w,b}[\hat \phi(T_i)^2] = O_d(\binom{d}{k}^{-1}) = O_d(d^{-k}),$ where we used $k \leq P = O_d(1)$. Now by property (A2), we can conclude that $\mathbb{E}_{w,b}[\hat{\phi}_{w,b}(T)^2] = \Theta(d^{-|T|}) \mathrm{\;for\;} |T| \leq P$.

    For property (A4), assuming $T\neq T'$ take $i \in T \Delta T'$. Without loss of generality suppose $i \in T, i \notin T'$. For weight vector $w$, we flip the sign of the  $i$-th coordinate and denote the resulting vector by $w_{-i}$. Now note that $\E_x[\sigma(\langle w, x\rangle + b)\chi_T(x)] = -\E_x[\sigma(\langle w_{-i}, x\rangle + b)\chi_T(x)]$ and $\E_x[\sigma(\langle w, x\rangle + b)\chi_{T'}(x)] = \E_x[\sigma(\langle w_{-i}, x\rangle + b)\chi_{T'}(x)]$. Hence, $\hat\phi_{w,b}(T)\hat\phi_{w,b}(T') = -\hat\phi_{w_{-i},b}(T)\hat\phi_{w_{-i},b}(T')$ and 
    $\mathbb{E}_{w,b}[\hat{\phi}_{w,b}(T)\hat{\phi}_{w,b}(T')] = 0$.

    Note that the last property is a consequence of the continuity assumption on the activation function. 
\end{proof}
\vspace{1cm}
Now we can prove Theorem~\ref{thm:random-features}.
\begin{proof}\textit{(Theorem~\ref{thm:random-features})}
First, recall the set of all interpolating solutions on the training set $\mathcal{U}^c$ as 
$$
\mathcal{F}_{\mathrm{int}}(f_{\mathrm{target}}, \mathcal{U}) = \{f\colon \{\pm 1\}^d \to \mathbb{R} \mid f(x) =f_\mathrm{target}(x) \;\forall x \in \mathcal{U}^c\}.
$$
Note that a solution given by $a_1, \ldots, a_{N}$ is interpolating if and only if $
\frac{1}{\sqrt{N}}\sum_{i=1}^N a_i \phi_i(x) \in \mathcal{F}_{\mathrm{int}}
$.

Moreover, we study the features and solutions in the Fourier-Walsh basis. First, we index all possible monomials, i.e., $\{T_1, T_2, \ldots, T_{2^d}\} = 2^{\{1, 2, 3, \ldots, d\}}$ and $\chi_{T_i}(x) = \prod_{j \in T_i} x_j$. Further, we define the coefficient of monomial $T_j$ in the $i$-th feature as $\hat{\phi}_{i}(T_j) \coloneqq \mathbb{E}_x[\phi_i(x)\chi_{T_j}(x)]$ and $F \in \mathbb{R}^{2^d \times N}$ as the matrix of features in the Fourier expansion, i.e., $F_{i, j}=\frac{1}{\sqrt{N}}\hat{\phi}_j(T_i)$. Using this notation, $a$ corresponds to an interpolating solution if and only if 
\begin{equation}
\exists g \in \mathcal{F}_{\mathrm{int}}\;\; Fa = \hat{g}, \label{eq:int-sol-Fourier} 
\end{equation}
where $\hat{g}$ represents function $g$ in the Fourier-Walsh basis. 
Furthermore, note that
\begin{equation}
(FF^T)_{i,j} = \sum_{k=1}^{N}(\frac{1}{\sqrt{N}}\hat{\phi}_{k}(T_i))(\frac{1}{\sqrt{N}}\hat{\phi}_{k}(T_j)) = \frac{1}{N}\sum_{k=1}^{N} \hat{\phi}_{k}(T_i)\hat{\phi}_{k}(T_j).
\end{equation}
Note that weights and biases of the features are sampled i.i.d., therefore, as $N \to \infty$, $(FF^T)_{i,j}$ behaves like $\mathcal{N}(\mathbb{E}_{w}[\hat{\phi}_w(T_i)\hat{\phi}_w(T_j)], N^{-1}\var_w[{\hat{\phi}_w(T_i)\hat{\phi}_w(T_j)}])$ in  distribution, due to the central limit theorem (CLT). More precisely, we need to invoke the law of large number to get concentration on the mean. In our cases, the variances are finite due to property (A1). More specifically, $\E_{g\sim\mathcal{N}(0, 2)}[\sigma(g)^4]$ is finite, and hence, $\E_{g\sim\mathcal{N}(0, \frac{d+1}{d})}[\sigma(g)^4]$ is finite. Moreover, 
\begin{align}
    \infty &> \E_{g\sim\mathcal{N}(0, \frac{d+1}{d})}[\sigma(g)^4] = \E_{w, b}[\E_x[\sigma(\langle w, x\rangle + b)^4]]  \geq \E_{w, b}[\E_x[\sigma(\langle w, x\rangle + b)^2]^2] \label{eq:temp2}\\&= \E_{w, b}[(\sum_{T \subseteq [d]} \hat\phi_{w,b}(T)^2)^2] \geq \E_{w, b}[\hat\phi_{w,b}(T_i)^2\hat\phi_{w,b}(T_j)^2] \;\;\forall i, j, \label{eq:temp3}
\end{align}
where we used Parseval's identity from \cref{eq:temp2} to \cref{eq:temp3}.
We define $\Phi \in \mathbb{R}^{2^d \times 2^d}$ as a shorthand notation as
\begin{equation}
    \Phi_{i,j} = \mathbb{E}_{w,b}[\hat{\phi}_{w,b}(T_i)\hat{\phi}_{w,b}(T_j)] = \begin{cases}
    0 & i\neq j \\
    \mathbb{E}[\hat\phi_{w,b}(T_i)^2] & i=j
    \end{cases}, 
\end{equation}
where we have used properties (A3) and (A4).
\subsubsection{Existence of Interpolating Solutions}
Now, we show that an interpolating solution exists with high probability. Particularly, take any interpolator $g$ that only depends on the latent variables $x_{i_1}, \ldots, x_{i_P}$ and we show that $\hat{g}$ is in the image of $F$ w.h.p. and hence being an interpolating solution given \cref{eq:int-sol-Fourier}. 
Consider monomials such as $T$ for which $\forall w, b\; \hat\phi_{w,b}(T) = 0$. Due to properties (A2) and (A5), we know that such $T$'s satisfy $\deg(T) > P$, hence their corresponding rows are both zero in $F$ and in $\hat{g}$. We remove these rows from $F$ and $\hat{g}$ and call the new ones $\tilde{F}$ and $\tilde{\hat{g}}$. We also remove corresponding rows and columns from $\Phi$ and denote the new matrix by $\tilde{\Phi}$.

Note $Fa = \hat{g} \iff \tilde{F}a = \tilde{\hat{g}}$, therefore to prove that $\hat{g} \in \mathrm{Image}(F)$ its enough to show that $\tilde{F}$ is full row-rank, or equivalently, $\tilde{F}\tilde{F}^T$ is full rank. Note that $\tilde{F}\tilde{F}^T$ converges to $\tilde{\Phi}$ almost surely. 
Note that $\tilde{\Phi}$ is a diagonal matrix such that all elements on the diagonal are positive as all zero-entries of the diagonal are already removed by property (A5). Therefore $\tilde{\Phi}$ is full rank and $\tilde{F}\tilde{F}^T$ becomes full rank almost surely as $N \to \infty$. This concludes the proof of the existence of interpolators. 
\subsubsection{Learning the Min Degree-Profile Interpolating Solution}
Now, we investigate the interpolating solution found by the model. Note that we are interested in the interpolating solution with the minimum  norm $\|a\|_2$ (which is the solution found by GD starting from $a=0$). Consider an interpolating solution $g \in \mathcal{F}_{\mathrm{int}}.$ The interpolator $g$ is found by the model if and only if $Fa = \hat{g}$, where $\hat{g}$ is the Fourier expansion of $g$ written in the vector form. Moreover, note that the $a$ satisfying $Fa = \hat{g}$ with the minimum norm $\|a\|_2$ is $a^{*}_g = F^\dagger\hat{g}$, where $F^\dagger$ is the Moore-Penrose pseudo-inverse. Therefore, we have
\begin{equation}
    \|a_{\mathrm{RF}}\|^2_2 = \min_{g \in \mathcal{F}_{\mathrm{int}}, \hat{g} \in \mathrm{Im}(F)} \|F^\dagger\hat{g}\|_2^2 \Longrightarrow g_{\mathrm{RF}} = \arg\min_{g \in \mathcal{F}_{\mathrm{int}} , \hat{g} \in \mathrm{Im}(F)} \|F^\dagger\hat{g}\|_2^2.
\end{equation}
Now note that we have 
\begin{align}
    \|F^\dagger\hat{g}\|_2^2 =  \|F^T(FF^T)^\dagger\hat{g}\|_2^2 = \hat{g}^T (FF^T)^\dagger FF^T (FF^T)^\dagger \hat g
\end{align}
We know that $FF^T$ almost surely converges to $\Phi$, which is a diagonal matrix. Moreover, by property (A5), we know that the zero elements on the diagonal of $\Phi$ correspond to zero rows of $F$, and hence zero entries of $g$ since $g \in \mathrm{Im}(F)$. Thus, we can say that $(FF^T)^\dagger$ and $\|F^\dagger \hat g\|^2$ converge to $\Phi^\dagger$ and $g^T\Phi^\dagger g$ as $N \to \infty$ w.h.p. Furthermore, since $g \in \mathrm{Im}(F)$, zero entries on diagonal $\Phi$ (or $\Phi^\dagger$) correspond to zero entries of $g$, thus, we also have  
\begin{equation}
    g_{\mathrm{RF}} = \arg\min_{g \in \mathcal{F}_{\mathrm{int}} , \hat{g} \in \mathrm{Im}(F)} \|F^\dagger\hat{g}\|_2^2 \xrightarrow{N \to \infty (a.s.)} \arg\min_{g \in \mathcal{F}_{\mathrm{int}} , \hat{g} \in \mathrm{Im}(F)} g^T\Phi^\dagger g.
\end{equation}

Also note that 
\begin{equation}
    g^T\Phi^\dagger g = \sum_{T\subseteq [d]: \E_{w,b}[\hat\phi(T)^2] \neq 0} \hat{g}(T)^2\E_{w,b}[\hat\phi(T)^2]^{-1}. \label{eq:new-norm}
\end{equation}
We now focus on interpolators minimizing the quantity introduced in \cref{eq:new-norm}. 
First, note that these interpolators do not have any monomials having a variable other than latent variables $\{x_{i_1}, \ldots, x_{i_p}\}$, i.e., all of the learned monomials would be in $2^{\{x_{i_1}, \ldots, x_{i_p}\}}$. To see this, consider an interpolating solution $g$ containing such monomials, $T_1, \ldots, T_m \not\subseteq I_P= \{i_1, \ldots, i_P\}$. For simplicity, we use the notation $x=(x_{I_P}, x_{[d]\setminus I_P})$ to differentiate between latent variables and the rest of the bits.  Now define 
\begin{equation}
    g_I((x_{I_P}, x_{[d]\setminus I_P})) \coloneqq 2^{-(d-P)}\sum_{x_{[d]\setminus I_P} \in \{\pm 1\}^{d-P}} g(x). 
\end{equation}
Note that $g_I((x_{I_P}, x_{[d]\setminus I_P}))$ is independent of $x_{[d]\setminus I_P}$. Therefore $g_I(x) = g(x)$ for all the training samples. Moreover, note that 
\begin{equation}
    \hat{g}_I(T) = \begin{cases}
    \hat{g}(T) & T \subseteq I_P \\
    0 & o.w.
    \end{cases}, 
\end{equation}
which shows that $g_I\Phi^\dagger g_I < g\Phi^\dagger g $ unless $g=g_I$. Note that if $\E_{w,b}[\hat\phi(T)^2] = 0$ for some $T$, then $\hat{g}(T) =0$, since we are considering the solution learned by the RF model and $\hat{g} \in \mathrm{Im}(F)$. 
In sum, the function learned by the RF model converges to an interpolator that only contains the latent coordinates, as $N \to \infty$ w.h.p.
Note that $\mathbb{E}[\hat\phi_{w,b}(T)^2]$ is the same for all $T$ of the same size due to symmetry, we denote this shared quantity by $\hat\phi_{|T|, d}$.
Now, we revisit \cref{eq:new-norm}, for the functions defined on latent coordinates $I_P$, we have
\begin{align}
    g^T\Phi^\dagger g &= \sum_{T\subseteq [d]: \E_{w,b}[\hat\phi(T)^2] \neq 0} \hat{g}(T)^2\E_{w,b}[\hat\phi_{w,b}(T)^2]^{-1} \\
    &= \sum_{T \subseteq I_P} \hat{g}(T)^2 \mathbb{E}_{w,b}[\hat\phi_{w,b}(T)^2]^{-1}= \sum_{i=0}^{P}  \left(\sum_{T  \subseteq I_P: |T|=i} \hat{g}\left(T\right)^2 \right) \hat\phi_{|T|, d}^{-1}. \label{eq:new-norm-final-form}
\end{align}
Note that since $\sigma$ is strongly expressive up to $P$, we have $\hat \phi_{k, d}^{-1} = \Theta(d^k)$.
Putting this along \cref{eq:new-norm-final-form} shows that the solution of the RF model converges to $\mathrm{MinDegInterp} + \epsilon_d$ almost surely as $N\to\infty$, where $\epsilon_d$ is a vanishing function (w.r.t. $d$) on the latent coordinates, which concludes the proof. 
\end{proof}
Now we can easily prove Corollary \ref{cor:coef-bound}. 
\begin{proof}\textit{(Corollary \ref{cor:coef-bound})}
    Consider the random features model $f_{\mathrm{RF}}(x;a) = \frac{1}{\sqrt{N}}\sum_{i=1}^N a_i\phi_i(x)$. Similar to the proof of the main theorem, we can represent the function given by the random features model in the Fourier-Walsh basis as a vector of size $2^d$ denoted by $\hat f_{\mathrm{RF}}$. We consider the same $F \in \mathbb{R}^{2^d \times N}$ as above which represented the coefficients of monomials in individual features. As a result, $\hat f_{\mathrm{RF}} = Fa$. Note that 
    property (A1) ensures that $F^TF \rightarrow \Phi$ due to the central limit theorem as $N \to \infty$.
    With the same arguments as the main proof, we have
    \begin{equation*}
        \hat f_{\mathrm{RF}} = Fa \Longrightarrow \|a\|^2 \geq \|F^\dagger \hat f_\mathrm{RF}\|^2 = \hat f_\mathrm{RF}^T (F^\dagger)^TF^\dagger \hat f_\mathrm{RF} \xrightarrow{N \to \infty} \hat f_\mathrm{RF}^T \Phi^\dagger \hat f_\mathrm{RF}.
    \end{equation*}
    Note that $\Phi, \Phi^\dagger$ are both diagonal matrices. Now consider a monomial $\chi_S$ of size $|S| = O_d(1)$. Assuming that this monomial corresponds with index $s$ of $\hat f_\mathrm{RF}$, we have
    \begin{equation*}
        \|a\|^2 \geq \hat f_{\mathrm{RF}, s}^2 \Phi^\dagger_s \Longrightarrow \hat f_{\mathrm{RF}, s}^2 \leq \|a\|^2 \Phi_s = O(\frac{\|a\|^2}{d^{|S|}}),
    \end{equation*}
    which concludes the proof. Note that we used the fact that if $\Phi_s = \Phi_s^\dagger = 0$, this monomial cannot be generated and $f_{\mathrm{RF}, s}=0$. Also, we used the upper bound $\Phi_s = \E_{w,b}[\hat \phi (S)^2] = O(d^{-|S|})$ (part of property A3) which is a consequence of property (A1).
\end{proof}
\subsubsection{RF Model with ReLU Activation}\label{sec:rf-relu-poly-act}
In this part, we study the random features model equipped with the ReLU activation function. Here, we mostly rely on the results of \citet{AbbeINAL}. First, following  proposition B.1 of \citet{AbbeINAL}, we note that for every odd $k \geq 3$, the coefficient of $k$-th Hermite polynomial in the Hermite expansion of $\mathrm{ReLU}$ is zero. On the other hand, this coefficient is non-zero for $k=1$ and any even $k$. Consequently, following Lemma A.2 of \citet{AbbeINAL}, for monomials $\chi_T$ and $|T|\leq P = O_d(1)$ we have
\begin{equation}
    \E_{w,b}[\E_x[ \mathrm{ReLU}(\langle w, x \rangle + b) \chi_T(x)]^2] = \E_{w,b}[\hat \phi_{w,b,\mathrm{ReLU}}(T)] = \begin{cases}
        \Omega(d^{-|T|}) & |T|=1~or~even \\
\Omega(d^{-(|T|+1)}) & o.w. 
    \end{cases}, \label{eq:relu} 
\end{equation}
where $\hat \phi_{w,b,\mathrm{ReLU}}(T)$ is the coefficient of monomial $T$ in random feature created by the weights and bias $w, b$ and the ReLU activation. Informally, \cref{eq:relu} indicates that odd monomials with degrees larger than one are not strongly expressed in the random features when ReLU is used as the activation function. Nonetheless, note that as in Lemma~\ref{lemma:random-low-deg}, we can still deduce that 
$\E_{w,b}[\hat \phi_{w,b,\mathrm{ReLU}}(T)] = O(d^{-|T|})$ for $|T| \leq P=O_d(1)$. This upper bound along with the lower bounds obtained in \cref{eq:relu} and the minimization problem of \cref{eq:new-norm-final-form} indicate that the random features model with ReLU activation would replace degree $2$ or $2k+1$ monomials with lower degree monomials if possible. However, it might not replace degree $2k+2$ monomials with degree $2k+1$ monomials for $k \geq 1$.
We further illustrate this with an experiment. 


We consider learning $f(x_0, \ldots, x_{14}) = x_0x_1x_2+x_0x_3x_4x_5$ under the unseen domain $\mathcal{U} = \{x \in \{\pm 1\}^{14}|x_0 = -1\}$. Note that in this case, the min-degree interpolator is $x_1x_2+x_3x_4x_5$. However, for the ReLU activation, we know that $x_3x_4x_5$ would not necessarily be preferred to $x_0x_3x_4x_5$ since the $\deg(x_3x_4x_5)=3$ is odd. In \cref{fig:rf-comparison}, we compare the solution learned by the RF model with ReLU and polynomial activation (here $(1+x)^6$). It can be seen that the polynomial activation learns the MD interpolator, whereas the RF with the ReLU activation function only learns the lower-degree monomial for the odd monomial and not for the even one. 
\begin{figure}[H]
     \centering
     \begin{subfigure}[b]{0.4\textwidth}
         \centering
         \includegraphics[width=\textwidth]{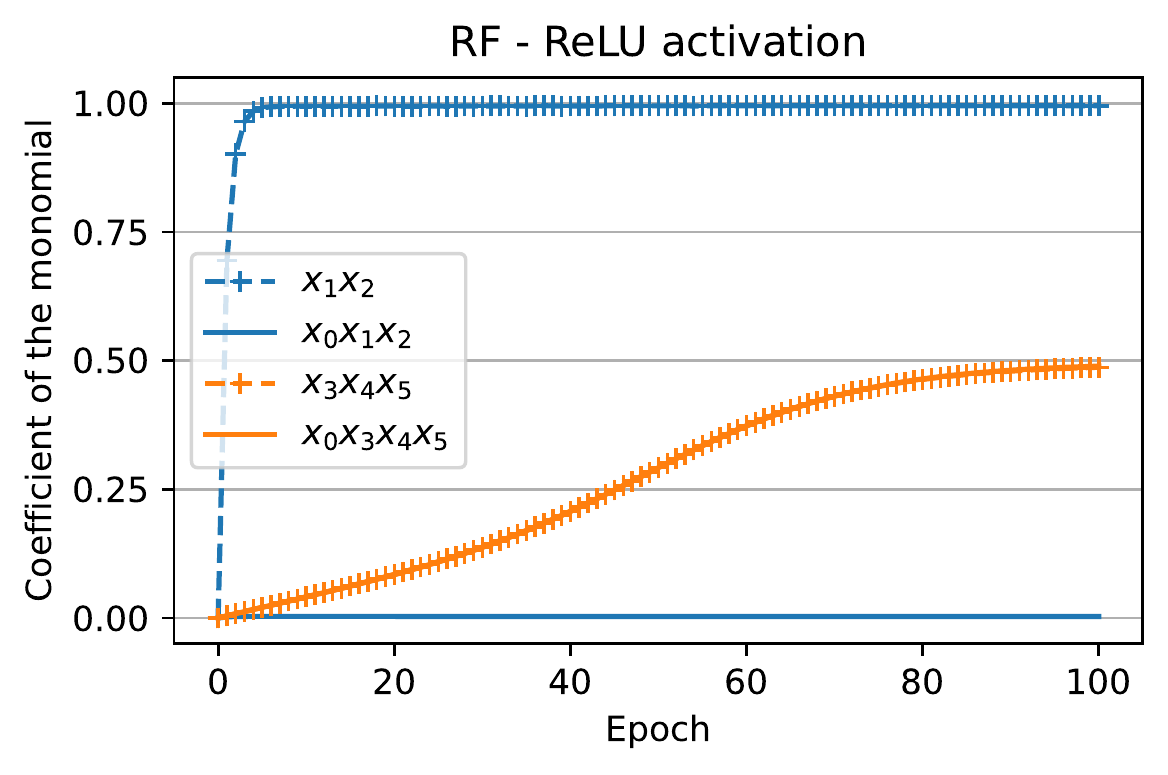}
     \end{subfigure}
     \hfill
     \begin{subfigure}[b]{0.4\textwidth}
         \centering
         \includegraphics[width=\textwidth]{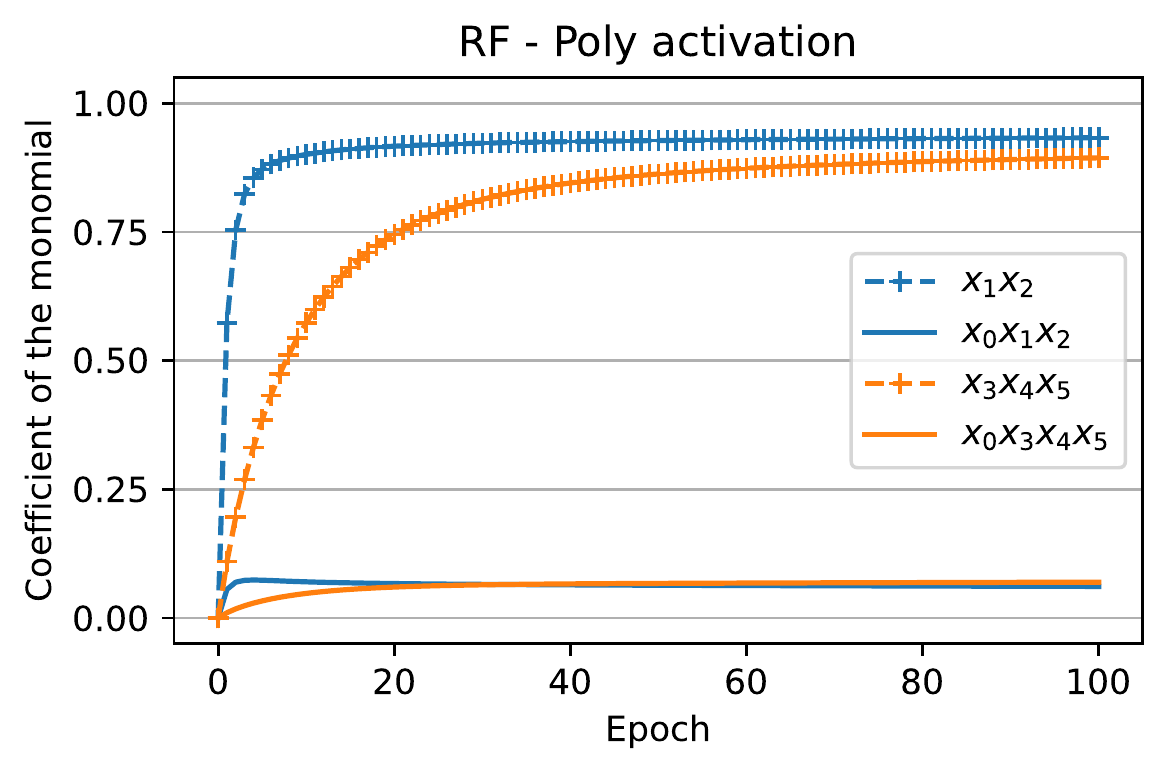}
     \end{subfigure}

        \caption{Target function $f(x_0, \ldots, x_{14}) = x_0x_1x_2+x_0x_3x_4x_5$ being learned by random features models under $\mathcal{U} = \{x_0 = -1\}$. The RF model with strongly expressive activation (here $(1+x)^6$) learns the min-degree interpolator (right), while the min-degree bias of the RF model with ReLU activation depends on the degree of monomials being even or odd (left). More precisely, the RF model does not prefer degree $2k+1$ monomial to degree $2k+2$ monomial for $k \geq 1$. Note that for the RF with ReLU activation (left), the coefficients of $x_3x_4x_5$ and $x_0x_3x_4x_5$ are equal and hence overlap.}
        \label{fig:rf-comparison}
\end{figure}

\subsection{Proof for Diagonal Linear Neural Networks}
Here, we present the proof of Theorem \ref{thm:diagonal}.
\begin{proof}
	We denote parameters at time $t$ by $\theta(t)$. Also, we consider the training under half $\ell_2$ loss, to simply remove the $2$ factor from gradients. Consider the (half) $\ell_2$ loss function for a training sample $x$, we have
	\begin{align}
	L(\theta(t), x, f) &= \frac{1}{2}\left(f_{\mathrm{NN}}(x) - f(x)\right)^2 \\
	&= \frac{1}{2}\left(\left(b - \hat{f}(\emptyset)\right) + \sum_{i=1}^d \left(\prod_{l=1}^{L} w_{i}^{(l)} - \hat{f}(\{i\})\right)x_i\right)^2.
	\end{align}
Moreover, we know every component of the training sample is sampled from $\mathrm{Rad}(\frac{1}{2})$, except the frozen bit which is set to $x_k = 1$. We denote this uniform distribution by $U^{d-1}_{-k}$. Given this, the expected loss of the training set can be calculated as follows
\begin{align}
\E_{U^{d-1}_{-k}}&[L(\theta(t), x, f)] = \frac{1}{2}\E_{U^{d-1}_{-k}}\left[\left(\left(b - \hat{f}(\emptyset)\right) + \sum_{i=1}^d \left(\prod_{l=1}^{L} w_{i}^{(l)} - \hat{f}(\{i\})\right)x_i\right)^2\right] \nonumber\\
&= \frac{1}{2}\E_{U^{d-1}_{-k}}\left[\left(\left((b + \prod_{l=1}^{L} w_{k}^{(l)}) - (\hat{f}(\emptyset) + \hat{f}(\{k\}))\right) + \sum_{i\neq k}^d \left(\prod_{l=1}^{L} w_{i}^{(l)} - \hat{f}(\{i\})\right)x_i\right)^2\right] \nonumber\\
&= \frac{1}{2}\left((b + \prod_{l=1}^{L} w_{k}^{(l)}) - (\hat{f}(\emptyset) + \hat{f}(\{k\}))\right)^2 + \frac{1}{2}\sum_{i\neq k}^d \left(\prod_{l=1}^{L} w_{i}^{(l)} - \hat{f}(\{i\})\right)^2, \label{eq:loss-parseval} 
\end{align}
where we have used Parseval's theorem \citep{o'donnell_2014} to get the last equation. For simplicity, we define $B \coloneqq~\hat{f}(\emptyset) + \hat{f}(\{k\})$ and $B_{\mathrm{NN}} \coloneqq b + \prod_{l=1}^{L} w_{k}^{(l)}$ as the total bias of the target function and the neural network respectively. We know the gradient flow (GF) of the parameters of the neural network is given by
\begin{align}
\dot \theta = -\nabla_\theta \E_{U^{d-1}_{-k}}[L(\theta(t), x, f)].	
\end{align}
Therefore, using (\ref{eq:loss-parseval}), we can derive the gradient flow for each of the parameters as below
\begin{align}
\dot b &= -\nabla_b \E_{U^{d-1}_{-k}}[L(\theta(t), x, f)]	 = -(b + \prod_{l=1}^{L} w_{k}^{(l)}) + (\hat{f}(\emptyset) + \hat{f}(\{k\})) = -(B_{\mathrm{NN}} - B),\label{eq:bias-bias-derivative}\\
\dot w_{k}^{(l)} &= -\nabla_{w_{k}^{(l)}}  \E_{U^{d-1}_{-k}}[L(\theta(t), x, f)]	 = -((b + \prod_{j=1}^{L} w_{k}^{(j)}) + (\hat{f}(\emptyset) + \hat{f}(\{k\})))\prod_{j\neq l}^{L} w_{k}^{(j)} \\
&= -\prod_{j\neq l}^{L} w_{k}^{(j)}(B_{\mathrm{NN}} - B), \nonumber \\
\forall i &\neq k,~~\dot w_{i}^{(l)} = -\nabla_{w_{i}^{(l)}}  \E_{U^{d-1}_{-k}}[L(\theta(t), x, f)]	 = -\left(\prod_{j=1}^{L} w_{i}^{(j)} - \hat{f}(\{i\})\right)\prod_{j\neq l}^{L} w_{i}^{(j)}.
\end{align}
Using the above, we can derive the balancedness property of the neural network, i.e., 
\begin{align}
	\frac{d}{dt}(w_{k}^{(l)})^2 &= 2 w_{k}^{(l)} \dot w_{k}^{(l)} = -2\prod_{j = 1}^{L} w_{k}^{(j)}(B_{\mathrm{NN}} - B) =  2 w_{k}^{(l')} \dot w_{k}^{(l')} = \frac{d}{dt}(w_{k}^{(l')})^2,\\
	\forall i\neq k, ~~ \frac{d}{dt}(w_{i}^{(l)})^2 &= 2 w_{i}^{(l)} \dot w_{i}^{(l)} = -2\left(\prod_{j=1}^{L} w_{i}^{(j)} - \hat{f}(\{i\})\right)\prod_{j=1}^{L} w_{i}^{(j)} =  2 w_{i}^{(l')} \dot w_{i}^{(l')} = \frac{d}{dt}(w_{i}^{(l')})^2.
\end{align}
Therefore, $\forall i~~ (w_{i}^{(l)})^2 - (w_{i}^{(l')})^2$ is constant during training.
Using this property, we can show that most of the model's parameters are always bounded away from $0$ during training. To see this, fix an index $i \in [d]$. Let $j_{i}^* = \mathrm{argmin}_{j \in [L]} |w_{i}^{(j)}(0)|$. Furthermore, define
\begin{equation}
	c_i \coloneqq \min_{j\neq j_{i}^* \in [L]} (w_{i}^{(j)}(0))^2 - (w_{i}^{(j_{i}^*)}(0))^2 \geq 0.
\end{equation}
Since the model parameters are initialized randomly using the uniform distribution, we can say that $c_i > 0$ with probability 1.
Now, due to the balancedness property, we know that 
\begin{align}
	\forall j\neq j_{i}^*, ~~ (w_{i}^{(j)}(t))^2 - (w_{i}^{(j_{i}^*)}(t))^2 &= (w_{i}^{(j)}(0))^2 - (w_{i}^{(j_{i}^*)}(0))^2 \geq c_i \Longrightarrow \nonumber\\
 (w_{i}^{(j)}(t))^2 &\geq c_i + (w_{i}^{(j_{i}^*)}(t))^2 \geq c_i. \label{eq:bounded-ness}
\end{align}
Now we are able to show the convergence of the model. To begin with, note that
\begin{align}
	\frac{d}{dt}(\prod_{l=1}^{L} w_{k}^{(l)}) &= \sum_{l=1}^{L} \dot w_{k}^{(l)} \prod_{j\neq l}w_{k}^{(j)} = - \left(\sum_{l=1}^{L} (\prod_{j\neq l}^{L} w_{k}^{(j)})^2 \right)(B_{\mathrm{NN}} - B),  \label{eq:prod-bias-derivative}\\
	\forall i\neq k, ~~\frac{d}{dt}(\prod_{l=1}^{L} w_{i}^{(l)}) &= \sum_{l=1}^{L} \dot w_{i}^{(l)} \prod_{j\neq l}w_{i}^{(j)} = - \left(\sum_{l=1}^{L} (\prod_{j\neq l}^{L} w_{i}^{(j)})^2 \right)\left(\prod_{l=1}^{L} w_{i}^{(l)} - \hat{f}(\{i\})\right).
\end{align}
Now, first, we consider an index $i \neq k$. We have
\begin{align}
	\frac{d}{dt} \left(\prod_{l=1}^{L} w_{i}^{(l)} - \hat{f}(\{i\})\right)^2 &= 2 \left(\prod_{l=1}^{L} w_{i}^{(l)} - \hat{f}(\{i\})\right)\frac{d}{dt} \left(\prod_{l=1}^{L} w_{i}^{(l)} - \hat{f}(\{i\})\right) \nonumber\\
	&= - 2\left(\sum_{l=1}^{L} (\prod_{j\neq l}^{L} w_{i}^{(j)})^2 \right)\left(\prod_{l=1}^{L} w_{i}^{(l)} - \hat{f}(\{i\})\right)^2.
\end{align}
Now using (\ref{eq:bounded-ness}), we can say
\begin{equation}
	\left(\sum_{l=1}^{L} (\prod_{j\neq l}^{L} w_{i}^{(j)})^2 \right) \geq (\prod_{j\neq j^{*}_{i}}^{L} w_{i}^{(j)})^2 \geq c_i^{L-1} > 0.
\end{equation}
Therefore, we have
\begin{align*}
	\frac{d}{dt} \left(\prod_{l=1}^{L} w_{i}^{(l)} - \hat{f}(\{i\})\right)^2 &= - 2\left(\sum_{l=1}^{L} (\prod_{j\neq l}^{L} w_{i}^{(j)})^2 \right)\left(\prod_{l=1}^{L} w_{i}^{(l)} - \hat{f}(\{i\})\right)^2 \\
 &\leq -2c_i^{L-1}\left(\prod_{l=1}^{L} w_{i}^{(l)} - \hat{f}(\{i\})\right)^2,
\end{align*}
which shows
\begin{equation}
	\left(\prod_{l=1}^{L} w_{i}^{(l)}(t) - \hat{f}(\{i\})\right)^2 \leq \left(\prod_{l=1}^{L} w_{i}^{(l)}(0) - \hat{f}(\{i\})\right)^2 e^{-2c_i^{L-1}t}; \label{eq:converge-w}
\end{equation}
in other words, $\left(\prod_{l=1}^{L} w_{i}^{(l)} - \hat{f}(\{i\})\right)^2$ goes to $0$ exponentially fast in time, $t$. Finally, we make the same analysis for $(B_{\mathrm{NN}}-B)^2$. We have
\begin{align*}
	\frac{d}{dt}(B_{\mathrm{NN}}-B)^2 &= \frac{d}{dt}\left((b + \prod_{l=1}^{L} w_{k}^{(l)}) - B)\right)^2 = 2\left((b + \prod_{l=1}^{L} w_{k}^{(l)}) - B)\right)
\frac{d}{dt}\left(b + \prod_{l=1}^{L} w_{k}^{(l)}\right) \nonumber \\
&= 2\left((b + \prod_{l=1}^{L} w_{k}^{(l)}) - B)\right)
\left(-(B_{\mathrm{NN}}-B) - \left(\sum_{l=1}^{L} (\prod_{j\neq l}^{L} w_{k}^{(j)})^2 \right)(B_{\mathrm{NN}} - B)\right) \nonumber\\
&= -2(B_{\mathrm{NN}}-B)^2\left(1 + \left(\sum_{l=1}^{L} (\prod_{j\neq l}^{L} w_{k}^{(j)})^2 \right)\right) \leq -2(B_{\mathrm{NN}}-B)^2. 
\end{align*}
The last equation shows that 
\begin{equation}
	(B_{\mathrm{NN}}(t)-B)^2 \leq (B_{\mathrm{NN}}(0)-B)^2 e^{-2t}, \label{eq:converge-b} 
\end{equation}
i.e., $(B_{\mathrm{NN}}(t)-B)^2$ converges to $0$ exponentially fast in $t$ as well. 
Equations (\ref{eq:loss-parseval}), (\ref{eq:converge-w}), and (\ref{eq:converge-b}) show that
\begin{equation}
	L(\theta(t), x, f) \leq L(\theta(0), x, f)e^{-ct}, 
\end{equation}
where $c = 2\min(1, \min(\{c_i\}_{i\neq k})^{L-1})$; hence, loss converges to zero exponentially fast in time (however, it is still initialization-dependent).

As shown in (\ref{eq:converge-b}), the bias of neural network, converges like $(B_{\mathrm{NN}}(t)-B)^2 \leq (B_{\mathrm{NN}}(0)-B)^2 e^{-2t}$. 
We denote $R \coloneqq|\hat{f}(\emptyset) + \hat{f}(\{k\})| + 1 > |B_{\mathrm{NN}}(0)-B|$. Now notice that if $t \geq T_\epsilon \coloneqq \log \frac{8R}{\epsilon}$, then we have
\begin{equation}
	(B_{\mathrm{NN}}(t)-B)^2 \leq (B_{\mathrm{NN}}(0)-B)^2 e^{-2t} \leq R^2 e^{-\log \frac{64R^2}{\epsilon^2}} = \frac{\epsilon^2}{64}.
\end{equation}
We now show the growth of $\prod_{l=1}^{L} w_{k}^{(l)}$ is comparatively slower, and therefore, it will not capture the bias fast enough and will remain small during the entire training process. More precisely, we first bound $\prod_{l=1}^{L} w_{k}^{(l)}$ at the beginning of training ($t \leq T_{\epsilon}$). We define $m = \mathrm{argmax}_{l \in [L]} |w_k^{(l)}(0)|$. Again, by balancedness property, we know it will remain the largest during training, i.e.,
\begin{equation}
	|w_k^{(i)}| = \sqrt{(w_k^{(i)})^2} \leq \sqrt{(w_k^{(m)})^2} \leq |w_k^{(m)}|.
\end{equation}
Now note that
\begin{align}
	\frac{d}{dt}(w_k^{(m)})^2 &= -2\prod_{j = 1}^{L} w_{k}^{(j)}(B_{\mathrm{NN}}(t) - B) \leq 2|\prod_{j = 1}^{L} w_{k}^{(j)}(B_{\mathrm{NN}}(t) - B)| \nonumber\\
	&\leq 2|w_k^{(m)}|^L|B_{\mathrm{NN}}(t) - B| \nonumber\\
	&\leq 2((w_k^{(m)})^2)^{\frac{L}{2}}|B_{\mathrm{NN}}(0) - B| = 2((w_k^{(m)})^2)^{\frac{L}{2}} R, 
\end{align}
where in the last line we used the fact that $(B_{\mathrm{NN}}(t) - B)^2$ is decreasing. Now, we provide a bound for $|w_k^{(m)}|$. First, we consider the case that $L=2$. In this case, we have
\begin{equation}
	\frac{d}{dt}(w_k^{(m)})^2 \leq 2((w_k^{(m)})^2)R \Longrightarrow (w_k^{(m)}(t))^2 \leq (w_k^{(m)}(0))^2 e^{2Rt}, 
\end{equation}
where we used Gronwall's lemma in the last equation. It also shows 
\begin{equation}
	\prod_{l=1}^{L} w_{k}^{(l)}(t) \leq w_{k}^{(m)}(t)^L = w_{k}^{(m)}(t)^2 \leq (w_k^{(m)}(0))^2 e^{2Rt}.
\end{equation}
Now, we consider the case that $L > 2$. In this case, we also have (this could be considered as an extension of Gronwall's lemma, note that $w_k^{(m)} > 0$)
\begin{align}
	\frac{d}{dt}(w_k^{(m)})^2 &\leq 2((w_k^{(m)})^2)^{\frac{L}{2}}R \Longrightarrow \\
	\frac{d}{dt}((w_k^{(m)})^2)^{1-\frac{L}{2}} &= -(\frac{L}{2}-1)((w_k^{(m)})^2)^{-\frac{L}{2}}\frac{d}{dt}(w_k^{(m)})^2 \geq -(L-2) 
	R, 
\end{align}
using the above we have
\begin{align} 
(w_k^{(m)}(t)^2)^{1-\frac{L}{2}} &- (w_k^{(m)}(0)^2)^{1-\frac{L}{2}}	=\int_{0}^{t} \frac{d}{dt}(w_k^{(m)}(t)^2)^{1-\frac{L}{2}} \geq -(L-2)Rt \Longrightarrow \\
w_k^{(m)}(t)^2 &\leq \frac{1}{(|w_k^{(m)}(0)|^{2-L}-(L-2)Rt)^{\frac{1}{\frac{L}{2} - 1}}} ~~~~~~~t < \frac{|w_k^{(m)}(0)|^{2-L}}{(L-2)R}, 
\end{align}
hence, we have
\begin{equation}
\prod_{l=1}^{L} w_{k}^{(l)}(t) \leq (w_k^{(m)}(t)^2)^{\frac{L}{2}} \leq \frac{1}{(|w_k^{(m)}(0)|^{2-L}-(L-2)Rt)^{\frac{L}{L - 2}}} ~~~~~~~t < \frac{|w_k^{(m)}(0)|^{2-L}}{(L-2)R}. \label{eq:bound-general}
\end{equation}
Now we consider each of these bounds at $t = T_\epsilon$. First, for $L=2$, we have
\begin{equation}
	\prod_{l=1}^{L} w_{k}^{(l)}(t) \leq (w_k^{(m)}(0))^2 e^{2Rt} = (w_k^{(m)}(0))^2e^{2RT_\epsilon},
\end{equation}
which is upper bounded by $\frac{\epsilon}{8}$ if $(w_k^{(m)}(0))^2 \leq \alpha^2 \leq \alpha_{max}^2 = \frac{\epsilon}{8e^{2RT_\epsilon}}$.
Now, we consider the bound for deeper networks, $L > 2$, at time $t = T_\epsilon$. We want to bound $\prod_{l=1}^{L} w_{k}^{(l)}(t)$ by $\frac{\epsilon}{8}$. Using (\ref{eq:bound-general}) this will happen if we have
\begin{align}
	\frac{1}{(|w_k^{(m)}(0)|^{2-L}-(L-2)RT_\epsilon)^{\frac{L}{L - 2}}} \leq \frac{\epsilon}{8} \iff (L-2)RT_\epsilon + (\frac{8}{\epsilon})^{\frac{L-2}{L}} \leq |w_k^{(m)}(0)|^{2-L},
\end{align}
which will happen if $|w_k^{(m)}(0)| \leq \alpha_{\max} \coloneqq ((L-2)RT_\epsilon + (\frac{8}{\epsilon})^{\frac{L-2}{L}})^{\frac{1}{2-L}}$.

So we proved for small enough initializations, there exists a time, $T_\epsilon$, where
\begin{align}
|b(T_\epsilon) + \prod_{l=1}^{L} w_{k}^{(l)}(T_\epsilon) - B| &\leq \frac{\epsilon}{8}, \\
|\prod_{l=1}^{L} w_{k}^{(l)}(T_\epsilon)|  &\leq \frac{\epsilon}{8}, \\
|b(T_\epsilon) - B| &\leq |b(T_\epsilon) + \prod_{l=1}^{L} w_{k}^{(l)}(T_\epsilon) - B| + |\prod_{l=1}^{L} w_{k}^{(l)}(T_\epsilon)|  \leq \frac{2\epsilon}{8}.
\end{align}
We now show that this picture will not change much during the rest of the training process. To see this, note that $|B_{\mathrm{NN}}(t) - B|$ is always decreasing over time and is continuous. Therefore, $B_{\mathrm{NN}}(t) - B$ cannot change the sign (since changing the sign means that the variable had become equal to $0$ at some time, which is contrary to the fact that its absolute value is decreasing). Considering equations (\ref{eq:bias-bias-derivative}) and (\ref{eq:prod-bias-derivative}) we can conclude that both $b(t) - B$ and $\prod_{l=1}^{L} w_{k}^{(l)}(t)$ are either increasing or decreasing during the whole training. 
First, assume both of them are increasing. For $t > T_\epsilon$, we have
\begin{align}
	|\prod_{l=1}^{L} w_{k}^{(l)}(t) + b(t) - B| &\leq |\prod_{l=1}^{L} w_{k}^{(l)}(T_\epsilon) + b(T_\epsilon) - B| \leq \frac{\epsilon}{8} \Longrightarrow \\
	\frac{-\epsilon}{8} \leq \prod_{l=1}^{L} w_{k}^{(l)}(T_\epsilon) \leq \prod_{l=1}^{L} w_{k}^{(l)}(t) &\leq \frac{\epsilon}{8} - (b(t) - B) \leq \frac{\epsilon}{8} - (b(T_\epsilon) - B) \leq \frac{3\epsilon}{8} \Longrightarrow \\
 |\prod_{l=1}^{L} w_{k}^{(l)}(t)| &\leq  \frac{3\epsilon}{8},	\\
	|b(t) - B| &\leq |\prod_{l=1}^{L} w_{k}^{(l)}(t) + b(t) - B| + |\prod_{l=1}^{L} w_{k}^{(l)}(t)| \leq \frac{4\epsilon}{8}.
\end{align}
The case for both functions being decreasing is also similar. This shows that $f_{\mathrm{NN}}(\{k\}) < \epsilon$ during the entire training. Now we can study GOTU loss for $t \geq T_\epsilon$ using Parseval's theorem as follows:
\begin{align}
GOTU(f, f_{\mathrm{NN}}, \{x_k = - 1\}) 
	&=((b- \prod_{l=1}^{L} w_{k}^{(l)}) - (\hat{f}(\emptyset) - \hat{f}(\{k\})))^2 + \sum_{i\neq k}^d (\prod_{l=1}^{L} w_{i}^{(l)} - \hat{f}(\{i\}))^2 \\
	&= ((b-B) - \prod_{l=1}^{L} w_{k}^{(l)} + 2\hat{f}(\{k\}))^2+ O(e^{-ct})  \\ 
	&= 4\hat{f}(\{k\})^2 + O_t(e^{-ct}) + O_\epsilon(\epsilon), 
\end{align}
which proves the theorem. Note that if we consider half $\ell_2$ loss for the entire population $\Omega$ the loss becomes $\hat{f}(\{k\})^2 + O_t(e^{-ct}) + O_\epsilon(\epsilon)$.
\end{proof}
\begin{remark}[Initialization of bias variable]
	Note that the analysis is independent of the initialization of the bias variable (as long as it satisfies a simple bound such as $|b(0)| \leq \frac{1}{2}$).
\end{remark}
\begin{remark}[Effect of depth] \label{remark:1}
	 The current theorem proves that the low-degree solution is learned when the initialization scale is small enough. To see the effect of depth, we prove that $\alpha_{\max}$ found in this proof is increasing by depth, $L$. In other words, if we have deeper networks, we can use larger initializations and still have the generalization error close to the Boolean influence. 
\end{remark}
\begin{proof}\textit{(Remark \ref{remark:1})}\label{proof:remark-diag-depth}
Consider $L \geq 3$. We know that 
$\alpha_{\max} \coloneqq ((L-2)RT_\epsilon + (\frac{8}{\epsilon})^{\frac{L-2}{L}})^{\frac{1}{2-L}}$. For simplicity define 
\begin{align}
	P &\coloneqq RT_\epsilon, \\
	Q &\coloneqq \frac{8}{\epsilon} > e^3 ~(\mathrm{we~assume~this}), \\
	g(x) &\coloneqq (xP + Q^{\frac{x}{x+2}})^{\frac{-1}{x}}.
\end{align}
Now note that $\alpha_{\max} = g(L-2)$. Therefore, we need to prove $g(x)$ is increasing for $x \geq 1$. To see this, note that
\begin{align}
	\frac{d}{dx} \frac{x}{x+2} &= \frac{2}{(x+2)^2}, \\
	\frac{d}{dx} Q^\frac{x}{x+2} &= Q^\frac{x}{x+2} (\ln Q) \frac{2}{(x+2)^2}, \\
	\frac{d}{dx} \ln (xP + Q^\frac{x}{x+2}) &= \frac{P + Q^\frac{x}{x+2} (\ln Q) \frac{2}{(x+2)^2}}{xP + Q^\frac{x}{x+2}}, \\
	\frac{d}{dx} \frac{-\ln (xP + Q^\frac{x}{x+2})}{x} &= \frac{-x\frac{P + Q^\frac{x}{x+2} (\ln Q) \frac{2}{(x+2)^2}}{xP + Q^\frac{x}{x+2}}+ \ln (xP + Q^\frac{x}{x+2})}{x^2}. 
\end{align}
Therefore, $\frac{d}{dx} \frac{-\ln (xP + Q^\frac{x}{x+2})}{x} \geq 0$, iff
\begin{align*}
	\ln (xP + Q^\frac{x}{x+2}) &\geq x\frac{P + Q^\frac{x}{x+2} (\ln Q) \frac{2}{(x+2)^2}}{xP + Q^\frac{x}{x+2}} \iff\\
(xP + Q^\frac{x}{x+2})\ln (xP + Q^\frac{x}{x+2}) &\geq xP + Q^\frac{x}{x+2} (\ln Q) \frac{2x}{(x+2)^2},
\end{align*}
which holds because
\begin{align*}
	xP\ln (xP + Q^\frac{x}{x+2}) &\geq xP \ln (Q^\frac{1}{3}) \geq xP, \\
	Q^\frac{x}{x+2}\ln (xP + Q^\frac{x}{x+2})&\geq Q^\frac{x}{x+2}\ln (Q)\frac{x}{x+2} \geq Q^\frac{x}{x+2} (\ln Q) \frac{2x}{(x+2)^2}.
\end{align*}
Therefore, $\exp(\frac{-\ln (xP + Q^\frac{x}{x+2})}{x}) = g(x)$ is increasing. Finally, we have to compare $\alpha_{\max}$ for depths $2$ and $3$. Note that for depth two $\alpha_{\max}(2) = \sqrt{\frac{\epsilon}{8}}e^{-RT_\epsilon} = \sqrt{\frac{1}{Q}}e^{-P}$ while for depth three, we have
$\alpha_{\max}(3) = \frac{1}{P + \sqrt[3]{Q}}$. Therefore, we have
\begin{equation*}
	\frac{1}{\alpha_{\max}(2)} = e^P \sqrt{Q} \geq (P + 1)\sqrt{Q} \geq P + \sqrt[3]{Q} = \frac{1}{\alpha_{\max}(3)},
\end{equation*}
which gives the desired result. 
\end{proof}
\subsection{Proof for 2-Layer Fully Connected Linear Network}
Here, we provide the proof for Theorem \ref{thm:2-linear} and Remark \ref{remark:equivalence-linear}.
\begin{proof}(\textit{Theorem \ref{thm:2-linear}})
    We recall that the loss function is 
    \begin{align*}
        L(t) = \frac{1}{2}\left(\|W_1w_2 - w^* \|^2 + (w_2^Tb_1 + b_2 - b^*)^2\right).
    \end{align*}
    We define $r_{w} \coloneqq W_1w_2 - w^* $ and $r_{b} \coloneqq w_2^Tb_1 + b_2 - b^*$ which in turn assess the reconstruction of the weights and the bias. Using the gradient flow, each of the parameters follows the update rule presented below. (We often use the dot notation for derivatives with respect to time.)
    \begin{align*}
        \dot w_2 &= \frac{dw_2}{dt} = -\nabla_{w_2}L = -W_1^Tr_w - b_1r_b, \\
        \dot W_1 &= \frac{dW_1}{dt} = -\nabla_{W_1}L = -r_ww_2^T, \\
        \dot b_2 &= \frac{db_2}{dt} = -\nabla_{b_2}L = -r_b, \\
        \dot b_1 &= \frac{db_1}{dt} = -\gamma\nabla_{b_1}L = -\gamma r_bw_2.
    \end{align*}
    First note that 
    \begin{align}
        \frac{dL}{dt} = \nabla_{\theta}L^T \frac{d\theta}{dt} = - \|W_1^Tr_w + b_1r_b\|^2 - \|r_w\|^2\|w_2\|^2 - r_b^2 -\gamma r_b^2 \|w_2\|^2.
    \end{align}
    Thus, the loss function is decreasing and hence
    \begin{equation}
        \|r_w(t)\|^2, r_b(t)^2 \leq L(t) \leq L(0) = \theta_\alpha(1).
    \end{equation}
    Further, we can check that 
    \begin{align*}
        \frac{d}{dt}(w_2w_2^T) = \dot w_2 w_2^T+w_2 \dot w_2^T &= -W_1^Tr_ww_2^T - b_1w_2^Tr_b - w_2r_w^TW_1 - r_bw_2b_1^T  \\
        &= (-W_1^Tr_ww_2^T  - w_2r_w^TW_1) + (-r_bw_2b_1^T - b_1w_2r_b) \\
        &= (W_1^T\dot W_1 + \dot W_1^TW_1) + \gamma^{-1}(\dot b_1 b_1^T + b_1 \dot b_1^T) \\
        &= \frac{d}{dt}(W_1^TW_1) + \gamma^{-1}\frac{d}{dt}(b_1b_1^T).
    \end{align*}

As a result, we get the following conservation rule
\begin{equation}
    w_2w_2^T(t) - W_1^TW_1(t) - \gamma^{-1}b_1b_1^T(t) = w_2w_2^T(0) - W_1^TW_1(0) - \gamma^{-1}b_1b_1^T(0),
\end{equation}
taking traces from both sides, we have
\begin{equation}
    \|w_2(t)\|^2 - \|W_1(t)\|^2_F - \gamma^{-1}\|b_1(t)\|^2 = \|w_2(0)\|^2 - \|W_1(0)\|^2_F - \gamma^{-1}\|b_1(0)\|^2.
\end{equation}
Now we are ready to move to phase 1 of the training. 
\\\textit{Phase 1.} There exists vanishing $\alpha_1 = o_\alpha(1)$ such that when we start training we reach a point such that $\|b_1\|, \|W_1\|_F, \|w_2\|, |r_b| \leq \alpha_1$. In other words, the bias of the neural network has been learned and all parameters other than the last layer's bias have stayed small.

First, we bound the growth of $\|w_2\|, \|W_1\|_F, \|b_1\|$. Let $m=\max\{\|w_2\|^2, \|W_1\|_F^2, \|b_1\|^2\}$. Note that 
\begin{align*}
    \frac{d\|w_2\|^2}{dt} &= -2w_2^TW_1^Tr_w-w_2^Tb_1r_b \leq 4m\sqrt{L(0)}, \\
    \frac{d\|W_1\|^2_F}{dt} &= -2\mathrm{tr}(W_1^Tr_ww_2^T) \leq 2m\sqrt{L(0)}, \\
    \frac{d\|b_1\|^2}{dt} &= -2\gamma b_1^T r_b w_2 \leq 2\gamma \sqrt{L(0)}.
\end{align*}
As a result, for $c_1 \coloneqq \max\{4, 2\gamma\}\sqrt{L(0)}$, we have $ \frac{d\|w_2\|^2}{dt}, \frac{d\|W_1\|^2_F}{dt}, \frac{d\|b_1\|^2}{dt} \leq c_1 m$. Thus
\begin{equation}
    m(t) \leq m(0) + \int_{0}^{t} c_1 m(t) dt \Longrightarrow m(t) \leq m(0) e^{c_1t}, \label{eq:bound-phase1-m}
\end{equation}
where we used Grönwall's lemma to bound $m(t)$. Also note that at the initialization, $m(0) = O(\alpha^2)$.

Now we focus on the reduction rate of $r_b$. Note that 
\begin{align}
    \dot r_b &= \frac{d}{dt}(w_2^Tb_1 + b_2 - b^*) = -r_w^TW_1b_1 - r_b \|b_1\|^2 - r_b = - r_b(1 + \|b_1\|^2) - r_w^TW_1b_1, \label{eq:update_rb}\\
    \dot r_b^2 &= -2r_b^2(1 + \|b_1\|^2) - 2r_br_w^TW_1b_1 
\end{align}
Define $\beta_1=({2\alpha^2 \max\{\|\overline{w_2}\|^2, \|\overline{W_2}\|^2_F, \|\overline{b_1}\|^2\} L(0)^{c_1}})^{1/(c_1+1)}$ and $t_1$ to be the first moment such that $m=\max\{\|w_2\|^2, \|W_1\|_F^2, \|b_1\|^2\} \geq \beta_1$ or $r_b^2 \leq \beta_1$. For $0 \leq t \leq t_1$, we have (for $\alpha$ small enough)
\begin{align*}
     \dot r_b^2 &= -2r_b^2(1 + \|b_1\|^2) - 2r_br_w^TW_1b_1 \leq -r_b^2 \Longrightarrow r_b^2(t) \leq r_b^2(0)e^{-t} \leq L(0)e^{-t},
\end{align*}
where we used Grönwall's lemma again. We can also conclude that
\begin{align*}
    \beta_1 \leq r_b(t_1)^2 \leq L(0)e^{-{t_1}} \Longrightarrow t_1 \leq \log \frac{L(0)}{\beta_1}.
\end{align*}
Now, given Equation \eqref{eq:bound-phase1-m}, we have
\begin{align*}
    m(t_1) \leq m(0)e^{c_1t_1} \leq \alpha^2 \max\{\|\overline{w_2}\|^2, \|\overline{W_2}\|^2_F, \|\overline{b_1}\|^2\} (\frac{L(0)}{\beta_1})^{c_1} = \frac{\beta_1}{2},
\end{align*}
where we used the definition of $\beta_1$. Therefore, $m(t_1) \leq \frac{\beta_1}{2}$, and hence, it is $r_b^2$ that must have become less than or equal to $\beta_1$. Indeed, if we define $\alpha_1 = \sqrt{\beta_1} = o_\alpha(1)$ we can see that at $t_1$ we reach to a point such that $\|b_1\|, \|W_1\|_F, \|w_2\|, |r_b| \leq \alpha_1$.

Now we can analyze phase 2 of the dynamics.

\textit{Phase 2.} Assuming that $\|b_1\|, \|W_1\|_F, \|w_2\|, |r_b| \leq \alpha_1$ for vanishing $\alpha_1$, we can show that there exists vanishing $\alpha_2 = o_{\alpha_1}(1)$ such that we reach to a point satisfying $\|r_w\|, |r_b|, \|b_1\| \leq \alpha_2$. In other words, during training, we reach to point that both the weights and the bias of the target is almost learned by the neural network and first layer's bias $b_1$ is still small. 

In this phase, we are going to use the assumption 
\begin{align*}
    \|w_2(t)\|^2 - \|W_1(t)\|^2_F - \gamma^{-1}\|b_1(t)\|^2 &= \|w_2(0)\|^2 - \|W_1(0)\|^2_F - \gamma^{-1}\|b_1(0)\|^2 \\&= \alpha^2 (\|\overline{w_2}\|^2 - \|\overline{W_1}\|^2_F - \gamma^{-1}\|\overline{b_1}\|^2) > 0.
\end{align*}
Let $t_2$ be the first moment that $\|w_2(t_2)\|^2 = \alpha_1$ (note that $\|w_2(t_1)\|^2 \leq \alpha_1^2$ at the beginning of this phase at $t_1$). Note that if the dynamics never reach this point, then both $\|b_1\|, \|w_2\|$ are bounded and there is nothing to prove. For $t_1 \leq t \leq t_2$, we have that $\|W_1\|_F, \|w_2\| \leq \sqrt{\alpha_1}$ and $\|b_1\| \leq \sqrt{\gamma\alpha_1}$. Thus, considering \cref{eq:update_rb}, $r_b$ also remains bounded by $\sqrt{\gamma L(0)}\alpha_1$. Note that (for small enough $\alpha_1$)
\begin{align*}
    \frac{d\|w_2\|^2}{dt} = -2w_2^TW_1^Tr_w-w_2^Tb_1r_b \leq 4\sqrt{L(0)} \|w_2\|^2 \Rightarrow \\\int_{t_1}^{t_2} \|w_2(t)\|^2 dt \geq \frac{\|w_2(t_2)\|^2- \|w_2(t_1)\|^2}{4\sqrt{L(0)}} \geq \frac{\alpha_1 - \alpha_1^2}{4\sqrt{L(0)}}.
\end{align*}
Also for $t_1 \leq t \leq t_2$, $\|w^*\| - \alpha_1 \leq \|r_w\| = \|W_1w_2 - w^*\| \leq \|w^*\| + \alpha_1$.
Using this simple bound and $\|w_2(t)\|^2 \geq \|W_1(t)\|^2_F + \gamma^{-1}\|b_1(t)\|^2$, we have (for small enough $\alpha_1$)
\begin{align*}
    \frac{d\|r_w\|^2}{dt} = 2r_w^T \dot{(W_1w_2)} &= -2r_w^T (W_1W_1^Tr_w + W_1b_1r_b + r_w w_2^Tw_2)\\
    &= -2\|W_1^Tr_w\|^2 - 2\|r_w\|^2\|w_2\|^2  - 2r_w^TW_1b_1r_b \\
    &\leq -\|r_w\|^2\|w_2\|^2.
\end{align*}
Consequently, assuming that $\alpha_1$ is small enough, we have
\begin{align*}
    \|r_w(t_2)\|^2  &\leq \|r_w(t_1)\|^2+ \int_{t_1}^{t_2} -\|r_w\|^2\|w_2\|^2 dt \\
    &\leq \|r_w(t_1)\|^2 + (\|w^*\| - \alpha_1)^2   \int_{t_1}^{t_2} -\|w_2\|^2 dt \\
    &\leq \|r_w(t_1)\|^2-(\|w^*\| - \alpha_1)^2 \frac{\alpha_1 - \alpha_1^2}{4\sqrt{L(0)}}\\
    &\leq (\|w^*\| + \alpha_1^2)^2-(\|w^*\| - \alpha_1)^2 \frac{\alpha_1 - \alpha_1^2}{4\sqrt{L(0)}} \leq (\|w^*\| + \alpha_1^2)^2 -\|w^*\|^2 \frac{\alpha_1}{8\sqrt{L(0)}}, 
\end{align*}
therefore, one can see that there exists a constant $c_2 > 0$ such that $\|r_w(t_2)\| \leq \|w^*\|-c_2\alpha_1$. From this point, we can analyze the learning dynamics concerning $w^*$.
Similar to the first phase, define $\alpha_2 = (\log \log \frac{1}{\alpha_1})^{-1}$ and $t'_2$ to be the first moment such that $\max\{\sqrt{\gamma}, 1\}\max\{|r_b|, \|b_1\|\} \geq \alpha_2$ or $\|r_w\| \leq \alpha_2$. For $t_2 \leq t \leq t_2'$ and small enough $\alpha_1$, we have
\begin{align*}
    \frac{d\|r_w\|^2}{dt} &= -2\|W_1^Tr_w\|^2 - 2\|r_w\|^2\|w_2\|^2  - 2r_w^TW_1b_1r_b \leq -\|r_w\|^2\|w_2\|^2 \Longrightarrow \\
    \frac{d\|r_w\|}{dt} &\leq -\frac{\|r_w\|\|w_2\|^2}{2} \leq -\frac{\|r_w\|\|W_1w_2\|}{2} \leq -\frac{\|r_w\|(\|w^*\|-\|r_w\|)}{2} \Longrightarrow \\
    &\int_{t_2}^{t'_2} (\frac{1}{\|r_w\|} + \frac{1}{\|w^*\| - \|r_w\|})\frac{d\|r_w\|}{dt} dt \leq \int_{t_2}^{t'_2} \frac{-\|w^*\|}{2} dt \Longrightarrow \\
    &\log \frac{\|r_w(t'_2)\|}{\|w^*\| -\|r_w(t'_2)\|} - \log \frac{\|r_w(t_2)\|}{\|w^*\| -\|r_w(t_2)\|} \leq -(t'_2 - t_2)\frac{\|w^*\|}{2} \Longrightarrow
    \\
    & \frac{\alpha_2}{\|w^*\|}<\frac{\|r_w(t'_2)\|}{\|w^*\| -\|r_w(t'_2)\|} \leq  \frac{\|r_w(t_2)\|}{\|w^*\| -\|r_w(t_2)\|} e^{-(t'_2 - t_2)\frac{\|w^*\|}{2}} < \frac{\|w^*\|}{c_2\alpha_1}e^{-(t'_2 - t_2)\frac{\|w^*\|}{2}}.
\end{align*}
From the inequality above, we can conclude that 
\begin{equation*}
    t'_2-t_2 \leq \frac{2}{\|w^*\|}\log(\frac{\|w^*\|^2}{c_2\alpha_1\alpha_2}).
\end{equation*}
Now we bound $r_b$ and $b_1$. Define $m' = \max \{r_b^2, \|b_1\|^2\}$. Given the derivatives for $r_b, b_1$ one can easily see that $\frac{d\|b_1\|^2}{dt} \leq 2\gamma m' \|w_2\|$ and $\frac{dr_b^2}{dt} \leq 2 m' \|w_2\|\sqrt{L(0)}$. Therefore, for $\eta = \max\{2\sqrt{L(0)}, 2\gamma\}$ and $t_2 \leq t \leq t'_2$, we have
\begin{equation}
    m'(t) \leq m(t_2) + \int_{t_2}^t \eta m \|w_2\| dt \Longrightarrow m(t) \leq m'(t_2)\exp{\eta\int_{t_2}^t  \|w_2\| dt} \leq \sqrt{\alpha_1}\exp({\eta\int_{t_2}^t  \|w_2\| dt}),
\end{equation}
where we have used Grönwall's inequality again. Now we bound $\int_{t_2}^{t'_2}  \|w_2\| dt$. Note that 
\begin{align*}
    \frac{d\|r_w\|^2}{dt} \leq -\|r_w\|^2\|w_2\|^2 \Longrightarrow \int_{t_2}^{t'_2}  \|w_2\|^2 dt \leq \frac{1}{\alpha_2^2}(\|r_w\|^2(t_2) -\|r_w\|^2(t'_2)) \leq\frac{\|w^*\|^2}{\alpha_2^2} \Longrightarrow \\
    \int_{t_2}^{t'_2}  \|w_2\| dt \leq \sqrt{(\int_{t_2}^{t'_2}  \|w_2\|^2 dt)(t'_2 - t_2)} \leq \sqrt{\frac{\|w^*\|^2}{\alpha_2^2}(\frac{2}{\|w^*\|}\log(\frac{\|w^*\|^2}{c_2\alpha_1\alpha_2}))},
\end{align*}
where we used Cauchy's inequality in the last line. Therefore,
\begin{equation*}
    m'(t) \leq \sqrt{\alpha_1}\exp({\eta \alpha_2^{-1} \sqrt{2\|w^*\|\log(\frac{\|w^*\|^2}{c_2\alpha_1\alpha_2}))}}).
\end{equation*}
Recall that we had $\alpha_2 = (\log \log \frac{1}{\alpha_1})^{-1}$. Therefore, we can  see that $m'(t) \ll \alpha_2^2$ (note that $\sqrt{\alpha_1}\exp( \sqrt {\log \frac{1}{\alpha_1}})=o(\alpha_1^{0.5 + \epsilon})$ for any $\epsilon > 0$). Therefore, by $t'_2$ values of $|r_b|, \|b_1\|$ have not grown to $\alpha_2$. So it must be $\|r_w\|$ which has become less than or equal to $\alpha_2$. Thus, we proved the claim of the second phase that we reach a point such that $\|r_w\|, |r_b|, \|b_1\| \leq \alpha_2$ for a vanishing $\alpha_2 = o_\alpha(1)$. Now, we are ready to move to the last phase of training. 

\textit{Phase 3.} Considering that $\|r_w\|, |r_b|, \|b_1\| \leq \alpha_2 = o_\alpha(1)$ (i.e., both the weight and the bias of the target is almost learned and bias of the first layer is still small), we show that the parameters would not move significantly after this point. 

First, we show that the loss function is having an exponential decay at this point. Note that 
\begin{align*}
    \frac{dL}{dt} = \nabla_\theta L^T \frac{d\theta}{dt} \leq -\|\dot W_1\|_F^2 - (\dot b_2)^2 = -\|r_w\|^2\|w_2\|^2 - r_b^2.
\end{align*}
Now note that $L(t'_2) \leq 2\alpha_2^2$. We also know that the loss is decreasing, thus for small enough $\alpha_2$ we have
\begin{align}
    \|w_2\|^2 \geq \|W_1w_2\| \geq \|w^*\| - \|r_w\| \geq \|w^*\| - \sqrt{2}\alpha_2 \geq \frac 1 2 \|w^*\|.
\end{align}
As a result, we can conclude that for $c_3 = \min\{\frac 1 2 \|w^*\|, 1\} > 0$, we have
\begin{align*}
    \frac{dL}{dt} \leq -\|r_w\|^2\|w_2\|^2 - r_b^2 \leq -c_3(r_b^2 + \|r_w\|^2) = -c_3L(t) \Longrightarrow L(t) \leq L(t'_2)e^{-c_3(t-t'_2)},
\end{align*}
where we used Grönwall's inequality again. Consequently, we can see that $\|r_w\|, |r_b|$ have exponential decay as well, i.e., 
\begin{align*}
    \|r_w\|, |r_b| \leq \sqrt{L(t)} \leq \sqrt{L(t'_2)e^{-c_3(t-t'_2)}} \leq \sqrt{2}\alpha_2 e^{-0.5c_3(t-t'_2)}.
\end{align*}
We also need to provide an upper bound for $\|w_2(t'_2)\|$. Note that we have proved that 
$w_2w_2^T - W_1^TW_1 - \gamma^{-1}b_1b_1^T = O(\alpha^2)$ is constant during training. Thus at $t'_2$, we have
\begin{align*}
    w_2^T(w_2w_2^T - W_1^TW_1 - \gamma^{-1}b_1b_1^T)w_2 = O(\alpha^2)\|w_2\|^2 \Longrightarrow \\
    \|w_2\|^4 - \|W_1w_2\|^2 - \gamma^{-1}(b_1^Tw_2)^2 =  O(\alpha^2)\|w_2\|^2
\end{align*}
Note that we know $\|r_w\| = \|W_1w_2 - w^*\| \leq \alpha_2$ and hence $W_1w_2 = \theta(1)$. Also $(b_1^Tw_2)^2 \leq \alpha_2^2 \|w_2\|^2$. Therefore, using the above equation, one can see that $\|w_2(t'_2)\| = O(1)$ (as a result $\|W_1(t'_2)\|_F = O(1)$ as well), i.e., there exists $c_4 > 0$ (not depending on $\alpha, \alpha_1, \alpha_2$) such that for small enough $\alpha_2$ $\|W_1\|_F, \|w_2\| \leq c_4$. (Note that during the entire training $\|r_w\|$ and thus $\|W_1w_2\|$ is bounded. So we have $\|W_1\|_F, \|w_2\| = O_\alpha(1)$ during the first two phases as well.) To conclude, assume the contrary of phase 3, i.e., define $t_3$ as the first moment that at least one of $\|W_1\|_F \geq 2c_4$, $\|w_2\| \geq 2c_4$, or $\|b_1\| \geq \sqrt{\alpha_2}$ happens. For $t'_2 \leq t \leq t_3$ we can bound the change in all of the variables as follows (assuming that $\alpha_2$ is small enough)
\begin{align*}
    \|\dot w_2\| &< 3c_4\sqrt{L(t)} \Rightarrow \|w_2(t_3)\| < \|w_2(t'_2)\| + 3c_4 \int_{t'_2}^{t_3} \sqrt{L(t)} dt < c_4 + \frac{6\sqrt{2}c_4}{c_3}\alpha_2 < 2c_4,\\
    \|\dot W_1\|_F & \leq 2c_4\sqrt{L(t)} \Rightarrow \|W_1(t_3)\|_F \leq  \|W_1(t'_2)\|_F + 2c_4 \int_{t'_2}^{t_3} \sqrt{L(t)} dt < c_4 + \frac{4\sqrt{2}c_4}{c_3}\alpha_2 < 2c_4,\\
    \|\dot b_1\| &\leq  2\gamma c_4\sqrt{L(t)} \Rightarrow \|b_1(t_3)\| < \|b_1(t'_2)\| + 2\gamma c_4 \int_{t'_2}^{t_3} \sqrt{L(t)} dt < \alpha_2 + \frac{4\gamma\sqrt{2}c_4}{c_3}\alpha_2 < \sqrt{\alpha_2}\\
\end{align*}
where we used the exponential decay of $\sqrt{L(t)}$ to derive the inequalities above. Note that the above inequalities show that none of $\|W_1\|_F \geq 2c_4$, $\|2_2\| \geq 2c_4$, $\|b_1\| \geq \sqrt{\alpha_2}$ can happen, proving our claim that the parameters would not move significantly. Particularly note that $\|b_1\| < \sqrt{\alpha_2}$ and $|w_2^Tb_1| \leq 2c_4\sqrt{\alpha_2}$ where $\alpha_2 = o_{\alpha}(1)$ vanishes as $\alpha \to 0$, showing that the proof of the theorem is now complete. 
\end{proof}
Now we can also provide the proof for Remark \ref{remark:equivalence-linear}.
\begin{proof}\textit{(Remark \ref{remark:equivalence-linear})}
Assume that we want to learn $f(x) = \hat f(\emptyset) + \sum_{i=1}^d \hat f(\{i\})x_i$ and suppose that the $k$-th bit is frozen. We know the other bits have a uniform and independent distribution. Therefore, we can use Parseval's identity to write the training loss as 
\begin{align}
    2L=\mathbb{E}_x[(f_{\mathrm{NN}}(x) - f(x))^2] = &\|\underline{W_1}W_2\dots W_{L-1}w_L - \underline{w^*}\|^2 + (b_L + w_L^Tb_{L-1} + w_L^TW_{L-1}^Tb_{L-2} \nonumber\\&+ \cdots + w_L^T \cdots W_2^T b_1 + w_L^T \cdots W_2^T W_{1,k} - (\hat f(\emptyset) + \hat f(\{k\})))^2 \label{eq:equivariance}
\end{align}
where the expectation on the LHS is uniform over all $x$ with frozen coordinate $k$. Also, $W_{1,k}$ represents the first layer's weights incident to the frozen coordinate and $\underline{W_1}$ represents $W_1$ where $W_{1,k}$ is removed from the matrix. Similarly, $\underline{w^*}$  represents $w^*$ when the $k$-th coordinate is removed. Furthermore, define 
$$r_b \coloneqq b_L + w_L^Tb_{L-1}  + \cdots + w_L^T \cdots W_2^T b_1 + w_L^T \cdots W_2^T W_{1,k} - (\hat f(\emptyset) + \hat f(\{k\})).$$
One can easily check that 
\begin{align*}
    \dot W_{1,k}= -\nabla_{W_{1,k}} L = -r_bW_2W_3\cdots W_{L-1}w_{L} = -\nabla_{b_1} L = \dot b_1.
\end{align*}
In other words, $W_{1,k}$ and $b_1$ have the same gradient, and thus their difference stays the same over time. Now we define a new bias variable as $\tilde b_1 = W_{1,k} + b_1$. We know that 
$$\frac{d\tilde b_1}{dt} = \dot W_{1,k} + \dot b_1 = -2r_bW_2W_3\cdots W_{L-1}w_{L}.$$
Now by revisiting \eqref{eq:equivariance}, we see that this is the loss function of the neural network resulting from removing the frozen coordinate and combining the bias variables as in Conjecture \ref{conj:linear}. The only difference is that 
$$\frac{d\tilde b_1}{dt} = -2r_bW_2W_3\cdots W_{L-1}w_{L} = 2\nabla_{\tilde b_1} L,$$
i.e., the update speed of this newly defined biased variable is $\gamma=2$. Therefore, if the conjecture is satisfied then we know that this bias $\|\tilde b_1\|$ and its contribution $\|w_L^T\cdots W_2^T\tilde b_1\|$ are bounded by $\epsilon$. Here, we can deduce the same about $W_{1,k}$ as 
\begin{align*}
    W_{1,k}(t) &= 0.5(W_{1,k}(t) + b_1(t) +W_{1,k}(t) - b_1(t)) = 0.5(\tilde b_1(t) +W_{1,k}(0) - b_1(0)) \\&= 0.5\tilde b_1(t) + O(\alpha) \leq 0.5\epsilon + O(\alpha).
\end{align*}
Therefore, $\hat f_{\mathrm{NN}}(\{k\}) = w_L^T\cdots W_2^TW_{1,k}$ is also $O(\epsilon + \alpha)$ proving the remark.

\end{proof}
\subsection{Proof for the Length Generalization Theorem}
\label{app:proof-for-length-gen}
\begin{proof}
First, we prove the existence and uniqueness of such low-degree interpolators. Afterward, we consider it explicitly for parity functions.

Note that we know there are no $r+1$ bits which are all equal to $-1$ in $B_r$. 
Therefore, for any $r+1$ indices, we have $(x_{i_1} - 1)\cdots(x_{i_{r+1}} - 1) = 0$. Consequently, each $x_{i_1} \cdots x_{i_{r+1}}$ can be replaced by a degree $r$ polynomial. Now consider the Fourier-Walsh expansion of $f(x)$. By applying the previous identity, one can replace all monomials in the Fourier-Walsh expansion of $f(x)$ with degree $r$ (or less) alternatives, while the value of the function on $B_r$ does not change.  

Now we prove the uniqueness. Consider all monomials $\chi_T(x)$ where $|T| \leq r$. There are in total $\binom{d}{0} + \binom{d}{1} + \cdots + \binom{d}{r} = |B_r|$ of such monomials and consider all functions given by these monomials $f_a(x) = \sum_{i = 1}^{|B_r|} a_i \chi_{T_i}(x)$. Note that for each $x_j \in B_r \; 1 \leq j \leq |B_r|$, $f_a(x_j) = \sum_{i = 1}^{|B_r|} a_i \chi_{T_i}(x_j)$. In other words, $f_a(x_j)$ is a linear combination of $a_i$'s, i.e., $(f_a(x_1), \ldots, f_(x_{|B_r|}))^T = M(a_1, \ldots, a_{|B_r|})^T$, where $M_{i,j} = \chi_{T_j}(x_i)$. Now note that we have proven that any function can be written in this way, i.e., $\mathrm{rank}(M) = |B_r|$ showing that $\dim(\ker(M)) = 0$ and hence the uniqueness.

Now, we particularly study the case of monomials. Without loss of generality, consider degree $k > r$ monomial $\mathrm{parity}_k(x) \coloneqq x_1x_2\cdots x_k$. We claim that 
\begin{align*}
    f_r(x) & \coloneqq 1 + \sum_{1 \leq i \leq k} (x_i-1) + \sum_{1\leq i<j\leq k} (x_i-1)(x_j-1) + \cdots + \sum_{i_1 <\cdots <i_r \leq k} (x_{i_1}-1)\cdots(x_{i_r}-1) \nonumber\\&= 1 + \sum_{T \subseteq [k]: |T|=1} \prod_{i \in T} (x_i -1)+ \cdots + \sum_{T \subseteq [k]: |T|=r} \prod_{i \in T} (x_i-1)
\end{align*}
is the the unique low-degree equivalent of $\mathrm{parity}_k$ on $B_r$, i.e., $\mathrm{parity}_k(x) = f_r(x) \;\;\forall x \in B_r$. To see this, take any $x \in B_r$. Define $s(x)$ as the number of $-1$ bits in $x_1, \cdots, x_k$, i.e., $s(x) \coloneqq |\{x_i = -1 | 1 \leq i \leq k\}|$. Note that $0 \leq s(x) \leq k$ and $\mathrm{parity}_k(x) = (-1)^{s(x)}$. Furthermore, we have 
\begin{equation*}
    \forall 1\leq i\leq r\;\;\; \sum_{T \subseteq [k]: |T|=i} \prod_{j \in T} (x_j-1) = (-2)^i\binom{s(x)}{i}.
\end{equation*}
Therefore, 
\begin{align*}
    f_r(x) &= 1 + \sum_{T \subseteq [k]: |T|=1} \prod_{i \in T} (x_i -1)+ \cdots + \sum_{T \subseteq [k]: |T|=r} \prod_{i \in T} (x_i-1) \\&= 1 + (-2)^1\binom{s(x)}{1} + \cdots + (-2)^i\binom{s(x)}{i} + \cdots + (-2)^r\binom{s(x)}{r} \\&= (1-2)^{s(x)} = (-1)^{s(x)} = \mathrm{parity}_k(x),
\end{align*}
where we used the fact that $s(x) \leq r$.
Now we consider the constant term (i.e., bias) of $f_r(x)$. Indeed notice that the constant in $f_r(x)$ is given by
\begin{equation*}
    \hat{f_r}(\emptyset) =  1- \binom{k}{1} + \binom{k}{2} - \cdots + (-1)^{r}\binom{k}{r}. 
\end{equation*}
It can easily be proven that the above constant is equal to $(-1)^r\binom{k-1}{r}$ by induction on $r$. Note that it is clear for $r=1$ and the induction step from $r$ to $r+1$ is given by  
\begin{align*}
    1 - \binom{k}{1} + \cdots + (-1)^{r}\binom{k}{r} &+ (-1)^{r+1}\binom{k}{r+1} = (-1)^r\binom{k-1}{r} +(-1)^{r+1}\binom{k}{r+1} \\&= (-1)^{r+1}(\binom{k}{r+1} - \binom{k-1}{r}) = (-1)^{r+1}\binom{k-1}{r+1}.
\end{align*}
Therefore, by Parseval's identity we have 
\begin{equation*}
    \E_x[(\mathrm{parity}_{k}(x) - f_r(x))^2] > \hat{f_r}(\emptyset)^2 = \binom{k-1}{r+1}^2,
\end{equation*}
which proves the lower bound. Note that we ignored other Fourier-Walsh coefficients for the lower bound above. 

\end{proof}
\section{Experiment Details and Additional Experiments}\label{app:exps}
In this section, we provide details on the architectures and experimenting procedure as well as presenting additional results.
\subsection{Experiment Details}
 
\subsubsection{Architectures}
We use MLP, Transformer \citep{vaswani2017attention-transformer}, mean-field \citep{mei2018mean} and random features model (Definition~\ref{def:random-features-model}) for experiments.  Here, we describe them in detail. 
\begin{itemize}
    \item \textbf{MLP.} The MLP model is a fully connected network consisting of 4 hidden layers of sizes $512, 1024, 512, 64$. The ReLU activation function is used for all layers except the output layer. Moreover, the standard initialization of PyTorch has been followed, i.e., the weights of each layer are initialized with $U(\frac{-1}{\sqrt{\mathrm{dim_\mathrm{in}}}}, \frac{1}{\sqrt{\mathrm{dim_\mathrm{in}}}})$ where $\mathrm{dim_\mathrm{in}}$ is the input dimension of the layer.
    \item \textbf{Transformer.} We have employed the encoder part of Transformer networks which are widely used in computer vision \citep{dosovitskiy2020image} and language modeling \citep{raffel2019exploring}. First, all binary $\pm 1$ bits are encoded  into a 256-dimensional vector using a shared embedding layer. Afterward, the embedded input is passed through 12 transformer layers. Finally, a linear layer is used to compute the output of the model. Moreover, the size of MLP hidden layers is set to 256, and 6 heads are used for the self-attention blocks. \\
    Also, note that we use bidirectional attention with learnable positional embeddings for our main experiments. Only for \cref{sec:causal-transformer} we use causal attention masking. 
    \item \textbf{Mean-field.}
    We also use a two-layer neural network in the mean-field parametrization. More precisely, following \citet{mergedstaircase},  $f_{\mathrm{MF}}(x) = \frac{1}{N} \sum_{i=1}^N a_i\sigma(\langle w_i, x \rangle + b_i)$, where $a_i \sim U(-1,1)$ and $w_i, b_i \sim U(\frac{-1}{\sqrt{d}}, \frac{1}{\sqrt{d}})^{\otimes d} \otimes U(\frac{-1}{\sqrt{d}}, \frac{1}{\sqrt{d}})$. We use ReLU as the activation function and set the number of neurons to $N=2^{16}$. Note that with this formulation, one has to take large values for the learning rate, e.g., $100$ or $1000$.

    \item \textbf{Random features model.}
    Following Definition~\ref{def:random-features-model}, we use $f_{\mathrm{RF}} = \sum_{i=1}^N a_i\sigma(\langle w_i, x \rangle + b_i)$ as the parametrization of the RF model. Moreover, we initialize $a_i = 0$ and $w_i, b_i \sim \mathcal{N}(0, \frac{1}{d})^{\otimes d} \otimes \mathcal{N}(0, \frac{1}{d})$ where $d$ is the input dimension. We also use $N= 2^{13}$ random features for our experiments. We have used the ReLU activation function for almost all of the experiments. We have only used polynomial activation $(1+x)^6$ for the experiment comparing RF models with the ReLU activation and polynomial activation (\cref{fig:rf-comparison}).  
\end{itemize}
\subsubsection{Procedure}
The implementation of experiments has been done using the PyTorch framework \citep{torch}. Additionally, the experiments were executed on NVIDIA A100 GPUs and the experiments took around 90 GPU hours in total (excluding the selection of hyperparameters). Now we discuss the training procedures.

\noindent\textbf{Length generalization and main experiments.}
We first explain the experiments of the main experiment section and also experiments for the length generalization. For each function $f\colon\{\pm 1\}^d \to \mR$ and unseen domain $\mathcal{U}$, we generate all binary vectors in  $\mathcal{U}^c = \{\pm 1\}^d \setminus \mathcal{U}$ for the training set. Consequently, we usually take small values of $d$ for the experiments. Our main motivation for doing so is to eliminate the randomness generated by the sampling of training examples and also to assume the in-distribution generalization. Nonetheless, we believe the min-degree bias still holds when training examples are sampled randomly as is illustrated in the experiments included in this appendix.

We then train our models. For the Transformer, we have used Adam \citep{kingma2014adam} optimizer with batch size $256$. For the RF models, we have used mini-batch SGD with a batch size of $256$. Also, for the rest of the architectures, SGD with batch size $64$ has been used. We did not observe any significant difference in the results of experiments by varying the batch sizes. We generally selected the learning rates per model (and task) by the stability of the training and the speed of convergence. We have included more details about the learning rate in \cref{app:learning-rate-sensitivity}. We also set the number of training epochs large enough that the loss of models is always less than $10^{-2}$. We also note that Transformers usually learn the target function in a few epochs, reaching a loss of the order of $10^{-4}$. After that, the training becomes unstable in some instances. Indeed note that Transformers are usually trained with learning rate schedulers. However, we did not use any learning rate schedulers for simplicity and instead opted for early stopping to avoid unstable phases of training. Also for the results reported for causal attention masking in \cref{tab:decoder} and particularly for instance $x_7x_{13}$ with $x_7=1$ some of the seeds became unstable and did not converge. So for this particular instance, we reported the results for the first $10$ seeds where the training loss converged. We note that even for the unstable seeds the min-degree bias was never violated. Note that our main objective is to demonstrate the min-degree bias of neural networks and not to optimize any performance metric. As a result, we did not focus on hyperparameter tuning in these experiments. Generally, hyperparameters used for our experiments are available in our code online: \url{https://github.com/aryol/GOTU}. 

Finally, we track the coefficient of different monomials, i.e., $\hat{f}_{\mathrm{NN}}(T) = \E_x[\chi_T(x){f}_{\mathrm{NN}}(x)]$ during the training. We have also repeated each experiment for $10$ different seeds and reported the averages. Particularly, we have also drawn $95\%$-CI in Figures \ref{fig:curriculum}, \ref{fig:learningrate-sensitivity}, and \ref{fig:low-dimension} but we did not draw CI for other experiments to keep the plots more readable.

\noindent\textbf{Curriculum learning experiments.}
In contrast to other experiments, there is no unseen domain in these experiments. Also here we draw a fixed number of samples uniformly from $\{\pm 1\}^d$. We train the MLP model with the same training set, learning rate, and batch size, once with normal mini-batch SGD and once with Degree-Curriculum (Algorithm \ref{alg:degree-curriculum}). Therefore, everything between the Degree-Curriculum algorithm and the normal training is the same. We use Adam optimizer for these experiments as we found it to be faster than plain SGD. Moreover, we selected the learning rate based on the results of the normal training and then used the same learning rate for the Degree-Curriculum algorithm to have a fair comparison. Finally, we compare the average generalization loss between the two algorithms.

\subsection{Sensitivity to Learning Rate}\label{app:learning-rate-sensitivity}
We noticed that the min-degree bias of some architectures such as MLPs depends on the learning rate. More precisely, we observed that smaller learning rates promote the min-degree bias and larger learning rates increase the leakage of the models. Here, we demonstrate the effect of the learning rate with an example. Consider learning $f_2(x_0, \ldots, x_{14}) = x_0x_1$ under unseen domain $\mathcal{U}_2=\{(x_0, x_1) = (-1, -1)\}$. In this case, the min-degree interpolator is $x_0+x_1-1$. Nonetheless, any $\alpha_{\mathrm{Leak}}(x_0x_1) + (1-\alpha_{\mathrm{Leak}})(x_0+x_1-1)$ is also a valid interpolator where $\alpha_{\mathrm{Leak}}$ shows the leakage of the interpolator. We tried learning $f_2$ under $\mathcal{U}_2$ with an MLP and varied the learning rate; the results are depicted in \cref{fig:learningrate-sensitivity}. 
\begin{figure}[tb]
    \centering
    \includegraphics[width=0.55\textwidth]{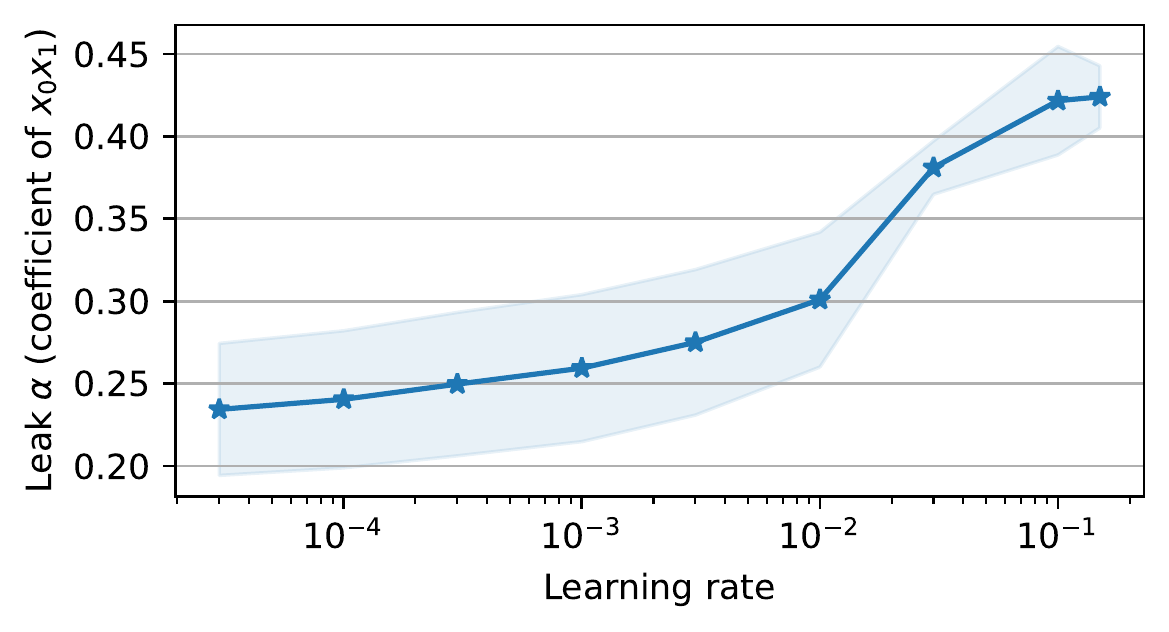}
    \caption{Leakage of the interpolators learned by the MLP model trained with different learning rates. Larger learning rates weaken the min-degree bias and lead to higher leaks.}
    \label{fig:learningrate-sensitivity}
\end{figure}
It can be seen that larger learning rates cause higher leaks in the models. We note that training becomes more unstable with larger learning rates to the point that the model cannot be trained with learning rates larger than $0.2$. Also notice that $\alpha < 0.5$ in all cases, hence, the min-degree alternatives are still dominant. In general, in our experiments, we tried to select moderate values for the learning rate to ensure that the optimization process is stable. Nonetheless, we never used learning rate below $10^{-5}$ for Adam and we usually set learning rate between $10^{-4}$ to $10^{-3}$ for SGD. Exact hyperparameters for different experiments are available in our code.

\subsection{Additional Experiments}
Here, we complete the experiments presented in \cref{sec:exps} and also provide an experiment in support of Conjecture \ref{conj:linear}. 

First, we try learning example $(f_2, \mathcal{U}_2)$ of Section \ref{sec:exps} in a larger ambient dimension.
More specifically, we use ambient dimension $d=40$ and consider learning $f_2(x_0, x_1, \ldots, x_{39}) = x_0x_1$ under unseen domain $\{(x_0, x_1) = (-1, -1)\}$. In this case, the MD interpolator is again $x_0+x_1-1$. For this experiment, we can not generate all $2^{39}$ elements of $\mathcal{U}^c$, thus, we only use $2^{15}$ samples uniformly drawn from $\mathcal{U}^c$. We also use the same number of samples for the estimation of Fourier-Walsh coefficients. The results are depicted in \cref{fig:f2-dim40}. For the random features model, it can be seen that the leakage is reduced compared to \cref{fig:transformers} where the ambient dimension is $15$. On the other hand, the leakage of other models has remained the same, which shows that the sparsity and ambient dimension do not affect them. This is indeed consistent with our expectations as we know models such as the mean-field are able to perform feature-learning and learn the support of sparse Boolean functions \citep{mergedstaircase}. 

\begin{figure}[tb]
     \centering
     \begin{subfigure}[b]{0.4\textwidth}
         \centering
         \includegraphics[width=\textwidth]{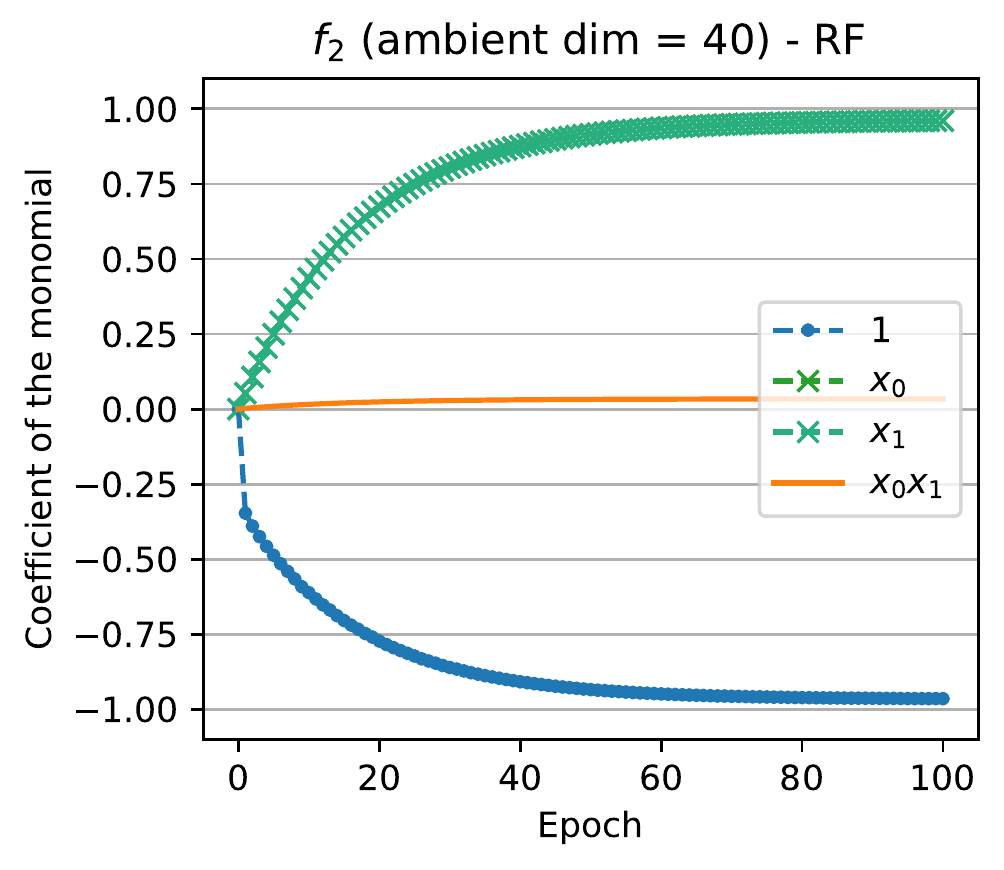}
     \end{subfigure}
     \hfill
     \begin{subfigure}[b]{0.4\textwidth}
         \centering
         \includegraphics[width=\textwidth]{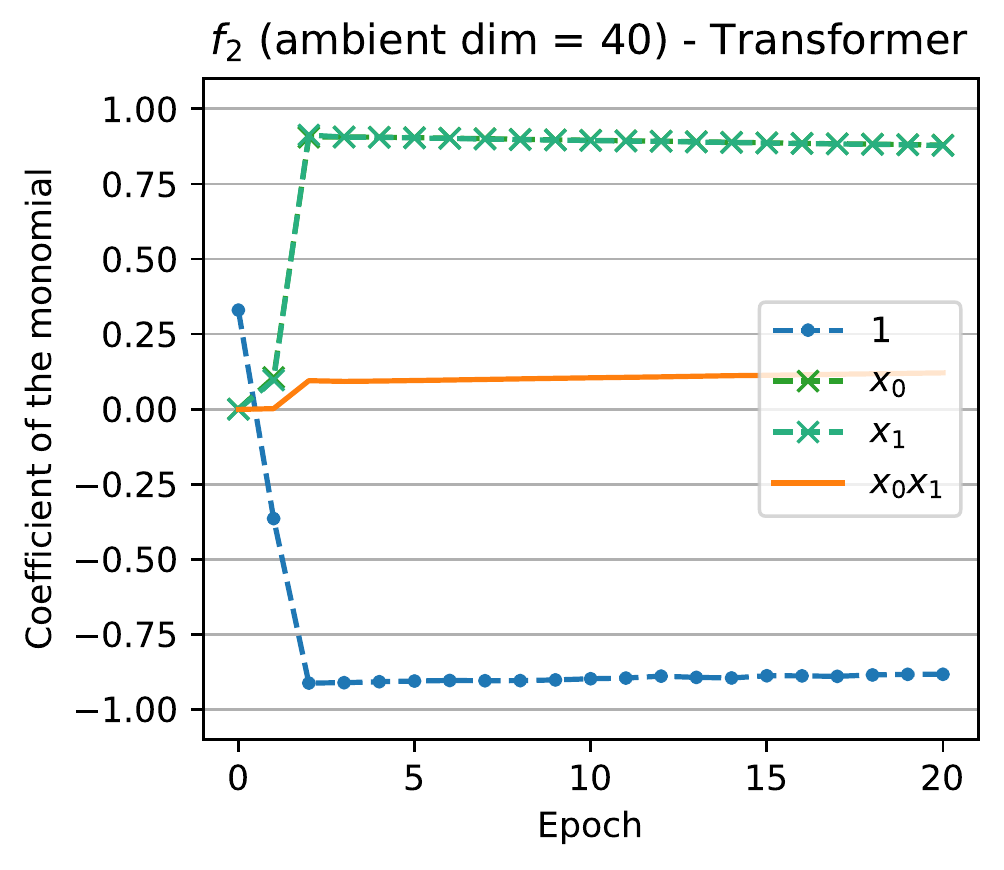}
     \end{subfigure}
     \hfill
     \begin{subfigure}[b]{0.4\textwidth}
         \centering
         \includegraphics[width=\textwidth]{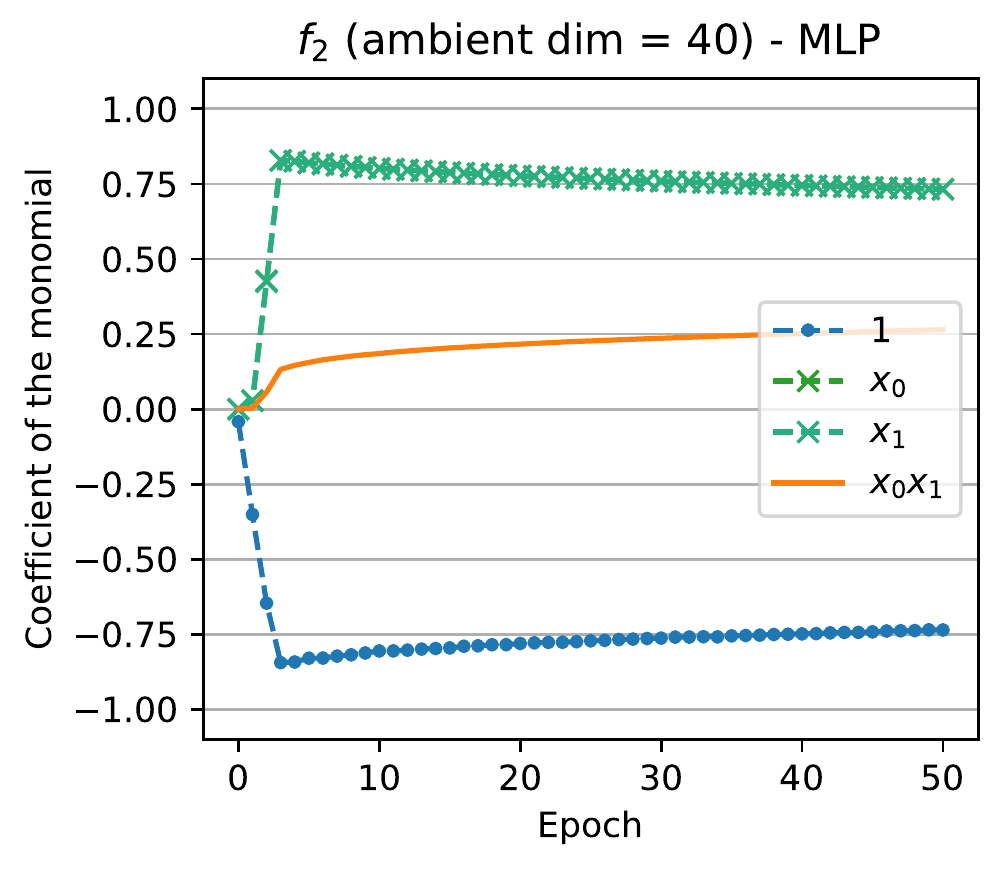}
     \end{subfigure}
     \hfill
     \begin{subfigure}[b]{0.4\textwidth}
         \centering
         \includegraphics[width=\textwidth]{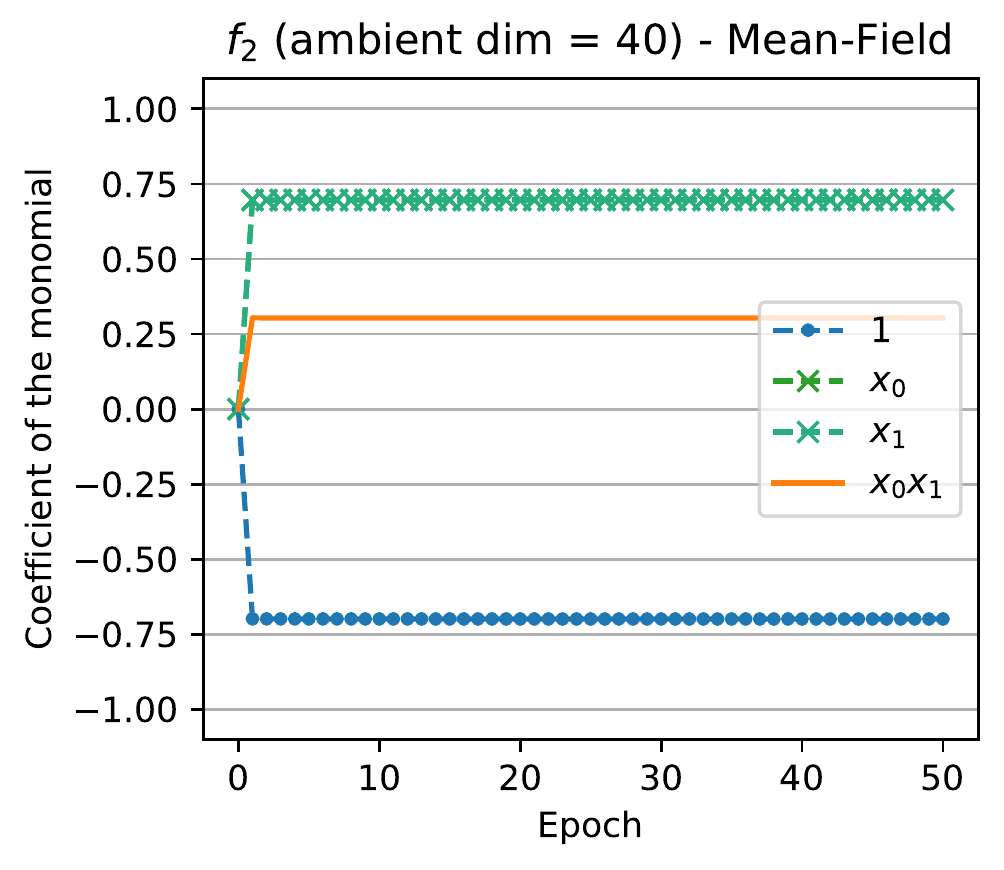}
     \end{subfigure}
     
        \caption{$f_2(x_0, x_1, \ldots, x_{39}) = x_0x_1$ learned by the RF, Transformer, MLP, and mean-field models while training samples satisfy $(x_0, x_1) \neq (-1, -1)$. Consequently, $x_0x_1$ (solid orange line) is replaceable by $x_0 + x_1 - 1$ (dashed lines). The Transformer, MLP, and mean-field models learn leaky solutions and the leakage is very similar to Figures \ref{fig:transformers} and \ref{fig:other-archs} where the ambient dimension is $15$. In contrast, the leak of the RF model is decreased in comparison to \cref{fig:transformers}.}
        \label{fig:f2-dim40}
\end{figure}

Further, we consider the majority function on 3 bits embedded in a 40-dimensional ambient space, i.e., $f_4(x_0, x_1, \ldots, x_{39}) = \mathrm{Maj}(x_0, x_1, x_2) = \frac{1}{2}(x_0 + x_1 + x_2 -x_0x_1x_2)$ under the unseen domain $\mathcal{U}_4 = \{x \in \{\pm 1\}^d|(x_0, x_1) = (-1, -1)\}$. Note that in this case $x_0x_1x_2$ can be replaced with $x_0x_2+ x_1x_2 - x_2$ which leads to MD interpolator being equal to $\frac{1}{2}(x_0 + x_1 + 2x_2 - x_0x_2 - x_1x_2)$. Similar to the previous experiments, we trained the RF, MLP, mean-field, and Transformer on this instance. For this example, we do not generate the whole $\mathcal{U}^{c}$, and instead, we use $2^{15}$ samples. This number is still large enough that gives the generalization on the seen domain. The results of this experiment are presented in \cref{fig:f4-all}. Note that in this case, the original target function is more symmetric than the MD interpolator. Nonetheless, none of the models are able to recover the more symmetric function. 

 \begin{figure}[htb]
     \centering
     \begin{subfigure}[b]{0.4\textwidth}
         \centering
         \includegraphics[width=\textwidth]{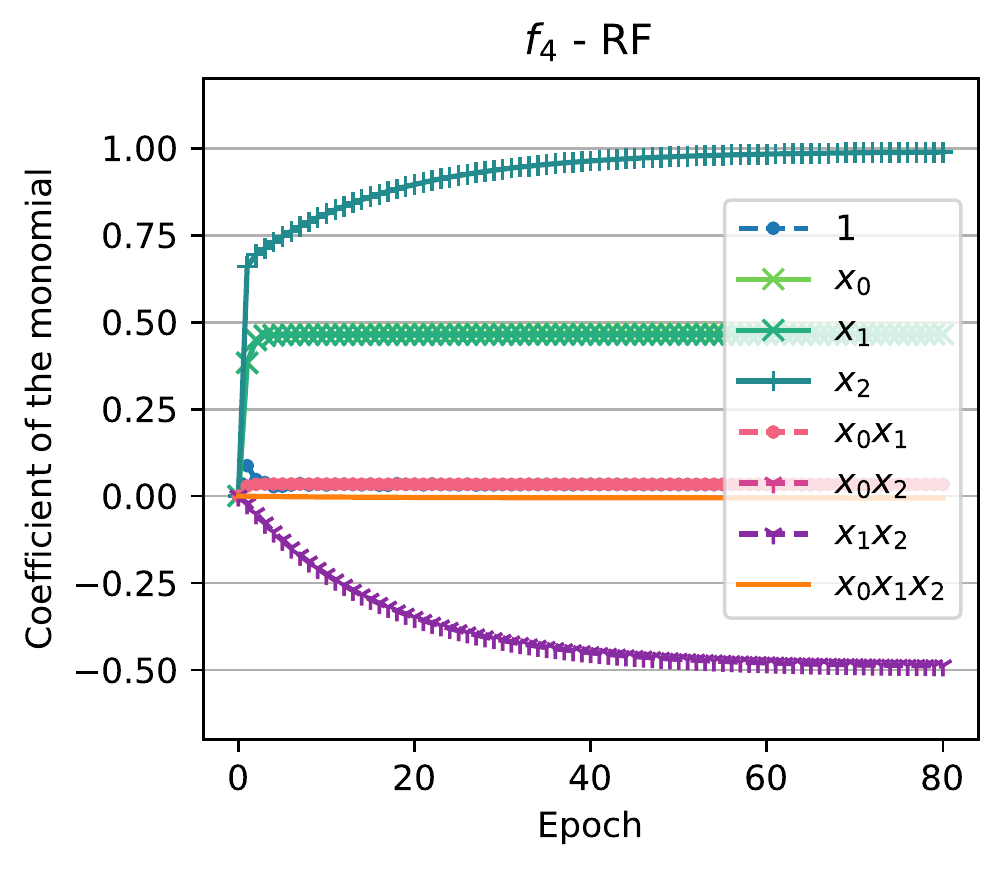}
     \end{subfigure}
     \hfill
     \begin{subfigure}[b]{0.4\textwidth}
         \centering
         \includegraphics[width=\textwidth]{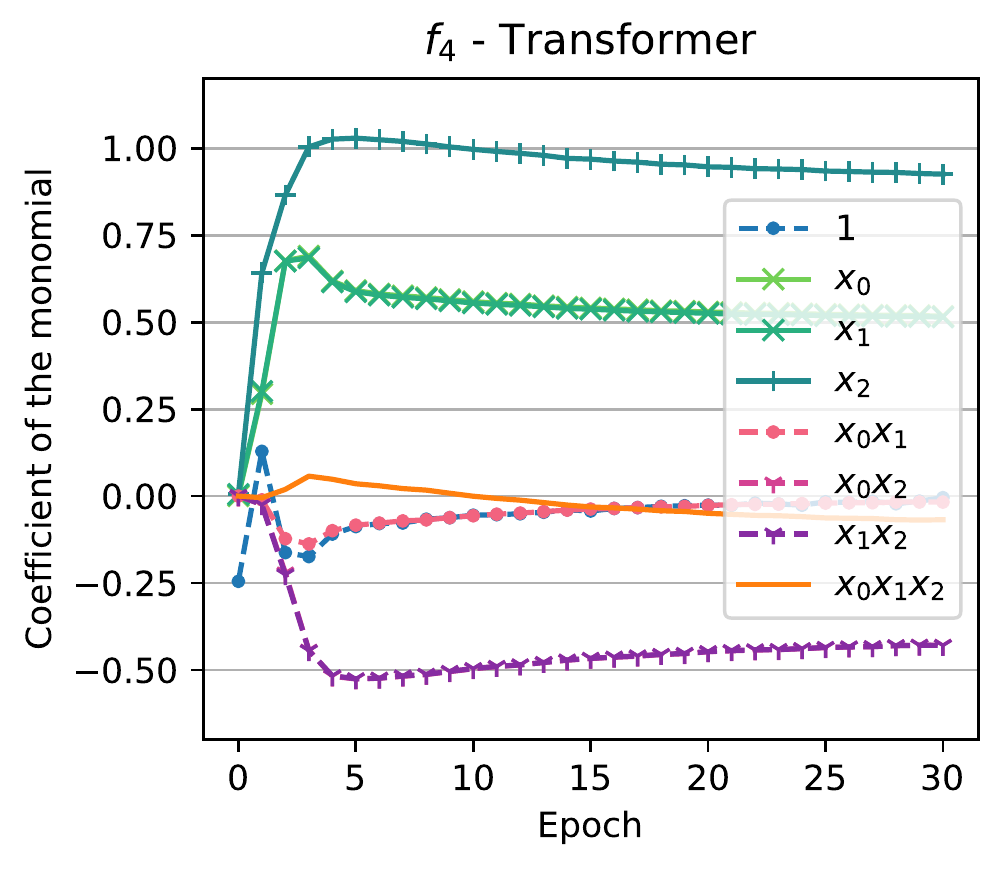}
     \end{subfigure}
     \hfill
     \begin{subfigure}[b]{0.4\textwidth}
         \centering
         \includegraphics[width=\textwidth]{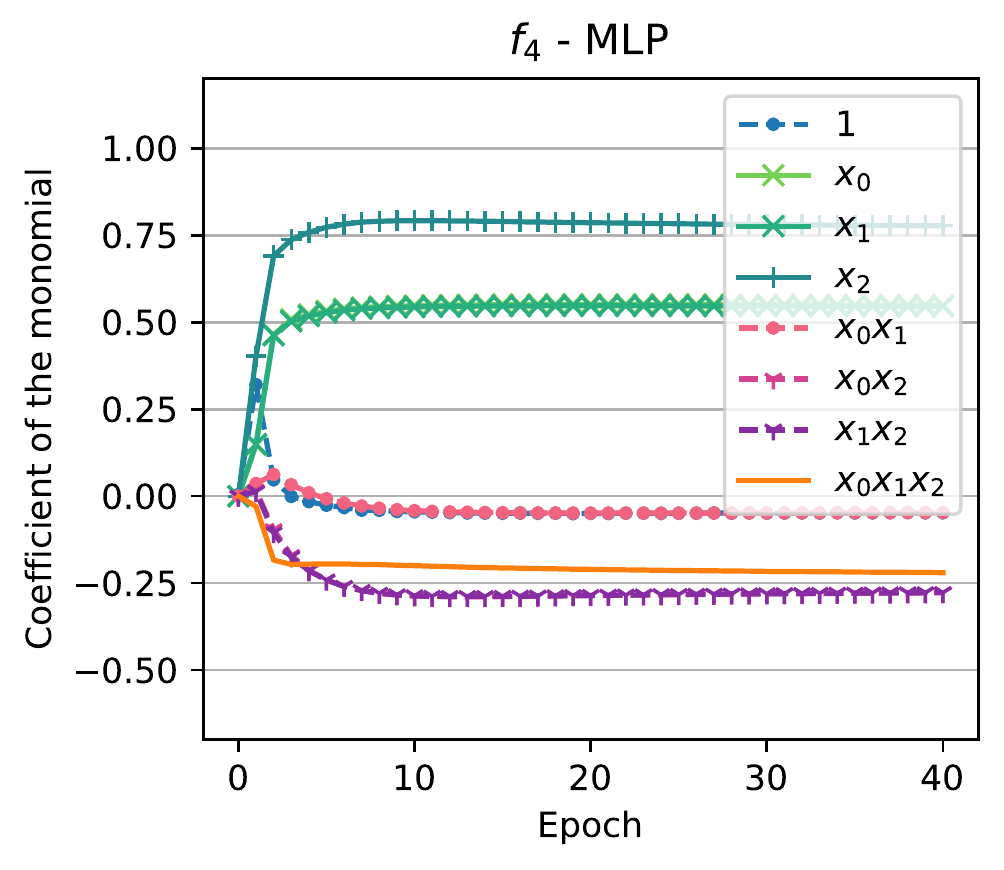}
     \end{subfigure}
     \hfill
     \begin{subfigure}[b]{0.4\textwidth}
         \centering
         \includegraphics[width=\textwidth]{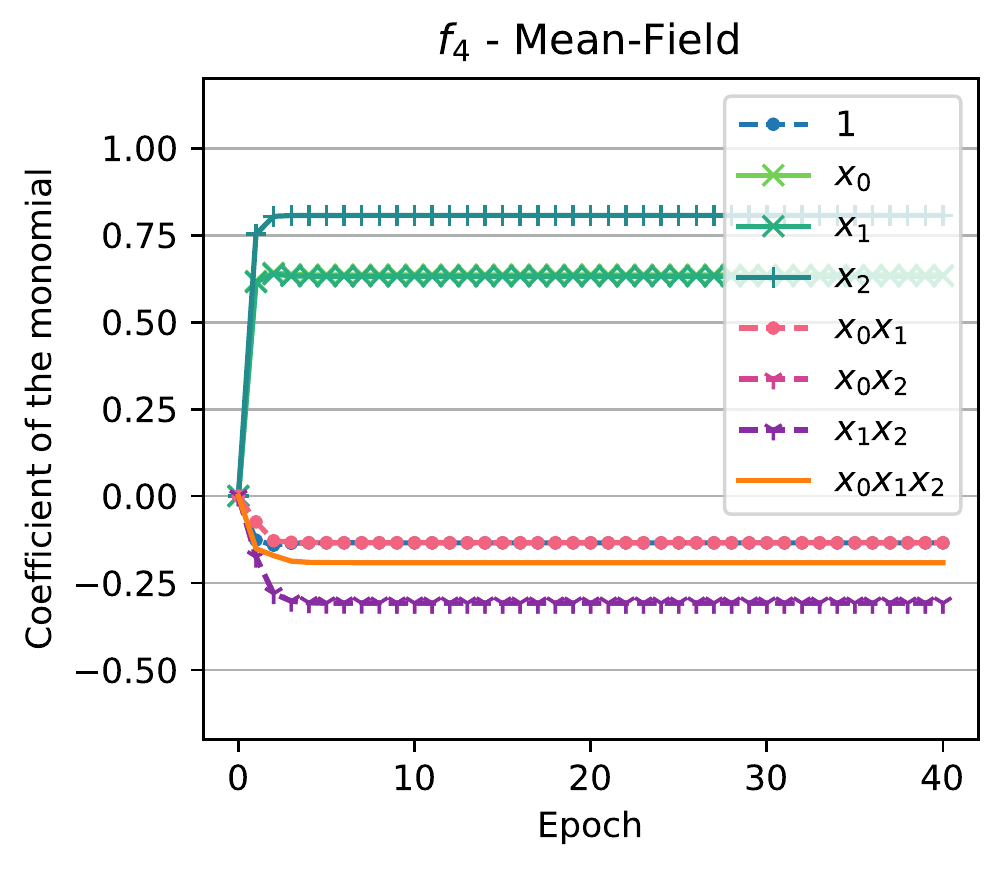}
     \end{subfigure}
     
        \caption{$f_4(x_0, x_1, \ldots, x_{39}) = \mathrm{Maj}(x_0, x_1, x_2) = \frac{1}{2}(x_0 + x_1 + x_2 -x_0x_1x_2)$ learned by the RF, Transformer, MLP, and mean-field models while samples satisfying $(x_0,x_1) = (-1, -1)$ are excluded from training. In this case, $x_0x_1x_2$ (orange solid line) is replaceable by $x_0x_2+x_1x_2-x_2$. As expected, the RF learns the MD interpolator. The Transformer also learns the MD interpolator with a small leakage. On the other hand, the MLP and mean-field models have a more considerable leakage.}
        \label{fig:f4-all}
\end{figure}
Finally, we present an experiment on linear neural networks in support of Conjecture \ref{conj:linear}. Particularly, consider learning linear function 
$$f(x_0, x_1, \ldots, x_{12})= 1 + x_0 + 1.125x_1+1.25x_2 + 1.375x_3 + \cdots + 2.375x_{11} + 2.5x_{12}$$
with a 4-layer fully connected linear neural network such that the width of each layer is $256$. We first initialize each layer with PyTorch's default initialization (i.e., each layer's weights are initialized with $U(\frac{-1}{\sqrt{\mathrm{dim}_\mathrm{in}}}, \frac{-1}{\sqrt{\mathrm{dim}_\mathrm{in}}})$ where $\mathrm{dim}_\mathrm{in}$ is the input dimension of the layer). Then, we multiply the weights of each layer with the initialization scale parameter $\alpha$ to finish the initialization. Then we train the neural network on $f$ with the whole $2^{13}$ possible samples using mini-batch SGD with batch size 256 and learning rate $10^{-4}$. We stop the training when the training loss becomes less than $10^{-4}$. At the end, we compute how much each layer's bias is contributing to the bias learned by the neural network. More precisely, we compare $w_4^TW_3^TW_2^Tb_1$, $w_4^TW_3^Tb_2$, $w_4^Tb_3$, and $b_4$ (respectively the contributing bias from the first layer to the last layer). Particularly, we plot the absolute value of these contributing biases against the initialization scale in the log-log scale in Figure \ref{fig:linear}. It can be seen that as the initialization scale decreases, all of the bias becomes captured by the bias of the last layer and the contributions of other layers' biases (including the first layer's) go to zero supporting Conjecture \ref{conj:linear}.

\begin{figure}
    \centering
    \includegraphics[width=0.7\linewidth]{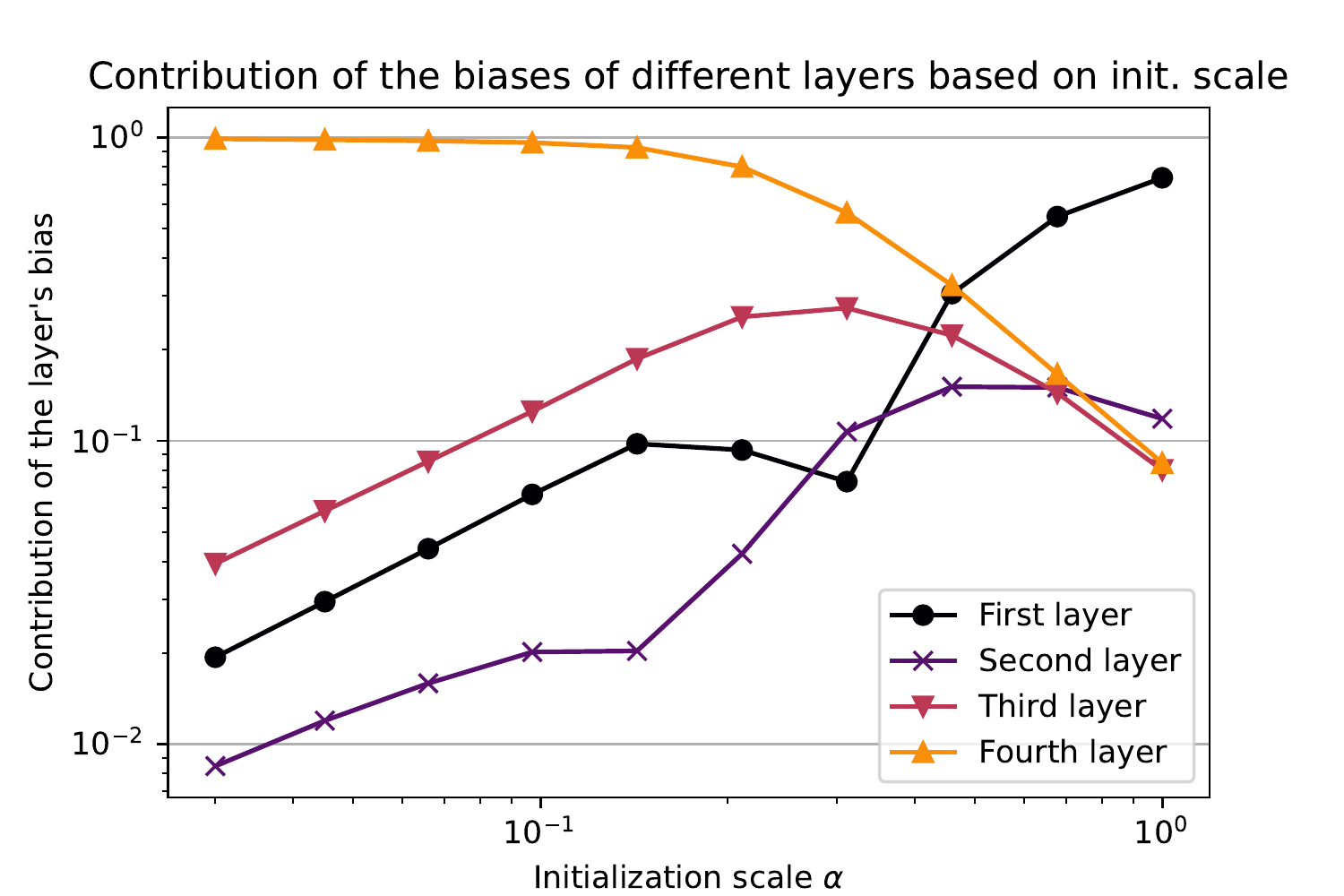}
    \caption{Contribution of each layer's bias to the bias learned by the neural network depending on the initialization scale. It can be seen that as the initialization scale decreases the bias of the neural network is dominantly learned by the last layer's bias. Further, the contribution of biases of all other layers goes to 0.}
    \label{fig:linear}
\end{figure}
\newpage
\section{Vanishing Ideals}\label{app:vanishing-ideals}

In this section, we discuss the connection between unseen domains in Boolean settings and algebraic geometry and vanishing ideals. We refer interested readers to the work of \citet{dummit2004abstract} for broader coverage of this topic. 
First, we recall some basic properties of rings and fields. A ring is a set with two binary operations, the addition $+$ and the multiplication $*$ where $*$ may not have an inverse. A field is a ring such that all nonzero elements have an inverse.
For example, 
$\mZ$ with addition and multiplication is a ring but not a field. Whereas $\mR$ and $\mC$ are examples of fields. Here we will mostly work with polynomial rings with $d$ variables. Note that $\mR[x_1, x_2, x_3, \cdots,x_d]$ is of special interest to us since any Boolean function $f\colon\{\pm 1\}^d \to \mR$ can be represented by a polynomial $p(x) \in \mR[x_1, x_2, x_3, \cdots,x_d]$ thanks to its Fourier-Walsh expansion. Particularly, we focus on polynomial rings 
$R = \mathbb{K}[x_1,x_2, \cdots x_d]$ where $\mathbb{K}$ is a field. We start by recalling a few definitions.
\begin{definition} [Ideal]
Let $R$ be a commutative ring. $I \subseteq R$ is an ideal if
\begin{itemize}
    \item $(I,+)$ is a group, and
        \item for all $r \in R$ and $i \in I$ we have $ri \in I$.
    \end{itemize}
\end{definition}
Having defined ideals, note that ideals can be generated from a $G \subseteq R$. 
\begin{definition}
Consider a commutative ring $R$ and let $G \subseteq R$. The ideal generated by $G$ denoted by $\langle G \rangle$ is the smallest ideal that contains $G$. Particularly, if $G = \{g_1, g_2, \cdots, g_n \}$ is finite, we have
\begin{align}
    \langle G \rangle = \langle g_1, g_2, \cdots, g_n \rangle = \{ \sum_{i=1}^n r_ig_i | \forall r_1, r_2, \ldots, r_n \in R \}.
\end{align}
\end{definition}
For example, for $R = \mR[x_1, x_2]$, we have  $\langle x_1-1,x_1x_2 + 5 \rangle = \{ p(x_1-1) + q(x_1x_2 + 5) | p,q \in  \mR[x_1, x_2]\}$. Another important notion is the notion of quotients, which is similar to the modulo operator. 
The following definition will make it more rigorous.

\begin{definition}[Quotient]
    Let $R$ be a commutative ring and $I$ an ideal of $R$.  Quotient $R/I$ is defined as elements of the form $r +I$ with $r \in R$ such that $r + I = r' + I$ if $r-r' \in I$. Furthermore, for any for $r + I, r' +I \in R/I$, addition $+$ and multiplication $\cdot$ for $R/I$ are defined as 
    \begin{itemize}
        \item $(r+I) + (r' + I) = r+r' + I$, and
        \item $(r+I) \cdot (r' + I) = rr' + I$.
    \end{itemize}
Also, for $r' \in R/I$, any element $r \in R$ satisfying $r-r' \in I$ is called a representative of $r'$. $R/I$ as defined above is indeed a ring.
\end{definition}



Consider the following ideal $I_\Omega = \langle x_1^2-1,x_2^2-1, \cdots, x_d^2 - 1 \rangle$ of $\mR[x_1, x_2, x_3, \cdots,x_d]$ for Boolean functions. Note that for each binary bit $x_i$ we have $x_i^2-1=0$. Therefore, the Fourier-Walsh transform is a bijection between   $\mR[x_1, x_2, x_3, \cdots,x_d]/I_\Omega$ and the set of Boolean functions.

Now we are ready to define vanishing ideals. Given a set of points $S \subseteq \mathbb{K}^d$ where $\mathbb{K}$ is a field, we are interested in the set of polynomials that are zero on $S$.  In the case of generalization on the unseen domain $\mathcal{U} \subseteq \Omega$, we are interested in the functions that vanish on $ \mathcal{U}^c = \Omega \setminus \mathcal{U}$, as they are $0$ on the training set and give a class of interpolators on $\mathcal{U}^c$.

\begin{definition}[Vanishing ideals]
    For a field $\mathbb{K}$ and $S \subseteq \mathbb{K}^d$, vanishing ideal $I(S)$ of $S$ is defined as
    \begin{align*}
        I(S) \coloneqq \{ f \in \mathbb{K}[x_1, \ldots, x_d] | f(x) = 0 \text{ for all $x \in S$} \}.
    \end{align*}
\end{definition}

Note that $I_\Omega$ is indeed the vanishing ideal of $\Omega$, i.e., $I(\Omega) = I_\Omega = \langle x_1^2-1,x_2^2-1, \cdots, x_d^2 - 1 \rangle$. Furthermore, for any $S \subseteq \{\pm 1\}^d$, we have $I_\Omega \subseteq I(S)$ and thus $I(S)$ can be written as $I(S) = \langle v_1,v_2,\ldots,v_n \rangle+ I_\Omega$ for some $n \in \mN$ and Boolean functions $v_1, \ldots, v_n$. For example, consider canonical holdout 
$\mathcal{U}=\{x \in \{\pm 1\}^d|x_1=-1\}$; in this case we get $I(\mathcal{U}^c) = \langle x_1 - 1 \rangle + I_\Omega$.
We could also do an `inverse operation': given an ideal or set of functions, find all the points which are zero under the elements of the ideal.

\begin{definition}
    For a field $\mathbb{K}$  and $G \subseteq R=\mathbb{K}[x_1, \ldots, x_d]$, and $I = \langle G \rangle$, we define $V(G) = V(I)$ as
    \begin{align*}
        V(G) = V(I) = \{ x \in \mathbb{K}^d | f(x) = 0 \text{ for all $f \in I$}\}.
    \end{align*}
\end{definition}
Therefore, operations $V$ and $I$ give us a way to transfer some algebraic properties to geometric properties.  What we defined could be seen as part of the theory of classical algebraic geometry. In algebraic geometry, we are interested in the following type of sets $S$:
\begin{definition}
    A set $S \subseteq \mathbb{K}^d$ is called an affine variety, if there exists some ideal $I$ such that $V(I) = S$.
\end{definition}
In our case, all the $S$ are affine varieties as they are finite. For more details about algebraic geometry, please refer to the work of \citet{cox2013ideals}.

Now given an $S \subseteq \Omega$, the following lemma gives us a recipe to find $I(S)$.
\begin{lemma}
    For $S$ and $W$ two affine varieties, we have that $I(S \cup W) = I(S) \cap I(W).$ Also, for $x = (i_1, i_2, \ldots, i_d) \in \mathbb{K}^d$, we have that $I(x) = m_x = \langle x_1 - i_1 , x_2 - i_2 , \ldots, x_d - i_d \rangle$, where $m_x$ is a maximal ideal.
\end{lemma}
 Since in our case $S$ is finite, one can apply this lemma a multitude of times to find $I(S)$. Moreover, this ideal only vanishes on the elements of $S$ and not on any other element in $\Omega \setminus S$.

\begin{example}
Suppose we work with $d=2$, and we only allow the set $V := \{ (-1,-1) , (1,1) \}$. We will have that $I(V):= \langle  x_1 +1 ,x_2 + 1 \rangle \cap \langle x_1 - 1,x_2 - 1 \rangle$. By doing the calculations or using an algebra program (e.g., SageMath) we find that
\begin{align*}
    I(V) := \langle x_1 - x_2 \rangle + I_\Omega.
\end{align*}   
\end{example}
So in general, for a certain $S \subseteq \Omega$, we would like to express, $I(S)$ as $\langle v_1, \ldots, v_n\rangle + I_\Omega$ for some desirable Boolean functions $v_1, v_2, v_3, \ldots, v_n$. In fact,  there are known algorithms that find a basis for an ideal \citep{moller1982construction}.

Before relating what we have defined to unseen domains, note that in our case the conditions only depend on a subset of the variables. Without loss of generality, suppose our conditions only depend on the  first $k$ coordinates. Mathematically, that means  $\mathcal{U} = \mathcal{U}_k \times \{-1,1\}^{d-k}$, where $\mathcal{U}_k \subseteq \{-1,1\}^k$. Hence, we have $\mathcal{U}^c = \mathcal{U}_k^c \times \{-1,1\}^{d-k}$. The following lemma allows us to calculate $I(\mathcal{U}^c)$ based on $I(\mathcal{U}^c_k)$.
\begin{lemma}
    Suppose that $\mathcal{U}^c = \mathcal{U}^c_k \times \{-1,1\}^{d-k}$ for some $\mathcal{U}^c_k \subseteq \{-1,1\}^k$, if $I(\mathcal{U}^c_k) = \langle v_1,v_2, \ldots , v_n \rangle + \langle x_1^2 -1, \ldots, x_k^2 - 1 \rangle$ for Boolean functions $v_1,v_2,\ldots, v_n$, we have
    \begin{equation*}
        I(\mathcal{U}^c) = \langle v_1,v_2, \ldots , v_n \rangle + I_\Omega.
    \end{equation*}
\end{lemma}

Now having defined the vanishing ideals and quotients, we explain how they relate to our setting. In our setting, we are given $\mathcal{U} \subset \Omega =  \{-1, 1 \}^d$ representing the unseen domain, and a Boolean function $f$ that we wish to learn, which could be seen as an element of $R=\mR[x_1, \ldots, x_d]$. As we finish training, we will converge to a solution $f_{\mathrm{sol}}$ which is an interpolator of  $f$ on $\mathcal{U}^c$. This means that $f - f_{\mathrm{sol}}$ vanishes on $\mathcal{U}^c$ and so $f - f_{\mathrm{sol}} \in I(\mathcal{U}^c)$. Hence, $f + I(\mathcal{U}^c) = f_{\mathrm{sol}} + I(\mathcal{U}^c)$, which means that $f_{\mathrm{sol}}$ is a representative of the class $f + I(\mathcal{U}^c)$ in the ring $R/I(\mathcal{U}^c)$. Here we are interested in the minimum degree-profile interpolator, and  our goal is to classify given a $f$ and $\mathcal{U}$, the minimum degree-profile representatives of $f + I(\mathcal{U}^c)$ in the ring $R / I(\mathcal{U}^c)$. This gives us an overview of how our settings can be related to algebraic notions.

\subsection{Minimum Degree-Profile Interpolators}

We are generally interested in finding the minimum degree-profile interpolators. One way to do this is as follows: given a Boolean function $f\colon \{-1,1\}^d \to \mR$ which we suppose depends only on variables $x_1, \ldots, x_P$ for some integer $P$ and an unseen set $\mathcal{U} \subseteq \Omega $, we find Boolean functions $v_1, \ldots, v_n$ which only depend on the variables $x_1, \ldots, x_P$ such that $I(\Omega \setminus \mathcal{U}) = \langle v_1,v_2, \ldots, v_n \rangle + I_{\Omega}$. We know that minimum degree interpolator $f_{\mathrm{MDI}}$ is of the form
\begin{equation*} \label{eq:min_prof_gen_form}
   f_{\mathrm{MDI}} =  f + g_1 v_1 + \dots g_n v_n,
\end{equation*}
for some  Boolean functions $g_1, \ldots, g_n$. Now note that if we look at the equation above through the lens of Fourier-Walsh expansion, we realize that coefficients of ${f}_{\mathrm{MDI}}$ are linear combinations of Fourier coefficients of $g_1, \ldots, g_n$. One can use this structure to minimize the elements of the degree-profile one by one since each element of the degree-profile is a quadratic expression in Fourier coefficients of $g_1, \ldots, g_n$. Therefore, one can solve these second-degree optimization problems to calculate the unique MD interpolator.

The process presented above is quite long, but there are some cases for which it is easier to find the minimum degree-profile interpolator. We will present some examples below.

\begin{example} [Generalized canonical holdout]
    Given a point in $ \{-1, 1 \}^k$ that is $ i =  (i_1, \ldots, i_k) \in \{-1, -1 \}^k$, for $\mathcal{U} = ( \{-1, 1 \}^k \setminus \{i \} ) \times \{-1, 1 \}^{d-k}$ and for any Boolean function $f$, the minimum degree-profile interpolator can be found as follows: we first notice that $I(\Omega \setminus \mathcal{U}) = \langle x_1 - i_1 , \ldots, x_k - i_k \rangle + I_\Omega$. And so given $f$, the minimum degree-profile interpolator corresponds to $f_{\mathrm{MDI}}(x_1, \ldots, x_k, x_{k+1}, \ldots, x_d) = f(i_1, \ldots, i_k, x_{k+1}, \ldots, x_d)$.
\end{example}

\begin{example} 
    For $\mathcal{U} = \{(-1,-1), (1,1)\} \times \{-1,1\}^{d-2}$ and for any Boolean function $f$, the MD interpolator can be computed by noticing that $I(\Omega \setminus \mathcal{U}) = \langle x_1+x_2 \rangle + I_\Omega$. Hence, given an $f$ and in order to find the MD interpolator one should replace any $x_1$ found by $\frac{1}{2}(x_1 - x_2)$ and all $x_2$ by $\frac{1}{2}(x_2 - x_1)$.
\end{example}
Here is another case where it is easy to find the MD interpolator. We further present a proof for why it is the MD interpolator in this case.
\begin{lemma}
Let $i = (i_1, \ldots, i_k) \in \{ -1, 1 \}^k$ be any point. For any Boolean function $f$ and $\mathcal{U} = i \times \{-1,1\}^{d-k}$, we have that $f$ has a minimum degree-profile interpolator given by replacing all $x_1x_2\cdots x_k$ found by another polynomial $g'(x_1,\ldots ,x_k)$ which can be determined.
\end{lemma}
This is an expected result, we provide nonetheless a formal proof. 
\begin{proof}
We have $I = I(\Omega \setminus \mathcal{U}) = \langle (x_1 + i_1)(x_2 + i_2) \cdots (x_k + i_k) \rangle + I_\Omega$. By expanding $(x_1 + i_1)(x_2+ i_2) \cdots (x_k + i_k)$, we get an expression of the form
\begin{align*}
    x_1x_2\cdots x_k + g(x_1,x_2,\ldots,x_k) 
\end{align*}
with $g(x_1,x_2,\ldots,x_k)$ containing all the possible monomials consisting of $x_1, \ldots,x_k$ of degree strictly less than $k$ with coefficients being $1$ or $-1$.  Consider the polynomial $f_{\mathrm{MDI}}$, by replacing all the $x_1x_2\cdots x_k p(x_{k+1}, \ldots, x_d)$ that appears in $f$ by $-g(x_1,x_2,\ldots,x_k) p(x_{k+1}, \ldots, x_d)$. We claim that $f_{\mathrm{MDI}}$ is the minimum degree-profile interpolator. In fact, suppose that this is not the case, so there exists a polynomial $q$  such that
\begin{equation*}
    (f_{\mathrm{MDI}} + q(x_1 + i_1)(x_2 + i_2) \cdots (x_k + i_k) ) \text{ modulo $I_\Omega$}
\end{equation*}
is not equal to and has a lower degree-profile than $f_{\mathrm{MDI}}$. For this to happen, we need at least one monomial of $f_{\mathrm{MDI}}$ to be (partly) replaced by the same degree or  lower degree alternatives. Among all such monomials, we consider the highest degree one, $\chi_M = \prod_{i\in M} x_i$. We assume that $M = T \cup R$ such that $T \subseteq [k]$ and $R \cap [k] = \emptyset$. Note that monomials that contained $x_1x_2\cdots x_k$ are already replaced, hence, $T \neq [k]$. We write $q$ in the form of $q(x) = s(x_1, \ldots, x_k)\chi_R + q'(x)$ where  $q'(x)$ does not contain any monomial of the form $\chi_R\chi_{T'}$ for $T' \subseteq [k]$. Note that by our assumption $ q(x)(x_1 + i_1)\cdots (x_k + i_k)$, and thus, $ s(x_1, \ldots, x_k)\chi_R(x_1 + i_1)\cdots (x_k + i_k)$ must have generated $\beta \chi_M$, for some $\beta \neq 0$. Note that $\chi_M$ is the highest degree monomial (partly) replaced by $q$. Thus, $\chi_{[k] \cup R}$, is not generated by $q$, otherwise, the degree-profile would have been increased. In other words, $ s(x_1, \ldots, x_k)\chi_R(x_1 + i_1)\cdots (x_k + i_k)$ must have generated $\beta \chi_M$ ($\beta \neq 0$) and not $\chi_{[k] \cup R}$. We now show that such a thing is impossible and reaches a contradiction. Assume that $s(x_1, \dots, x_k) = \sum_{U \subseteq [k]} \alpha_U \chi_U$. Notice we can remove the $\chi_R$ part from the question. Thus, we can consider the equivalent statement that $ s(x_1, \ldots, x_k)(x_1 + i_1)\cdots (x_k + i_k)$  does not generate $x_1\cdots x_k$ while it generates $\beta \chi_T$. Now we compute the coefficients of $\chi_T$ and $\chi_{[k]} = x_1\cdots x_k$ in $ s(x_1, \ldots, x_k)(x_1 + i_1)\cdots (x_k + i_k)$. We have 
\begin{align*}
    &s(x)(x_1 + i_1)\cdots (x_k + i_k) =  s(x)(\sum_{V \subseteq [k]} \chi_V \prod_{n \in [k] \setminus V} i_n) \nonumber\\
    &= (\sum_{U \subseteq [k]} \alpha_U \chi_U)(\sum_{V \subseteq [k]} \chi_V \prod_{n \in [k] \setminus V} i_n) \nonumber \\
    &= (\sum_{U \subseteq [k]} \alpha_U \prod_{n \in U} i_n)\chi_{[k]} + \cdots + (\sum_{U \subseteq [k]} \alpha_U  \prod_{n \in [k] \setminus (T \Delta U)} i_n)\chi_T + \cdots,
\end{align*}
and the coefficient of $\chi_T$ is $\sum_{U \subseteq [k]} \alpha_U  \prod_{n \in [k] \setminus (T \Delta U)} i_n = \beta$. Using $[k] \setminus (U \Delta T) = U \Delta ([k] \setminus T)$, we have 
\begin{align*}
    \beta &= \sum_{U \subseteq [k]} \alpha_U  \prod_{n \in [k] \setminus (T \Delta U)} i_n = \sum_{U \subseteq [k]} \alpha_U  \prod_{n \in ([k] \setminus T) \Delta U} i_n \\
    &=  (\prod_{n \in [k] \setminus T} i_n)(\sum_{U \subseteq [k]} \alpha_U \prod_{n \in  U} i_n) = 0,
\end{align*}
where we used the fact that $i_j^2=1 \;\; \forall j$. Hence, $\beta = 0$ which is a contradiction, showing that it is not possible to reduce the degree-profile. 
\end{proof}

\vskip 0.2in
\bibliography{references}

\end{document}